\documentclass{article}

\usepackage{amsmath,amssymb}
\usepackage[mathscr]{euscript} 
\usepackage{amsthm} 
\usepackage{tikz} 
\usepackage{graphicx} 
\newcommand*{\Scale}[2][4]{\scalebox{#1}{\ensuremath{#2}}} 
\usepackage{framed} 
\usepackage{float} 
\usepackage{multirow}
\usepackage{subfigure}
\usepackage{longtable}
\usepackage{makeidx} 
\makeindex
\usepackage{hyperref} 

\newtheorem{myDefinition}{Definition}
\newtheorem{myTheorem}{Theorem}
\newtheorem{myLemma}{Lemma}
\newtheorem{myCorollary}{Corollary}

\newtheorem{myExample}{Example}
\newtheorem{myRemark}{Remark}

\tikzstyle{every edge}=  [draw]
\tikzstyle{vertex} = [draw,circle,minimum size=1pt]
\tikzstyle{label} = [minimum size=.1pt,font=\tiny]
\tikzstyle{title} = [minimum size=.25cm,font=\small]

\newcommand{\var}[1]{\mathsf{#1}}
\newcommand{\bs}[1]{\boldsymbol{#1}}
\newcommand{\bb}[1]{\breve{#1}}

\newcommand{\Rdo}[1]{{\hyperref[hatsDef]{\mathbb{R}^d_{#1}}}}
\def \R{\mathbb{R}}
\def \spn{\mathrm{span}}
\def \miss{{\hyperref[missDef]{\cdot}}}
\def \see{{\hyperref[seeDef]{\times}}}
\def \c{\mathsf{c}}
\def \T{\mathsf{T}}
\def \<{\langle}
\def \>{\rangle}

\def \m{{\hyperref[nmDef]{m}}}
\def \n{{\hyperref[nmDef]{n}}}
\def \r{{\hyperref[rDef]{r}}}
\def \d{{\hyperref[dDef]{d}}}
\def \N{{\hyperref[XDef]{N}}}
\def \L{{\hyperref[LDef]{\ell}}}
\def \KK{{\hyperref[KKDef]{K}}}

\newcommand{\xii}[1]{{\hyperref[XDef]{x_{#1}}}}
\newcommand{\hatxi}[1]{{\hyperref[hatxDef]{\hat{x}_{#1}}}}
\newcommand{\xoi}[1]{{\hyperref[xoDef]{x_{\omega_{#1}}}}}
\def \y{{\hyperref[yDef]{\chi}}}
\def \x{{\hyperref[xDef]{x}}}
\def \hatx{{\hyperref[hatxDef]{\hat{x}}}}
\def \xo{{\hyperref[xoDef]{x_\omega}}}
\def \XI{{\hyperref[XIDef]{\Xi}}}
\def \hatXI{{\hyperref[XIDef]{\hat{\Xi}}}}
\def \X{{\hyperref[XDef]{\mathscr{X}}}}
\def \hatX{{\hyperref[hatXDef]{\hat{\mathscr{X}}}}}

\newcommand{\sstark}[1]{{\hyperref[sstarkiDef]{S^\star_{#1}}}}
\newcommand{\sstari}[1]{{\hyperref[sstariDef]{S^\star_{#1}}}}
\newcommand{\hatsstari}[1]{{\hyperref[hatsstariDef]{\hat{S}^\star_{#1}}}}
\newcommand{\sstaroi}[1]{{\hyperref[sstaroiDef]{S^\star_{\omega_{#1}}}}}
\newcommand{\soi}[1]{{\hyperref[soDef]{S_{\omega_{#1}}}}}
\newcommand{\hatsi}[1]{{\hyperref[hatsDef]{\hat{S}_{#1}}}}
\def \Sstar{{\hyperref[SstarkDef]{\mathscr{S}^\star}}}
\def \s{{\hyperref[sDef]{S}}}
\def \hats{{\hyperref[hatsDef]{\hat{S}}}}
\def \so{{\hyperref[soDef]{S_\omega}}}
\def \sstar{{\hyperref[sstarDef]{S^\star}}}
\def \hatsstar{{\hyperref[hatsstariDef]{\hat{S}^\star}}}
\def \sstaro{{\hyperref[soDef]{S^\star_\omega}}}

\newcommand{\uui}[1]{{\hyperref[uuiDef]{\mathscr{U}_{#1}}}}
\newcommand{\Ustari}[1]{{\hyperref[UstariDef]{U^\star_{#1}}}}
\newcommand{\hatUi}[1]{{\hyperref[hatsDef]{\hat{U}_{#1}}}}
\newcommand{\hatUstari}[1]{{\hyperref[hatUstariDef]{\hat{U}^\star_{#1}}}}
\newcommand{\Uoi}[1]{{\hyperref[soDef]{U_{\omega_{#1}}}}}
\newcommand{\uj}[1]{{\hyperref[UUDef]{u_{#1}}}}
\newcommand{\uoi}[1]{{\hyperref[UUDef]{u_{\omega_{#1}}}}}
\def \UUO{{\hyperref[UUDef]{\mathscr{U}_{\bs{\Omega}}}}}
\def \UU{{\hyperref[UUDef]{\mathscr{U}_\Omega}}}
\def \uu{{\hyperref[uuDef]{\mathscr{U}_{\omega}}}}
\def \U{{\hyperref[UDef]{U}}}
\def \hatU{{\hyperref[hatsDef]{\hat{U}}}}
\def \hatUstar{{\hyperref[hatsDef]{\hat{U}^\star}}}
\def \Uo{{\hyperref[soDef]{U_\omega}}}
\def \Ustar{{\hyperref[UstarDef]{U^\star}}}
\def \Ustaro{{\hyperref[UstaroDef]{U^\star_\omega}}}
\def \u{{\hyperref[UUDef]{u}}}
\def \uo{{\hyperref[xoDef]{u_\omega}}}
\def \hatu{{\hyperref[hatxDef]{\hat{u}}}}

\newcommand{\ai}[1]{{\hyperref[aiDef]{a_{#1}}}}
\newcommand{\aoi}[1]{{\hyperref[aoiDef]{a_{\omega_{#1}}}}}
\def \AA{{\hyperref[AADef]{\bs{A}}}}
\def \A{{\hyperref[ADef]{A}}}
\def \a{{\hyperref[aDef]{a}}}
\def \ao{{\hyperref[aoDef]{a_\omega}}}
\def \aoT{{\hyperref[aoDef]{a^\T_\omega}}}

\newcommand{\oi}[1]{{\hyperref[OODef]{\omega_{#1}}}}
\def \OO{{\hyperref[OODef]{\bs{\Omega}}}}
\def \O{{\hyperref[ODef]{\Omega}}}
\def \o{{\hyperref[oDef]{\omega}}}
\def \ups{{\hyperref[upsDef]{\upsilon}}}
\def \I{{\hyperref[IJDef]{\mathcal{I}}}}
\def \J{{\hyperref[IJDef]{\mathcal{J}}}}
\def \K{{\hyperref[KDef]{\mathscr{K}}}}

\newcommand{\jci}[1]{{\hyperref[jcDef]{j^\c_{#1}}}}
\newcommand{\ki}[1]{{\hyperref[KDef]{k_{#1}}}}
\def \i{{\hyperref[XDef]{i}}}
\def \ii{{\ddot{\imath}}}
\def \j{{\hyperref[jDef]{j}}}
\def \jc{{\hyperref[jcDef]{j^\c}}}
\def \k{{\hyperref[SstarkDef]{k}}}

\def \behave{{\hyperref[behaveDef]{behave}}}
\def \behaves{{\hyperref[behaveDef]{behaves}}}
\def \fitsO{{\hyperref[fittingODef]{fits}}}
\def \fitO{{\hyperref[fittingODef]{fit}}}
\def \fitso{{\hyperref[fittingoDef]{fits}}}
\def \fito{{\hyperref[fittingoDef]{fit}}}
\def \fitsxo{{\hyperref[fittingxoDef]{fits}}}
\def \fitxo{{\hyperref[fittingxoDef]{fit}}}
\def \fitsXO{{\hyperref[fittinghatXDef]{fits}}}
\def \fitXO{{\hyperref[fittinghatXDef]{fit}}}
\def \fitGen{{\hyperref[fittingSetsSec]{fit}}}
\def \fitsGen{{\hyperref[fittingSetsSec]{fits}}}
\def \independentO{{\hyperref[independenceDef]{independent}}}
\def \dependentO{{\hyperref[independenceDef]{dependent}}}
\def \independento{{\hyperref[redundantDef]{independent}}}
\def \redundant{{\hyperref[redundantDef]{redundant}}}
\def \dependento{{\hyperref[redundantDef]{dependent}}}
\def \basis{{\hyperref[basisDef]{basis}}}
\def \bases{{\hyperref[basisDef]{bases}}}
\def \degenerate{{\hyperref[degenerateDef]{degenerate}}}

\def \uniquenessThm{{\hyperref[uniquenessThm]{Theorem \ref{uniquenessThm}}}}
\def \allOfAKindThm{{\hyperref[allOfAKindThm]{Theorem \ref{allOfAKindThm}}}}
\def \aEntriesLem{{\hyperref[aEntriesLem]{Lemma \ref{aEntriesLem}}}}
\def \UUcontainsAllSLem{{\hyperref[UUcontainsAllSLem]{Lemma \ref{UUcontainsAllSLem}}}}
\def \dimUULem{{\hyperref[dimUULem]{Lemma \ref{dimUULem}}}}
\def \independenceLem{{\hyperref[independenceLem]{Lemma \ref{independenceLem}}}}
\def \basisLem{{\hyperref[basisLem]{Lemma \ref{basisLem}}}}
\def \kCharLem{{\hyperref[kCharLem]{Lemma \ref{kCharLem}}}}
\def \kCharConverseLem{{\hyperref[kCharConverseLem]{Lemma \ref{kCharConverseLem}}}}
\def \basisCharLem{{\hyperref[basisCharLem]{Lemma \ref{basisCharLem}}}}
\def \allOfAKindLem{{\hyperref[allOfAKindLem]{Lemma \ref{allOfAKindLem}}}}
\def \aoWithzerosLem{{\hyperref[aoWithzerosLem]{Lemma \ref{aoWithzerosLem}}}}
\def \UUfitsOCor{{\hyperref[UUfitsOCor]{Corollary \ref{UUfitsOCor}}}}
\def \dimUUrCor{{\hyperref[dimUUrCor]{Corollary \ref{dimUUrCor}}}}
\def \LCor{{\hyperref[LCor]{Corollary \ref{LCor}}}}
\def \converseAllOfAKindCor{{\hyperref[converseAllOfAKindCor]{Corollary \ref{converseAllOfAKindCor}}}}
\def \introEg{{\hyperref[introEg]{Example \ref{introEg}}}}
\def \subspacesBasesEg{{\hyperref[subspacesBasesEg]{Example \ref{subspacesBasesEg}}}}
\def \vectorsBasesEg{{\hyperref[vectorsBasesEg]{Example \ref{vectorsBasesEg}}}}
\def \observationSetsEg{{\hyperref[observationSetsEg]{Example \ref{observationSetsEg}}}}

\def \fittingOEg{{\hyperref[fittingOEg]{Example \ref{fittingOEg}}}}
\def \allOfAKindEga{{\hyperref[allOfAKindEga]{Example \ref{allOfAKindEga}}}}

\def \prologueaEg{{\hyperref[prologueaEg]{Example \ref{prologueaEg}}}}
\def \prologuebEg{{\hyperref[prologuebEg]{Example \ref{prologuebEg}}}}

\def \allOfAKindEgb{{\hyperref[allOfAKindEgb]{Example \ref{allOfAKindEgb}}}}
\def \independenceDef{{\hyperref[independenceDef]{Definition \ref{independenceDef}}}}

\def \fittingxoDef{{\hyperref[fittingxoDef]{Definition \ref{fittingxoDef}}}}
\def \fittinghatXDef{{\hyperref[fittinghatXDef]{Definition \ref{fittinghatXDef}}}}
\def \fittingODef{{\hyperref[fittingODef]{Definition \ref{fittingODef}}}}
\def \degenerateDef{{\hyperref[degenerateDef]{Definition \ref{degenerateDef}}}}

\def \fittingoRmk{{\hyperref[fittingoRmk]{Remark \ref{fittingoRmk}}}}
\def \fittingORmk{{\hyperref[fittingORmk]{Remark \ref{fittingORmk}}}}
\def \fittingOIsFuncRmk{{\hyperref[fittingOIsFuncRmk]{Remark \ref{fittingOIsFuncRmk}}}}

\def \severalOptionsFiga{{\hyperref[severalOptionsFiga]{Figure \ref{severalOptionsFiga}}}}
\def \severalOptionsFigb{{\hyperref[severalOptionsFigb]{Figure \ref{severalOptionsFigb}}}}
\def \differentSthreeFig{{\hyperref[differentSthreeFig]{Figure \ref{differentSthreeFig}}}}
\def \almostNeverFig{{\hyperref[almostNeverFig]{Figure \ref{almostNeverFig}}}}
\def \intesectionUoneUtwoFig{{\hyperref[intesectionUoneUtwoFig]{Figure \ref{intesectionUoneUtwoFig}}}}
\def \almostSurelyFig{{\hyperref[almostSurelyFig]{Figure \ref{almostSurelyFig}}}}
\def \aProp{{\hyperref[aProp]{{\bf (a)}}}}
\def \bProp{{\hyperref[bProp]{{\bf (b)}}}}
\def \rDimensionAss{{\hyperref[rDimensionAss]{{\bf A2}}}}
\def \sstarNonDegenerateAss{{\hyperref[sstarNonDegenerateAss]{{\bf A3}}}}
\def \sizeoAss{{\hyperref[sizeoAss]{{\bf A4}}}}
\def \unionOOAss{{\hyperref[unionOOAss]{{\bf A5}}}}
\def \existsNonDegenerateAss{{\hyperref[existsNonDegenerateAss]{{\bf A6}}}}
\def \LRMCconverseMot{{\hyperref[LRMCconverseMot]{{\bf M1}}}}
\def \LRMCcheckMot{{\hyperref[LRMCcheckMot]{{\bf M2}}}}
\def \universalValidationCheckMot{{\hyperref[universalValidationCheckMot]{{\bf M3}}}}

\title{{\bf To lie or not to lie in a subspace}}
\author{\\ \\
Daniel L. Pimentel-Alarc\'{o}n\\
Electrical and Computer Engineering, Mathematics \\
\texttt{pimentelalar@wisc.edu}
\and \\ \\
\textsf{RESEARCH PROPOSAL} \\ \\ \\
University of Wisconsin-Madison \\
Madison, WI, 53706, USA
}

\begin{document}
\maketitle

\begin{align*}
\begin{array}{ccccc}
\multicolumn{5}{c}{\textsf{Submitted for revision to committee members:}} \\
\\
\text{Robert D. Nowak} & & & & \text{Nigel Boston} \\
\text{Electrical and Computer} & & & & \text{Electrical and Computer} \\
\text{Engineering} & & & & \text{Engineering, Mathematics} \\
\texttt{nowak@ece.wisc.edu} & & & & \texttt{boston@math.wisc.edu}
\end{array}
\end{align*}
\\
\begin{abstract}
Give deterministic necessary and sufficient conditions to guarantee that if a subspace {\em \fitsXO} certain {\bf partially observed} data from a union of subspaces, it is because such data really {\em lies} in a subspace.

Furthermore, give deterministic necessary and sufficient conditions to guarantee that if a subspace \fitsXO\ certain partially observed data, such subspace is {\em unique}.

Do this by characterizing when and only when a set of incomplete vectors {\em \behaves} as a single but complete one.
\end{abstract}

\pagebreak
\tableofcontents
\pagebreak
\section{Prologue}
\label{prologueSec}
We love subspaces.  We observe a phenomenon and try to find a line that explains it.  We get our hands on some data, and we try to find a subspace that fits it.  But what if we are looking for subspaces where there really are not? How can we guarantee that if we find a subspace, it is because there really {\em is} a subspace?  In other words, how can we make sure that if certain data {\em fit} in a subspace, it is because it really {\em lies} in such subspace?

In many cases we don't really have to worry about this problem.  For instance, if we have a collection of generic vectors that \fitGen\ in an $\r$-dimensional subspace, as long as our collection has more than $\r$ vectors, we can always verify if our collection indeed lies in an $\r$-dimensional subspace, because we will always have an {\em extra}, {\em generic} vector to validate this.  This is because almost surely, a set of more than $\r$ generic vectors \fitsGen\ in an $\r$-dimensional subspace iff it actually lies in such subspace.

Nevertheless, if we suppose that our collection of vectors is only {\bf partially observed}, this becomes a much harder problem, as a set of arbitrarily many {\em incomplete vectors} may \fitXO\ in an $\r$-dimensional subspace even if their {\em complete} counterparts do not really lie in a subspace.

\begin{myExample}
\label{prologueaEg}
\normalfont
Suppose $\r=1$, and consider the following set of vectors:
\begin{align*}
\X=\left[ \begin{matrix} 1 & 1 \\ 1 & 2 \\ 1 & 3 \end{matrix} \right].
\end{align*}
It is easy to see that they do not lie in a $1$-dimensional subspace.  Nevertheless, suppose that we only observe a subset of their entries:
\begin{align*}
\hatX=\left[\begin{matrix}
	1 & 1 \\
	1 & \miss \\
	\miss & 3
\end{matrix}\right].
\end{align*}
Then both {\em incomplete vectors} \fitXO\ in the $1$-dimensional subspace spanned by
\begin{align*}
\U=\left[ \begin{matrix} 1 \\ 1 \\ 3 \end{matrix} \right],
\end{align*}
despite their full counterparts do not lie in a $1$-dimensional subspace.
$\blacksquare$
\end{myExample}

Of course, in general, without knowing anything a priori about our data there is no hope to succeed at this task, as the missing entries could be arbitrary.  Fortunately there are ma cases of data that lies in \----or can be accurately approximated by\---- a union of subspaces\cite{vidaltutorial}, a beautiful setup under which this task is not only feasible but also non-trivial.  This is precisely the assumption under which we will operate, i.e.,
\\
\\
\noindent {\bf We will assume in the rest of this document that every vector of our data lies in the union of $\Sstar$, a set of $\r$-dimensional subspaces of $\R^\d$.}

\begin{figure}[H]
\centering
\includegraphics[width=5cm]{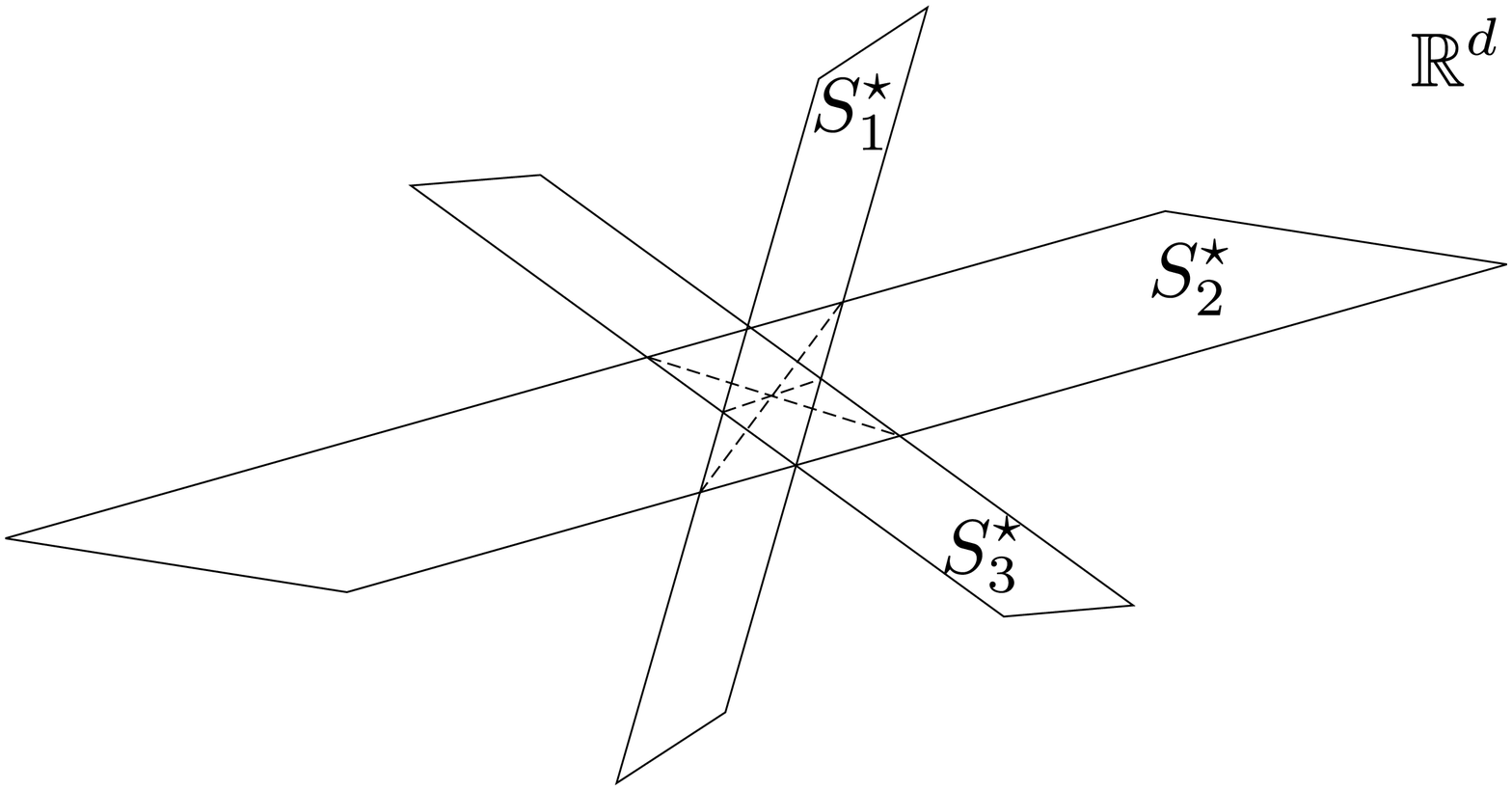}
\caption{Example of the union of three $2$-dimensional subspaces of $\R^3$.}
\label{unionSubspacesFig}
\end{figure}

At first glance this might deceivingly appear as a trivial task: if all vectors lying in a union of subspaces fit in one $\r$-dimensional subspace, how could they not all lie in an $\r$-dimensional subspace? To see this, consider the following.

\begin{myExample}
\label{prologuebEg}
\normalfont
With the same setup as in \prologueaEg, further suppose that $\Sstar$ is the set of the two subspaces spanned by the following vectors:
\begin{align*}
\Ustari{1} = \left[ \begin{matrix} 1 \\ 1 \\ 1 \end{matrix} \right], \hspace{.5cm}
\Ustari{2} = \left[ \begin{matrix} 1 \\ 2 \\ 3 \end{matrix} \right].
\end{align*}
If we again assume that our data is the set of vectors $\X$ \----which clearly lies in the union of the subspaces of $\Sstar$\---- but that we only observe the subset of their entries in $\hatX$, it is easy to see that $\U$ \fitsXO\ our data despite their full counterparts do not lie in a $1$-dimensional subspace.
$\blacksquare$
\end{myExample}

The motivation of this work is to find necessary and sufficient conditions to guarantee that if a set of incomplete vectors from a union of subspaces \fitsXO\ in an $\r$-dimensional subspace, it is because the set of all their complete yet unknown counterparts indeed lies in an $\r$-dimensional subspace.

\pagebreak
\section{Introduction}
\label{introSec}

Imagine that an $\r$-dimensional subspace $\s$ \fitsXO\ a set of incomplete vectors \phantomsection\label{XIDef}$\hatXI$.  We want to make sure that the set of all their complete yet unknown counterparts indeed lies in $\s$.

Using the same idea as if the vectors were complete, imagine we had an {\em extra}, {\em generic} complete validating vector \phantomsection\label{yDef}$\y \in \sstar \in \Sstar$ that \fitGen\ in $\s$.  It is easy to see that if $\s$ \fitsGen\ $\y$, it is because $\s=\sstar$.  Furthermore, since $\s$ \fitsXO\ $\hatXI$, this implies that $\sstar$ \fitsXO\ $\hatXI$.

It is also easy to see that if the subspaces in $\Sstar$ keep no relation with each other, $\hatXI$ can only \fitXO\ in one of the subspaces of $\Sstar$ if all its complete counterparts indeed lie in such subspace.  This way, $\s$ fitting $\y$ would directly imply that all the complete counterparts of $\hatXI$ indeed lie in $\s$.  All the more, it would imply that $\s \in \Sstar$.

This is all very nice, but it relies on the fantasy that we had the extra, generic, complete vector $\y$.  Of course, we cannot assume that we have such complete vector.  But what if we had several incomplete ones instead?  Could a set of extra generic incomplete vectors $\hatX$ \phantomsection\label{behaveDef}\index{behave}{\em behave} just as $\y$, allowing us to say that {\em if} $\s$ \fitsGen\ such set, then all the complete counterparts of $\hatXI$ indeed lie in $\s$?

The answer to this question is yes, and this is precisely what we characterize: when will a set of incomplete vectors $\hatX$ \behave\ as a complete one.  This characterization is given in \allOfAKindThm, the main result of the document, which, intuitively, and in a nutshell states that:

\begin{framed}
$\hatX$ \behaves\ as a complete vector iff $\hatX$ contains $\d-\r+1$ vectors such that for every strict subset of $\n$ of such vectors, there are at least $\n+\r$ distinct observed {\em rows}.
\end{framed}

This characterization allows us to fulfill the main task we pursue: determine if a set of incomplete vectors really lies in an $\r$-dimensional subspace whenever it \fitsGen\ in an $\r$-dimensional subspace.

\begin{myExample}
\label{introEg}
\normalfont
Suppose $\r=2$ and
\begin{align*}
\hatX=\left[\begin{matrix}
1 & \miss & \miss & 3 & 3 \\
1 & 2 & \miss & \miss & 4 \\
1 & 3 & 4 & \miss & \miss \\
\miss & 4 & 5 & 9 & 6 \\
\miss & \miss & 6 & 11 & \miss \\
\end{matrix}\right].
\end{align*}
Take the set of the first $\d-\r+1=4$ vectors.  We can verify that every one of its subsets has at least $\n+\r$ distinct rows with at least one observed entry.  For example, if we take the first $\n=2$ vectors, the number of distinct observed rows is $4$, which is equal to $\n+\r$.

We thus conclude that if $\s$ \fitsXO\ $\hatXI$ and $\hatX$, then all the complete counterparts of both $\hatX$ and $\hatXI$ indeed lie in $\s$.
$\blacksquare$
\end{myExample}

\subsection{Insight}
\label{insightSec}
There are two fundamental reasons why $\s$ fitting a generic $\y$ implies that the complete counterparts of $\hatXI$ indeed lie in $\s$:
\begin{enumerate}
\item[{\bf (a)}]
\phantomsection\label{aProp}
There is only one $\r$-dimensional subspace that \fitsGen\ $\y$.
\item[{\bf (b)}]
\phantomsection\label{bProp}
Obvious, but essential: $\y$ lies in one and only one of the subspaces of $\Sstar$.
\end{enumerate}
If a set of incomplete vectors $\hatX$ satisfied analogous properties, it would \behave\ just as the complete vector $\y$ in the sense that we would be able to conclude that $\s = \sstar \in \Sstar$ if $\s$ \fitsGen\ $\hatX$, and the remainder \----that $\hatXI$ indeed lies in $\s$\---- would follow just as before.

On the other hand, if $\hatX$ failed to have either property, it would fail to \behave\ as the complete vector $\y$.  More precisely, if $\hatX$ fails to satisfy \aProp, it is evident that we cannot conclude that $\s \in \Sstar$; if $\hatX$ fails to satisfy \bProp, even if there is only one $\r$-dimensional subspace that \fitsGen\ $\hatX$, such subspace might not belong to $\Sstar$, i.e., it could be a {\em false} subspace.  For an example of how this could happen, take \prologuebEg.

In other words, the analogous properties of \aProp\ and \bProp\ are necessary and sufficient for $\hatX$ to \behave\ as a complete vector.  This is precisely what we need to discover: when will a set of generic incomplete vectors $\hatX$ satisfy these two analogous properties.

\begin{myRemark}
\normalfont
Observe that \bProp\ is substantially different in the complete and incomplete vectors cases.  In the complete case, since $\y \in \Sstar$, and $\y$ is only one vector, we can automatically conclude that almost surely, $\y$ will lie in only one of the subspaces of $\Sstar$.

On the other hand, in the incomplete case, $\hatX \in \Sstar$ does not imply that $\hatX$ lies in one and only one of the subspaces of $\Sstar$, as different vectors from $\hatX$ could belong to different subspaces from $\Sstar$.
$\blacksquare$
\end{myRemark}

\subsection{The essence}
\label{theEssenceSec}
We will see in \textsection \ref{fittingSetsSec} that whether or not $\s$ \fitsGen\ a generic $\hatX$ depends only on the position of the observed entries of $\hatX$, namely $\OO$.  We will also see in \textsection \ref{allOfAKindSec} that whether a generic $\hatX \in \Sstar$ \fitsGen\ in a single $\sstar \in \Sstar$ or not can also be deduced from $\OO$ alone.

Therefore, we may focus on finding conditions on $\OO$ to determine when a generic $\hatX$ satisfies the analogous properties \aProp\ and \bProp.  This is exactly what we do.  Explicitly:

\vspace{.3cm}
\begingroup
\leftskip1.5em
\rightskip\leftskip
\noindent
{\em We derive deterministic necessary and sufficient conditions on $\OO$ to {\em guarantee} that {\em if} there exists an $\r$-dimensional subspace that \fitsGen\ a generic $\hatX$, such subspace is unique, and it is because all the vectors of $\hatX$ indeed lie in the same subspace of $\Sstar$.}
\par
\endgroup
\vspace{.3cm}

To be clear, the conditions are sufficient in the sense that if $\OO$ satisfies such conditions and there exists an $\r$-dimensional subspace that \fitsGen\ a generic $\hatX$, such subspace is unique, and it must be true that all the vectors of $\hatX$ indeed lie in the same subspace of $\Sstar$.  Conversely, the conditions are necessary in the sense that even if there exists an $\r$-dimensional subspace that \fitsGen\ a generic $\hatX$, if $\OO$ does not satisfy such conditions, such subspace may not be unique, and it cannot be guaranteed that the vectors of $\hatX$ lie in the same subspace of $\Sstar$.

The conditions to guarantee that all the elements of $\hatX$ indeed lie in the same subspace of $\Sstar$ are given in \allOfAKindThm.  They imply and rely on the conditions for uniqueness, which are given in \uniquenessThm.  As we could see in \textsection \ref{introSec}, these conditions are extremely simple and concrete, and depend only on the most elemental invariants of $\OO$: essentially, cardinalities of its subsets.  Both of these results, the main ones of the document, are presented formally in \textsection \ref{resultsSec}, the section of results.  Together, they characterize when a set of incomplete vectors \behaves\ as a complete one, which allows us to verify our final goal: when $\hatXI$ indeed lies in $\s$.

\subsection{Organization of the document}
\label{organizationSec}
In \textsection \ref{preambleSec} it is given a brief discussion about previous and related work; this helps as preamble to give some motivation for this problem and talk about some particularly interesting applications of this work that give simple yet powerful consequences of its results.  In \textsection \ref{setupSec} it is given a detailed exposition of the setup that will be used in the remainder of the document.  In \textsection \ref{assumptionsSec} the assumptions of this work are stated, explained and discussed.  The main results are given in \textsection \ref{resultsSec}.

The analysis to prove \uniquenessThm\ is presented in \textsection \ref{uniquenessSec}, and the one to prove \allOfAKindThm\ in \textsection \ref{allOfAKindSec}.  In \textsection \ref{intuitivelySpeakingSec} it is offered an intuitive explanation of the key ideas of the results and the assumptions are discussed in more detail, as well as some simple generalizations.  Finally, a brief proposal for future research is given in \textsection \ref{epilogueSec}.

To make the reading of this document easier, the main symbols, terms, statements, definitions, examples, etc., are referenced in the whole document in its electronic version; alternatively, an index and al list of symbols is also provided at the end of the document.

\pagebreak
\section{Preamble}
\label{preambleSec}
With the arrival of big data come big challenges: we want to find useful information in our datasets quickly, cleverly, using as few resources as possible.  Fortunately, in uncountable applications we may use subspaces to model our data, and this greatly simplifies things.

But this is not it.  As if finding useful information quickly, cleverly and efficiently were not ambitious enough endeavors, we also want \----and many times, need\---- to achieve these tasks only with partial information, which comes as no surprise, as the bigger our data, the more likely it is incomplete.

Fortunately, subspaces have a natural way of handling missing data, as data in subspaces have certain structure, and that gives us a way to infer the missing entries.  The problem of handling missing data has attracted a lot of attention in recent years. Remarkable work has been done to identify a subspace that fits certain incomplete data, e.g., \cite{mcRecht}, to detect if an incomplete datum fits in a certain subspace\cite{Balzano10a}, or even to do subspace clustering from missing data\cite{pimentel14}, but the converse problem, in the sense we discuss in \LRMCconverseMot, has been left unattended, and remained, to the best of our knowledge, an open problem until now.

But again, as if the task of finding useful information quickly, cleverly, efficiently, and only from partial information were not bold enough, we also want something else.  We want to make sure that if we reach a conclusion from our data, such conclusion is correct.  In other words, we want to make sure that the information we found is not a product of chance; the larger our data, the more outliers, the more likely we will find {\em something}, but that doesn't mean that that {\em something} is {\em true}.  If we toss a coin a trillion times, we will very likely see many sequences of many heads in a row, but that doesn't mean that a sequence of many heads in a row is very likely.  The more data we have, the more likely {\em some} subspace will {\em fit} {\em some} of it, but that doesn't mean that our data really {\em lies} in such subspace.

These is precisely the task that we are interested on: how to determine when certain incomplete data really {\em lies} in a subspace whenever it {\em fits} in a subspace.

Notice the subtle but fundamental difference between this work and, for example, the matched subspace detection with missing data problem in \cite{Balzano10a}, where they are concerned with determining if an incomplete datum {\em fits} a subspace, using only information about such datum and the subspace.  Here we are given an incomplete dataset that we already know {\em fits} in a subspace, and we want to make sure that it really {\em lies} in it, using the dataset as a whole, exploiting information about the relation between their datums.  Similarly, in \cite{mcRecht} they are concerned with identifying a subspace that {\em fits} certain incomplete data, under the assumption that the data {\em lies} in a subspace.  Here we drop such assumption; we are given an incomplete dataset and a subspace that {\em fits} it, and we want to know if the dataset really {\em lies} in such subspace.

The problem of determining if certain incomplete data really {\em lies} in a subspace whenever it {\em fits} in a subspace is tightly related to the problem of identifying when there is only one subspace that fits such data.  We answer these questions by characterizing when and only when a set of incomplete vectors \behaves\ as a single but complete one, in the sense described in \textsection \ref{introSec}.

Being these so fundamental problems, answering these questions should be enough motivation by itself, as they essentially apply to virtually every problem involving subspaces and missing data.  Nevertheless, just for completeness, we mention just a few motivating applications, to give an idea of the scope and relevance of these results.

\begin{description}
\item[M1.]
\phantomsection\label{LRMCconverseMot}
Consider the low-rank matrix completion problem\cite{mcRecht}: given that all the columns of a matrix $\XI$ lie in the same $\r$-dimensional subspace $\s$, under what conditions is $\s$ the only $\r$-dimensional subspace that fits a subset of the entries of such matrix, $\hatXI$?

The condition that $\XI$ lies in the same $\r$-dimensional subspace trivially implies that there exists an $\r$-dimensional subspace that fits $\hatXI$.  As we explained in \textsection \ref{prologueSec}, the converse is not necessarily true (see Examples \ref{prologueaEg}, \ref{allOfAKindEga} and \textsection \ref{allOfAKindSec} for a more detailed explanation).

Our work immediately provides a {\bf converse} of the low-rank matrix completion problem: all the columns of a matrix $\XI$ lie in the same $\r$-dimensional subspace if there exists an $\r$-dimensional subspace that fits $\hatXI$ and an additional generic $\hatX$ observed in a set $\OO$ satisfying the conditions of \allOfAKindThm.

\item[M2.]
\phantomsection\label{LRMCcheckMot}
Under the same setup of low-rank matrix completion, most algorithms, e.g., nuclear norm minimization\cite{mcRecht}, detect an $\r$-dimensional subspace $\s$ that fits an incomplete dataset $\hatXI$, and claim that with high probability, the detected subspace $\s$ is the only $\r$-dimensional one that does.  \uniquenessThm\ provides a deterministic validation check for any such algorithm: $\s$ is almost surely the unique $\r$-dimensional subspace that fits $\hatXI$ if in addition it also fits a generic $\hatX$ observed in a set $\OO$ satisfying the conditions of \uniquenessThm.

As we said in \LRMCconverseMot, under this setup $\XI$ is already assumed to belong to the same subspace, so \allOfAKindThm\ is not even required here; \uniquenessThm\ alone is sufficient for the purposes of this problem.

\item[M3.]
\phantomsection\label{universalValidationCheckMot}
Extending \LRMCcheckMot, there is no reason to stop with low-rank matrix completion.  \allOfAKindThm\ provides a deterministic validation check for {\em any} algorithm that performs low-rank, or even high-rank matrix completion\cite{aistatsHRMC}, or any algorithm that finds a subspace that fits data, e.g., the EM algorithm derived in \cite{pimentel14}.

\item[M4.] 
Continuing with \universalValidationCheckMot, a universal deterministic validation check on the output of any algorithm opens the door to answering an important open question: the real sample complexity of subspace clustering with missing data\cite{pimentel14}, which is somewhat equivalent to the sample complexity of high-rank matrix completion\cite{aistatsHRMC}.  One can see in \cite{pimentel14} that the gist of this problem is to be able to identify {\em false} subspaces that for some unfortunate circumstances {\em could} deceivingly appear to fit certain data.

\item[M5]
Of course, sometimes even when we know that our data lies in a subspace \----or want to approximate it with a subspace anyway\---- we don't always know the dimension of such subspace \----or the minimum possible dimension of a subspace that approximates it nicely.  \allOfAKindThm\ can be used iteratively to find with certainty the lowest-dimensional subspace or the {\em minimal} union of subspaces that fit certain data.
\end{description}

Not pretending to do a survey on the applications of subspaces, we think these motivations should be enough to give an idea of the scope and power of these results.  With this, we move on.

\pagebreak
\section{Setup}
\label{setupSec}
In this section we fully describe the setup and notation that will be used in the remainder of the document.

\subsection{Subspaces and bases}
\label{subspacesKcharacterizationSec}
Let \phantomsection\label{SstarkDef}$\Sstar=\{\sstark{k}\}_{\k=1}^\KK$ be a set of \phantomsection\label{KKDef}$\KK$ distinct \phantomsection\label{rDef}$\r$-dimensional subspaces of \phantomsection\label{dDef}$\R^\d$.

We use \phantomsection\label{sstarDef}$\sstar$ to denote an arbitrary subspace from $\Sstar$, and \phantomsection\label{UstarDef}$\Ustar$ to denote a basis of $\sstar$, i.e., whenever possible, we drop the subscript $\k$, which is generally used to index subspaces, and unless otherwise stated, runs from $1$ to $\KK$.

\begin{myExample}
\label{subspacesBasesEg}
\normalfont
Let $\d=5$, $\r=2$, and
\begin{align*}
\Ustar = \left[\begin{matrix} 1 & 1 \\ 1 & 2 \\ 1 & 3 \\ 1 & 4 \\ 1 & 5 \end{matrix}\right].
\end{align*}
Then $\sstar=\spn\{\Ustar\}$.
$\blacksquare$
\end{myExample}

In general, we use \phantomsection\label{sDef}$\s$ to denote an arbitrary subspace and \phantomsection\label{UDef}$\U$ to denote one of its bases.

\subsection{Vectors and bases}
\label{vectorsKcharacterizationSec}
Let \phantomsection\label{XDef}$\X:=\{\xii{i}\}_{\i=1}^\N$ be denote a collection of vectors of $\R^\d$ that lies in the union of the subspaces of $\Sstar$.

As we said before, each $\xii{i}$ is assumed to lie in one of the subspaces of $\Sstar$.  This correspondence is described by \phantomsection\label{KDef}$\K=\{\ki{i}\}_{\i=1}^\N$, a multiset of indices in $\{1,...,\KK\}$ that specifies that $\xii{i}$ lies in \phantomsection\label{sstarkiDef}\hyperref[SstarkDef]{$S^\star_{k_i}$}.  To keep notation from getting out of hand, we use \phantomsection\label{sstariDef}$\sstari{i}$ as shorthands for $\sstark{k_i}$, and \phantomsection\label{UstariDef}$\Ustari{i}$ to denote a basis of $\sstari{i}$.

We use \phantomsection\label{xDef}$\x$ to denote an arbitrary element of $\X$ that lies in $\sstar$, i.e., whenever possible, we drop the subscript $\i$, which is generally used to index vectors, and unless otherwise stated, runs from $1$ to $\N$.

\begin{myExample}
\label{vectorsBasesEg}
\normalfont
With the same setup as in \subspacesBasesEg, let $\N=3$ and
\begin{align*}
\xii{1} = \left[\begin{matrix} 2\\ 2\\ 2\\ 2 \\ 2 \end{matrix}\right], \hspace{.5cm}
\xii{2} = \left[\begin{matrix} 3\\ 6\\ 9\\ 12 \\ 15 \end{matrix}\right], \hspace{.5cm}
\xii{3} = \left[\begin{matrix} 2\\ 3\\ 4\\ 5 \\ 6 \end{matrix}\right].
\end{align*}
Then
\begin{align*}
\X = \{\xii{1},\xii{2},\xii{3}\} = \left[\begin{matrix} 2 & 3 & 2 \\ 2 & 6 & 3 \\ 2 & 9 & 4 \\ 2 &12 & 5 \\ 2 & 15 & 6 \end{matrix}\right].
\end{align*}
It is easy to see that $\xii{1}$, $\xii{2}$ and $\xii{3}$ belong to $\sstar$, i.e., $\ki{1}=\ki{2}=\ki{3}$.
$\blacksquare$
\end{myExample}

\subsection{Observation sets}
\label{observationSetsSec}
We are interested on partially observed vectors.  We handle this with \phantomsection\label{OODef}$\OO:=\{\oi{i}\}_{\i=1}^\N$, a set of $\N$ sets that specifies that $\xii{i}$ is only observed in the positions of the set $\oi{i}$.  Since $\xii{i} \in \R^\d$, $\oi{i} \subset \{1,...,\d\}$.

We make two assumptions about the entries we observe, only to simplify the analysis, but these can be most easily generalized:
\begin{enumerate}
\item[(i)]
$|\oi{i}|=\r+1$ for every $\i$.
\item[(ii)]
$\bigcup_\i \oi{i} = \{1,...,\d\}$.
\end{enumerate}
We use \phantomsection\label{oDef}$\o$ to denote an arbitrary subset of $\{1,...,\d\}$ of size $\r+1$, i.e., whenever possible we drop the subscript $\i$.

We also use \phantomsection\label{ODef}$\O$, $\bb{\O}$ and $\bar{\O}$ to denote arbitrary collections of sets of $\o$'s; typically subsets of $\OO$.

\begin{myDefinition}[$\n$, $\m$]
\label{nmDef}
Given $\O$ (resp. $\bb{\O}$ and $\bar{\O}$), we define $\n$ and $\m$ (resp. $\bb{\n}$, $\bb{\m}$ and $\bar{\n}$, $\bar{\m}$) as:
\begin{align*}
\n &:=|\O|, \\
\m &:=| \displaystyle \bigcup_{\o \in \O} \o |,
\end{align*}
\end{myDefinition}

\begin{myDefinition}
\label{IJDef}
Given $\O \subset \OO$, we define $\I:=\{\i:\oi{i} \in \O\}$ and $\J :=\displaystyle \bigcup_{\o \in \O} \o$.
\end{myDefinition}

Observe that $|\O|$ is the number of sets that $\O$ contains, i.e., $|\I|$, for example, $|\OO|=\N$.

Unless otherwise stated, we use \phantomsection\label{jDef}$\j$ to index the elements of $\{1,...,\d\}$, and typically to denote that such element belongs to some set $\o$ or to index an entry of a vector, for example, $\uj{j}$ denotes the $\j^{th}$ entry of $\u$.  This way, intuitively, $\J$ is the set of $\j$'s contained in the sets of $\O$, $\m$ is the number of distinct $\j$'s that are contained in the sets of $\O$, $\n$ is the number of $\o$'s that $\O$ has, and for $\O \subset \OO$, $\I$ is the set of $\i$'s such that $\oi{i} \in \OO$ also belongs to $\O$.

For convenience, rather than listing the set of sets to specify $\O$, we typically use a $\d \times \n$ matrix whose $(\j,\i)^{th}$ entry is {\em observed}, denoted by \phantomsection\label{seeDef}$\see$, if $\j \in \oi{i}$, and {\em missing} otherwise, denoted by \phantomsection\label{missDef}$\miss$ .  When there is no room for confusion, we use $\O$ to denote such matrix.  Under this convention, $\J$ can be thought of as the set of rows with at least one observed entry, and $\m$ as the number of such rows.

\begin{myExample}
\label{observationSetsEg}
\normalfont
With the same setup as in \vectorsBasesEg.  Let $\oi{1}=\{1,2,3\}$, $\oi{2}=\{2,3,4\}$ and $\oi{3}=\{3,4,5\}$.  Then
\begin{align*}
\OO = \{\oi{1},\oi{2},\oi{3}\} = \left[\begin{matrix} \see & \miss & \miss \\ \see & \see & \miss \\ \see & \see & \see \\ \miss & \see & \see \\ \miss & \miss & \see \end{matrix}\right].
\end{align*}
If we let $\O=\{\oi{1},\oi{2}\}$, then $\I=\{1,2\}$, $\J=\{1,2,3,4\}$, $\n=2$ and $\m=4$.
$\blacksquare$
\end{myExample}

\subsection{Incomplete vectors, bases and subspaces}
\label{incompleteVectorsSec}
We are now ready to define incomplete vectors.

\begin{myDefinition}[$\hatx$]
\index{incomplete vector $\hatx$}
\label{hatxDef}
Given $\o$, we define $\hatx$ as the vector with $\d$ components whose $\j^{th}$ entry is equal to the $\j^{th}$ entry of $\x$ if $\j \in \o$, and otherwise has a value of {\em missing}, denoted by $\miss$ (resp. for $\oi{i}$ and $\hatxi{i}$).
\end{myDefinition}

Notice that $\hatx$ depends on $\o$.  Technically, we could specify this by writing $\hat{x}_\omega$, but the index $\o$ is redundant, and we want our notation to be as simple as possible.  For a collection of vectors we simply have \phantomsection\label{hatXDef}$\hatX:=\{\hatxi{i}\}_{\i =1}^\N$.  When there is no room for confusion, we equivalently use $\X$ and $\hatX$ to denote the $\d \times \N$ matrices with $\{\xii{i}\}_{\i=1}^\N$ and $\{\hatxi{i}\}_{\i=1}^\N$ as its columns.

\begin{myDefinition}[$\xo$]
\index{restricted vector $\xo$}
\label{xoDef}
We define $\xo$ as the vector in $\R^{|\o|}$ whose entries are equal to the observed entries of $\hatx$.
\end{myDefinition}

For subspaces we have something similar.

\begin{myDefinition}[$\Rdo{\omega}$, $\hats$ and $\hatU$]
\index{$\Rdo{\omega}$}
\index{projected subspace $\hats$}
\index{projected basis $\hatU$}
\label{hatsDef}
Let $\Rdo{\omega}$ be the span of the canonical vectors of $\R^\d$ corresponding to the elements of $\o$.  We define $\hats$ as the projection of $\s$ onto $\Rdo{\omega}$, and $\hatU$ as the $\d \times \r$ matrix with the entries of $\U$ in the positions of $\o$, and zeros elsewhere (resp. for $\oi{i}$, $\hatsi{i}$ and $\hatUi{i}$).
\end{myDefinition}

It is easy to see that $\spn\{\hatU\}=\hats$.  Conversely, the rows of any basis of $\hats$ must be zero in the positions that don't belong to $\o$.

Similar to $\hatx$, $\hats$ and $\hatU$ depend on $\o$.  Technically, we could specify this by writing $\hat{S}_\omega$ or $\hat{U}_\omega$, but the index $\o$ is redundant, and we want to keep our notation from getting out of hand.  For this same purpose, we use \phantomsection\phantomsection\label{hatsstariDef}$\hatsstari{i}$ as shorthand for \hyperref[sstarkiDef]{$\hat{S}^\star_{k_{i_{\omega_i}}}$}, and similarly for \phantomsection\label{hatUstariDef}$\hatUstari{i}$ and \hyperref[sstarkiDef]{$\hat{U}^\star_{k_{i_{\omega_i}}}$}.

\begin{myDefinition}[$\Uo$ and $\so$]
\index{restricted subspace $\so$}
\index{restricted basis $\Uo$}
\label{soDef}
Given $\o$, we define $\Uo$ as the $|\o| \times \r$ matrix with the non-zero rows of rows of $\hatU$, and $\so:=\spn\{\Uo\}$.
\end{myDefinition}

To simplify our notation, we use \phantomsection\label{sstaroiDef}$\sstaroi{i}$ as shorthand for \hyperref[sstarkiDef]{$S^\star_{k_{i_{\omega_i}}}$}, and similarly for \phantomsection\label{UstaroDef}\hyperref[sstarkiDef]{$U^\star_{\omega_i}$}.

\begin{figure}[H]
     \begin{center}
        \subfigure[$\hats$ is the result of projecting $\s$ onto $\Rdo{\omega}$.]{
           \label{projectionFig}
            \includegraphics[width=5cm]{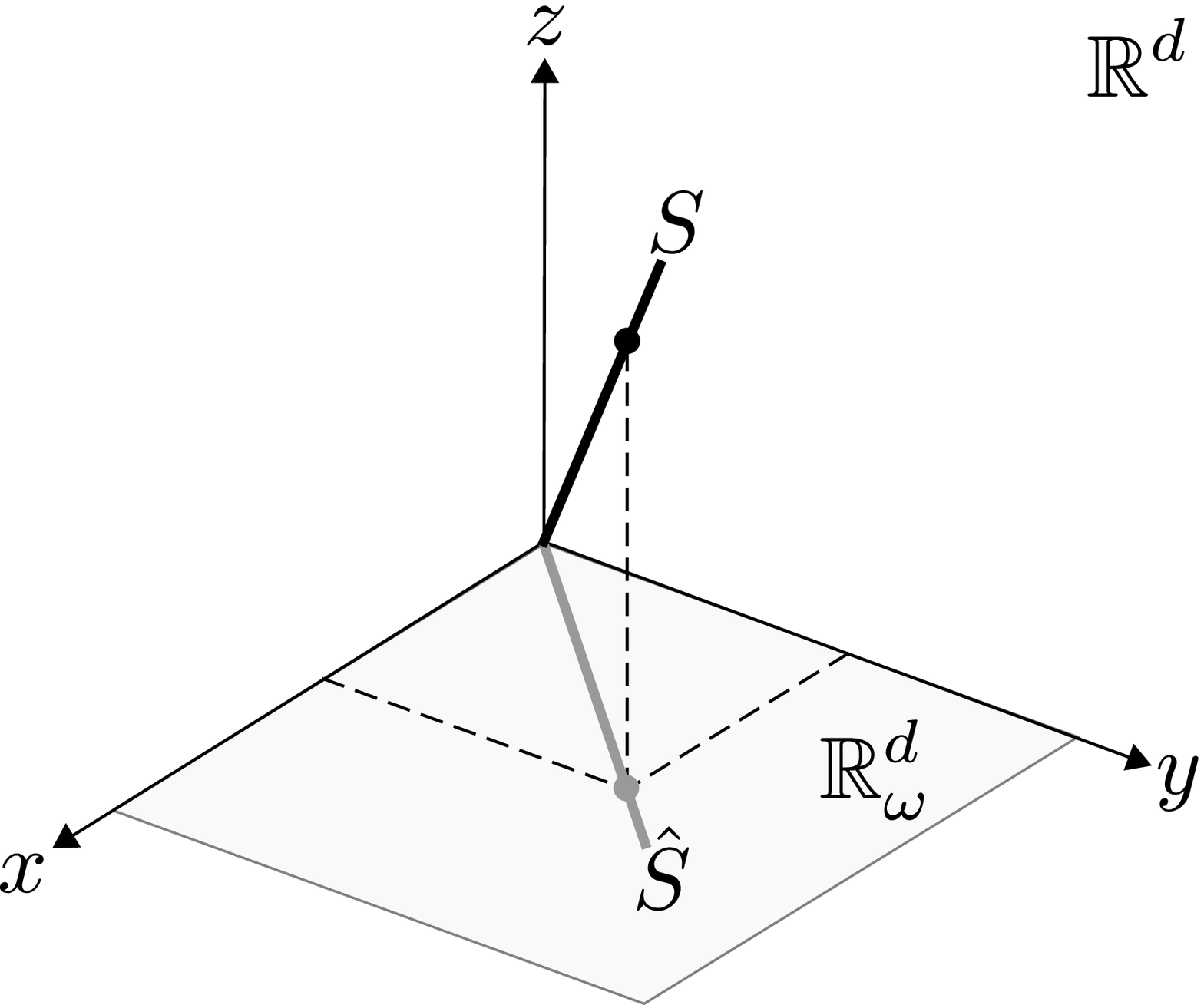}
        } \hspace{.5cm}
        \subfigure[$\so$ is the {\em restriction} of $\s$ to the positions of $\o$; a subspace in $\R^{\r+1}$.]{
           \label{somegaFig}
           \includegraphics[width=5cm]{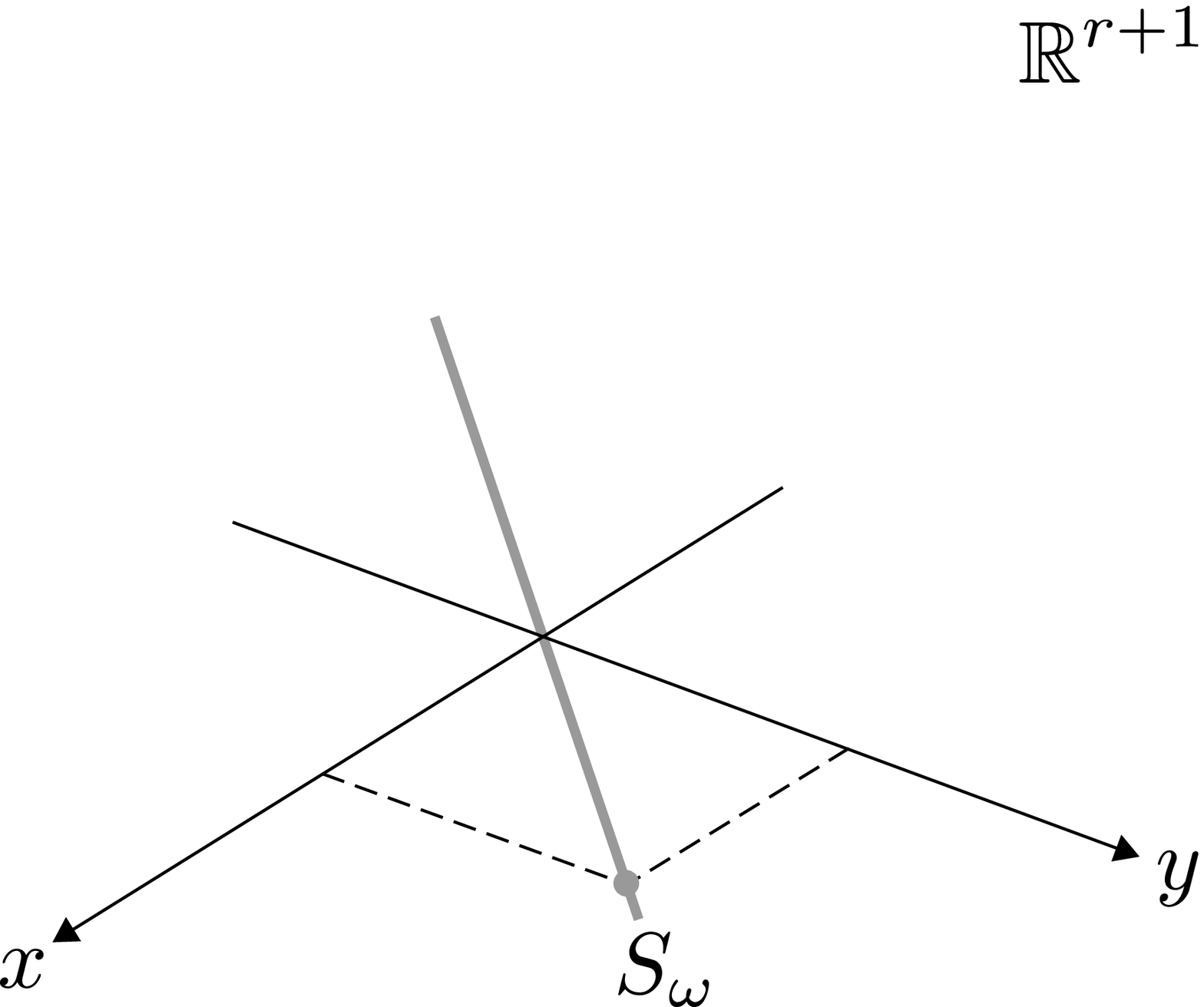}
        }
    \end{center}
    \caption{In this example, $\r=1$, and $\o=\{1,2\}$, so $\Rdo{\omega}$ is the $(x,y)$-plane. $\hats$ is a line in $\R^3$ that lies in the $(x,y)$-plane, while $\so$ is a line in $\R^2$.}
   \label{projectionSomegaFig}
\end{figure}

\begin{myExample}
\label{incompleteVectorsEg}
\normalfont
With the same setup as in Examples \ref{subspacesBasesEg}, \ref{vectorsBasesEg} and \ref{observationSetsEg}, we would obtain
\begin{align*}
\hatxi{1} &= \left[\begin{matrix} 2\\ 2 \\ 2\\ \miss \\ \miss \end{matrix}\right], \hspace{.125cm}
\hatxi{2} = \left[\begin{matrix} \miss \\ 6\\ 9 \\ 12 \\ \miss \end{matrix}\right], \hspace{.125cm}
\hatxi{3} = \left[\begin{matrix} \miss \\ \miss \\ 4 \\ 5 \\ 6 \end{matrix}\right], \hspace{.5cm}
\hatX = \left[\begin{matrix} 2 & \miss & \miss \\ 2 & 6 & \miss \\ 2 & 9 & 4 \\ \miss &12 & 5 \\ \miss & \miss & 6 \end{matrix}\right], \\
\xoi{1} &= \left[\begin{matrix} 2\\ 2 \\ 2\\ \end{matrix}\right], \hspace{.125cm}
\xoi{2} = \left[\begin{matrix} 6\\ 9 \\ 12 \\ \end{matrix}\right], \hspace{.125cm}
\xoi{3} = \left[\begin{matrix} 4 \\ 5 \\ 6 \end{matrix}\right], \\
\hatUi{1}& = \left[\begin{matrix} 1 & 1 \\ 1 & 2 \\ 1 & 3 \\ 0 & 0 \\ 0 & 0 \end{matrix}\right], \hspace{.5cm}
\hatUi{2} = \left[\begin{matrix} 0 & 0 \\ 1 & 2 \\ 1 & 3 \\ 1 & 4 \\ 0 & 0 \end{matrix}\right], \hspace{.5cm}
\hatUi{3} = \left[\begin{matrix} 0 & 0 \\ 0 & 0 \\ 1 & 3 \\ 1 & 4 \\ 1 & 5 \end{matrix}\right], \\
\Uoi{1}& = \left[\begin{matrix} 1 & 1 \\ 1 & 2 \\ 1 & 3 \end{matrix}\right], \hspace{.5cm}
\Uoi{2} = \left[\begin{matrix} 1 & 2 \\ 1 & 3 \\ 1 & 4 \end{matrix}\right], \hspace{.5cm}
\Uoi{3} = \left[\begin{matrix} 1 & 3 \\ 1 & 4 \\ 1 & 5 \end{matrix}\right].
\end{align*}
$\blacksquare$
\end{myExample}

Before we move to other things, one technical definition that will simplify our argumentation greatly without loss of generality.  We will discuss more about this in \textsection \ref{aboutOurAssumptionsSec}.
\begin{myDefinition}[Degenerate subspace]
\index{degenerate subspace}
\label{degenerateDef}
We say an $\r$-dimensional subspace is {\em degenerate} iff there exists an \phantomsection\label{upsDef}$\ups \subset \{1,...,\d\}$ with $|\ups|\leq\r$, such that $\dim S_\ups<|\ups|$.
\end{myDefinition}

\begin{myRemark}
\normalfont
\degenerateDef\ is saying that a subspace is non-\degenerate\ iff for every $\ups \subset \{1,...,\d\}$ with $|\ups|\leq\r$, $\dim S_\ups = |\ups|$, i.e. $S_\ups=\R^{|\ups|}$, or equivalently, iff every $|\ups| \times \r$ matrix formed with the rows of any of its bases is full-rank.  Notice that almost every subspace is non-\degenerate.
$\blacksquare$
\end{myRemark}

\begin{myExample}
\normalfont
Unless otherwise stated, subspaces of all examples in the document are non-\degenerate.  Here is an example of a \degenerate\ one:
\begin{align*}
\U=\left[\begin{matrix}
1 & 0 \\
0 & 1 \\
0 & 1
\end{matrix}\right].
\end{align*}
If we take $\ups=\{2,3\}$, one can verify that $\dim S_\ups=1< 2=|\ups|$ by simply looking at the bottom $2\times2$ minor of $\U$, which is rank-defficient.
$\blacksquare$
\end{myExample}

\subsection{Fitting incomplete vectors}
\label{fittingIncompleteVectorsSec}
Let us now define what it means to {\em \fitxo} an incomplete vector.

\begin{myDefinition}[To fit $\hatx$]
\index{to \fitxo\ $\hatx$}
\label{fittingxoDef}
We say $\s$ {\em fits} $\hatx$ iff there exists a vector in $\s$ that is equal to $\hatx$ in all its observed entries, or equivalently, iff $\xo \in \so$.
\end{myDefinition}

\begin{myDefinition}[To fit $\hatX$]
\index{to \fitXO\ $\hatX$}
\label{fittinghatXDef}
We say $\s$ {\em fits} $\hatX$ iff $\s$ \fitsxo\ $\hatxi{i}$ for every $\i$.
\end{myDefinition}

Notice that each $\hatxi{i}$ might belong to a different subspace in $\Sstar$.

\begin{myExample}
\normalfont
With the same setup as Examples \ref{prologueaEg} and \ref{prologuebEg}, it is easy to see that $\s$ \fitsXO\ $\hatX$.
$\blacksquare$
\end{myExample}

\subsection{Fitting generic vectors $\sim$ fitting observation sets}
\label{fittingSetsSec}
\index{to \fitGen\ a generic vector}
We now formalize what we mean by fitting generic vectors.  Intuitively, when we say that a vector is generic, we mean that it could be {\em any} vector, and whenever we say that $\s$ \fitsGen\ a generic vector from $\sstar$, what we formally mean is that $\s$ fits {\em every} vector from $\sstar$.

The same for an incomplete vector.  Moreover, with \fittingxoDef, it is easy to see that whether or not $\s$ fits {\em every} $\hatx \in \sstar$ depends only on $\o$.  Therefore, whenever we informally say that $\s$ \fitsGen\ a generic $\hatx$, what we formally mean is that $\s$ \fitso\ $\o$ in the following sense.

\begin{myDefinition}[To fit $\o$]
\index{to \fito\ $\o$}
\label{fittingoDef}
Given $\sstar$, we say that $\s$ {\em fits} $\o$ iff $\so$ \fitsxo\ {\em every} $\xo \in \sstaro$.
\end{myDefinition}

\begin{myRemark}
\label{fittingoRmk}
\normalfont
Notice that $\s$ will \fito\ $\o$ iff $\sstaro \subset \so$, i.e., iff for every $\x \in \sstar$ there exists a $\u \in \s$ such that $\uo=\xo$.

In general, $\s$ {\em \fitso} $\o$ iff $\sstaro \subset \so$.  Nevertheless since $\sstar$ is non-\degenerate, $\sstaro$ is an $\r$-dimensional subspace, so whenever $\s$ is also an $\r$-dimensional subspace, we also have $\so \subset \sstaro$, whence $\s$ {\em \fitso} $\o$ iff $\so=\sstaro$.
$\blacksquare$
\end{myRemark}

\begin{myExample}
\label{fittingoEg}
\normalfont
Let $\r=1$ and suppose
\begin{align*}
\U = \left[\begin{matrix} 2\\ 2\\ 3 \end{matrix}\right], \hspace{.5cm}
\Ustar = \left[\begin{matrix} 1\\ 1\\ 1 \end{matrix}\right], \hspace{.5cm}
\o = \left[\begin{matrix} \see \\ \see \\ \miss \end{matrix}\right].
\end{align*}
Since $\s$ would \fitxo\ any $\hatx \in \sstar$, we say $\s$ \fitso\ $\o$.
$\blacksquare$
\end{myExample}

Similarly, when we say that a set of vectors $\X \in \Sstar$, consisting of $\xii{1} \in \sstari{1},...,\xii{\N} \in \sstari{\N}$, is generic, we mean that $\xii{1}$ could be {\em any} vector from $\sstari{1}$, $\xii{2}$ could be {\em any} vector from $\sstari{2}$, and so on.  Whenever we say that $\s$ \fitsGen\ a generic set of vectors $\X \in \Sstar$, what we formally mean is that $\s$ fits {\em every} $\xii{i} \in \sstari{i}$ for every $\i$.

The same for sets of incomplete vectors.  With \fittinghatXDef\ it is easy to see that whether or not $\s$ fits {\em every} $\hatX \in \Sstar$ depends only on $\OO$.  Therefore, whenever we informally say that $\s$ \fitsGen\ a generic $\hatX$, what we formally mean is that $\s$ \fitsO\ $\OO$, with the following.

\begin{myDefinition}[To \fitO\ $\O$]
\index{to \fitO\ $\O$}
\label{fittingODef}
Given $\Sstar$ and $\K$, we say that $\s$ {\em fits}\ $\O$ iff $\s$ \fitsxo\ $\hatxi{i}$ for {\em every} $\xii{i} \in \sstari{i}$ and {\em every} $\i \in \I$.
\end{myDefinition}

\begin{myRemark}
\label{fittingORmk}
\normalfont
Recall that $\sstari{i}$ is a shorthand for $\sstark{k_i}$ and $\ki{i}$ is the index in $\K$ that specifies that $\xii{i}$ lies in $\sstark{k_i}$.  Hence the dependency on $\K$ in \fittingODef.
$\blacksquare$
\end{myRemark}

Notice that there are many equivalent ways of defining what it means to \fitO\ $\O$.  For example, we could also say that $\s$ {\em fits}\ $\O$ iff $\soi{i}$ \fitsxo\ $\xoi{i}$ for {\em every} $\xoi{i} \in \sstaroi{i}$ and {\em every} $\i \in \I$, or we could define it as in the next Remark.

\begin{myRemark}
\normalfont
\label{fittingOIsFuncRmk}
In general, $\s$ {\em \fitso} $\OO$ iff $\sstaroi{i} \subset \soi{i}$ for every $\i$.  Nevertheless, since $\sstar$ is non-\degenerate, $\sstaroi{i}$ is an $\r$-dimensional subspace, so whenever $\s$ is also an $\r$-dimensional subspace, we also have that $\soi{i} \subset \sstaroi{i}$ for every $\i$, whence $\s$ {\em \fitsO} $\OO$ iff $\soi{i}=\sstaroi{i}$ for every $\i$.
$\blacksquare$
\end{myRemark}

\begin{myExample}
\label{fittingOEg}
\normalfont
Suppose $\r=1$ and
\begin{align*}
\U = \left[\begin{matrix} 2 \\ 2 \\ 3 \end{matrix}\right], \hspace{.5cm}
\Ustari{1} = \left[\begin{matrix} 1 \\ 1 \\ 1 \end{matrix}\right], \hspace{.5cm}
\Ustari{2} = \left[\begin{matrix} 1 \\ 2 \\ 3 \end{matrix}\right], \hspace{.5cm}
\OO = \left[\begin{matrix} \see & \see \\ \see & \miss \\ \miss & \see \end{matrix}\right].
\end{align*}
Since $\s$ would \fitxo\ every $\hatxi{1} \in \sstari{1}$ {\em and} every $\hatxi{2} \in \sstari{2}$, we say $\s$ \fitsO\ $\OO$.
$\blacksquare$
\end{myExample}

\begin{myRemark}
\normalfont
When we informally say that an $\r$-dimensional subspace $\s$ \fitsGen\ a {\em generic} $\hatx$, what we formally mean is that $\s$ \fitso\ $\o$, i.e., that $\s$ \fitsxo\ $\hatx$ for {\em every} $\x \in \sstar$.  This guarantees that $\s$ is somehow independent of one particular instance of $\hatx$.  This is essential for our analysis, because if $\s$ is to \fito\ $\o$, then $\so$ must \fitxo\ $\xo$ for {\em every} $\xo \in \sstaro$.  This implies that $\s$ must satisfy $\so=\sstaro$ (see \fittingoRmk).

In contrast, $\s$ need not satisfy this to fit {\em one particular} $\hatx$. For instance, $\s$ could fit {\em one particular} $\hatx$ by just fixing the observed entries of $\hatx$ in the positions of $\o$ of a spanning vector of $\s$, e.g., with the same setup as in \subspacesBasesEg, let
\begin{align*}
\x = \left[\begin{matrix} 2\\ 3\\ 4\\ 5 \\6 \end{matrix}\right], \hspace{.5cm}
\o = \left[\begin{matrix} \see \\ \see \\ \see \\ \miss \\ \miss \end{matrix}\right], \hspace{.5cm}
\hatx = \left[\begin{matrix} 2\\ 3\\ 4\\ \miss \\ \miss \end{matrix}\right].
\end{align*}
Then we could construct
\begin{align*}
\U = \left[\begin{matrix} 2 & \uj{12} \\ 3 & \uj{22} \\ 4 & \uj{32} \\ \uj{41} & \uj{42} \\ \uj{51} & \uj{52} \end{matrix}\right], \hspace{.5cm}
\end{align*}
and $\U$ would \fitxo\ $\hatx$ for any choices of $\uj{12}, \uj{22}$ and $\uj{32}$, so $\s$ would {\em need not} satisfy $\so=\sstaro$ to \fitxo\ $\hatx$.

Similarly, when we informally say that an $\r$-dimensional subspace $\s$ \fitsGen\ a {\em generic} $\hatX$, what we formally mean is that $\s$ \fitsO\ $\OO$, i.e., that $\s$ \fitsxo\ $\hatxi{i}$ for {\em every} $\xii{i} \in \sstari{i}$ {\em and every} $\i$.  This guarantees that $\s$ is somehow independent of one particular instance of $\hatX$.  This is essential for our analysis, because $\s$ must satisfy $\soi{i}=\sstaroi{i}$ for every $\i$ in order to \fitO\ $\OO$.

$\blacksquare$
\end{myRemark}

\pagebreak
\section{Assumptions}
\label{assumptionsSec}
Now that we have fully specified our setup, we use this section to give a detailed, precise and unified list of our assumptions, in order to present them all together and clearly, to emphasize how lenient they are, how they are mostly used to ease our argumentation, and how easily they can be generalized.

\begin{description}
\item[A1.]
\phantomsection\label{unionSubspacesAss}
$\xii{i} \in \sstari{i}$ for every $\i$.  This is just the basic setup of the problem: that our data lie in a union of subspaces.  This is to give some structure to $\X$.  Without this assumption, we cannot possibly hope to determine the missing values in $\hatX$, as they could be arbitrary, whence nothing can be said about the complete counterpart $\X$, and no low-dimensional subspace can be guaranteed to fit $\X$.

\item[A2.]
\phantomsection\label{rDimensionAss}
All subspaces in $\Sstar$ are assumed to be $\r$-dimensional subspaces of $\R^\d$.  This is just to simplify our arguments, but can be easily generalized to the case where the subspaces in $\Sstar$ are of different dimensions (see \textsection \ref{aboutOurAssumptionsSec}).  Observe that we are not assuming anything about $\KK$.  For all we know, $\KK$ could be arbitrarily large; even larger than $\N$.

\item[A3.]
\phantomsection\label{sstarNonDegenerateAss}
All subspaces in $\Sstar$ are assumed to be non-\degenerate.   This is just to simplify our arguments, but can be easily generalized, {\em if necessary}.  We emphasize, {\em if necessary}, because fortunately, the set of \degenerate\ subspaces has measure zero, thus our results hold, without any further modification, for almost every $\Sstar$.  For a further discussion about \degenerate\ subspaces see \textsection \ref{aboutOurAssumptionsSec}.

\item[A4.]
\phantomsection\label{sizeoAss}
$|\oi{i}|=\r+1$ for every $\i$.  Observe that if $|\o| \leq \r$, there is no possible way to determine the subspace where $\x$ really lies.  More precisely, since subspaces in $\Sstar$ are non-\degenerate, a vector observed in fewer than $\r+1$ entries could belong to any of the subspaces in $\Sstar$.

On the other hand, if $|\o|>\r+1$ it can only be easier to determine if $\x$ really belongs to $\s$, as any subspace that \fitsxo\ it will have to satisfy more restrictions.  In other words, it is harder to \fitxo\ $\hatxi{1}$ than to \fitxo\ $\hatxi{2}$ if $|\oi{1}|>|\oi{2}|$.  This assumption can be immediately generalized using this simple observation.

This way, rather than an assumption \----being $|\o|>\r$ a requirement for the task that we want to achieve\---- this is just a convenience statement to simplify our arguments, analysis and notation, that at the same time states that we are working under the most minimal assumptions on $|\o|$.  More about this is discussed in \textsection \ref{aboutOurAssumptionsSec}.

Notice that we assume nothing about $\OO$ being spread uniformly, at random, or anything of the sort; our results apply to completely arbitrary $\OO$'s.

\item[A5.]
\phantomsection\label{unionOOAss}
$\bs{\J}:=\bigcup_\i \oi{i} = \{1,...,\d\}$.  This is just for simplicity of notation and argumentation.  Observe that if $\OO$ has a fully unobserved row, there will always be infinitely many subspaces that \fitO\ $\OO$.  We can only determine when the projection of such subspaces onto $\Rdo{\bs{\J}}$ will be unique, i.e., when the \hyperref[soDef]{restriction} of a subspace that \fitsO\ $\OO$, is unique.  We can easily generalize this by working only with the observed entries of $\OO$, and alternatively defining $\{1,...,\d\}:=\bs{\J}$, as we can say nothing anyway about the entries where no row is observed.

\item[A6.]
\phantomsection\label{existsNonDegenerateAss}
There exists a non-\degenerate\ $\r$-dimensional subspace that \fitsO\ $\OO$.  Rather than an assumption, this is the motivation of the document: assuming that there is a subspace that fits certain data, we want to determine when such data really {\em lies} in a subspace.  Since we are assuming that subspaces in $\Sstar$ are non-\degenerate, we know that if our data really lies in a subspace, it is a non-\degenerate\ one.  If certain \degenerate\ subspace fits our data, we would already know our data does not really lies in it. We use this assumption mainly to avoid all this uninteresting argumentation in every statement.
\end{description}

\pagebreak
\section{Results}
\label{resultsSec}
To avoid being anal in every statement of the document without giving up being precise, we would like to give an important remark about our results before presenting them:
\\
\\
\noindent {\bf All our results hold under assumptions A1-A6 above, and for almost every $\Sstar$ under the natural Lebesgue measure.}
\\
\\
In other words, our results hold for almost every dataset lying in a union of subspaces.

Having said that, we are now ready to present our results.  We would like to take advantage of this section to also give some intuitive meaning to them, to emphasize how simple they are, and to give an example of their usage.

Recall that the main goal is to determine how can we make sure that if certain data {\em fit} in a subspace $\s$, it is because such data really {\em lie} in a subspace.

As discussed in \textsection \ref{introSec}, we can answer this question by characterizing when a set of incomplete vectors observed only in $\OO$, \behaves\ as a complete one.  Such characterization, intuitively stated in \textsection \ref{introSec} is formalized by our main result:

\begin{framed}
\begin{myTheorem}[All of a kind]
\label{allOfAKindThm}
Suppose an $\r$-dimensional subspace $\s$ \fitsO\ $\OO$, and that there exists an $\bb{\O} \subset \OO$ of size $\d-\r+1$ such that
\begin{align}
\label{allOfAKindThmEq}
\m \geq \n + \r \hspace{1cm} \text{$\forall$ $\O \subsetneq \bb{\O}$}.
\end{align}

Then $\ki{i}=\ki{\ii}$ for every $(\i,\ii)$.  Furthermore, $\s$ is the only $\r$-dimensional subspace that \fitsO\ $\OO$, and $\s \in \Sstar$.

Conversely, if no such $\bb{\O}$ exists, $\ki{i}$ {\em might} be different from $\ki{\ii}$ for some $(\i,\ii)$, there {\em could} be infinitely many $\r$-dimensional subspaces that \fitO\ $\OO$, and even if there is only one such subspace, it {\em might} not even belong to $\Sstar$.
\end{myTheorem}
\end{framed}

Before moving on, we give some intuitive interpretations of \allOfAKindThm:

\begin{itemize}
\item[\aProp]
That there is only one subspace that \fitsO\ $\OO$ means that there is only one subspace that \fitsGen\ a set of generic incomplete vectors $\hatX$.
\item[\bProp]
That $\ki{i}=\ki{\ii}$ for every $(\i,\ii)$ means that all the vectors of $\hatX$ lie in the same subspace of $\Sstar$.
\end{itemize}
In other words, \allOfAKindThm\ is telling us precisely what we wanted: when will a generic $\hatX$ satisfy the desired properties \aProp\ and \bProp\ from \textsection \ref{insightSec}, and hence when it will \behave\ as a complete vector.

But not only that.  Notice that the converse is telling us that if $\OO$ does not satisfy the conditions of the theorem, then it cannot be guaranteed that all the vectors of $\hatX$ lie in the same subspace of $\Sstar$, whence $\hatX$ will not \behave\ as a complete vector.

In conclusion, \allOfAKindThm\ is telling us that a generic $\hatX$ will \behave\ as a complete generic vector iff $\OO$ satisfies the condition of the theorem, namely that it contains a set $\bb{\O}$ of size $\d-\r+1$ that satisfies \eqref{allOfAKindThmEq}.  One simple and intuitive interpretation of \eqref{allOfAKindThmEq} is that

\vspace{.3cm}
\begingroup
\leftskip1.5em
\rightskip\leftskip
\noindent
{\em For every strict subset of $\bb{\O}$ with $\n$ {\em columns}, the number of distinct {\em rows} with at least one observation, $\m$, is at least $\n+\r$.}
\par
\endgroup
\vspace{.3cm}

\begin{myExample}
\label{mainThmEg}
\normalfont
Continuing with \introEg, we have that
\begin{align*}
\OO=\left[\begin{matrix}
\see & \miss & \miss & \see & \see \\
\see & \see & \miss & \miss & \see \\
\see & \see & \see & \miss & \miss \\
\miss & \see & \see & \see & \see \\
\miss & \miss & \see & \see & \miss \\
\end{matrix}\right].
\end{align*}
It is easy to see that $\OO$ satisfies the conditions of \allOfAKindThm.  Explicitly, take $\bb{\O} = \{\oi{1},...,\oi{4}\}$.  One may verify that $\bb{\O}$ satisfies \eqref{allOfAKindThmEq}.  If $\s$ \fitsO\ $\OO$, then $\ki{i}=\ki{\ii}$ for every $(\i,\ii)$ and $\s=\sstari{i}$ is the only $\r$-dimensional subspace that \fitsO\ $\OO$.

For an example of the converse statement of \allOfAKindThm, consider the same setup as in \allOfAKindEga.  Observe that there exists no $\bb{\O} \subset \OO$ that satisfies the conditions of \allOfAKindThm.  Thus the columns of $\X$ {\em might} not belong to the same subspace, which is precisely the case.  We can also see that $\s$ is none of the subspaces in $\Sstar$.
$\blacksquare$
\end{myExample}

As we mentioned in \textsection \ref{theEssenceSec}, determining when there is only one $\r$-dimensional subspace that \fitsO\ $\OO$ is essential for the proof of \allOfAKindThm.  The answer to this is given by our second main result.

\begin{framed}
\begin{myTheorem}[Characterization of uniqueness]
\label{uniquenessThm}
There is {\em only one} $\r$-dimensional subspace that \fitsO\ $\OO$ iff there exists an $\bb{\O} \subset \OO$ of size $\d-\r$ such that
\begin{align}
\label{uniquenessThmEq}
\m \geq \n + \r \hspace{1cm} \text{$\forall$ $\O \subset \bb{\O}$}.
\end{align}
\end{myTheorem}
\end{framed}

Notice that the requirement of \allOfAKindThm\ is slightly stronger than the requirement of \uniquenessThm: \allOfAKindThm\ requires $\OO$ to contain a set $\bb{\O}$ of size $\d-\r+1$ that satisfies \eqref{uniquenessThm} for every one of its $\d-\r$ subsets of size $\d-\r$, while \uniquenessThm\ only requires one $\bb{\O}$ of size $\d-\r$ that satisfies \eqref{uniquenessThm}.

In other words, once we know that there is only one $\r$-dimensional subspace that \fitsO\ $\OO$, we only need a little bit more to make sure that all the vectors of $\X$ indeed lie in the same subspace.

\begin{myExample}
\label{uniquenessEg}
\normalfont
With the same setup as in \fittingOEg, it is clear that $\OO$ is the only possible set of size $\d-\r$, so take $\bb{\O} = \OO$.  One can trivially verify that $\m \geq \n+\r$ for every $\O \subset \OO$, so $\OO$ satisfies the conditions of \allOfAKindThm, hence there is only one $\r$-dimensional subspace that \fitsO\ $\OO$.

Nevertheless, observe that $\OO$ does not satisfy the conditions of \allOfAKindThm, as it does not even have $\d-\r+1$ sets.

This is one case where we can guarantee that there is only one $\r$-dimensional subspace that \fitsO\ $\OO$, but we cannot yet guarantee that $\k_{i}=\k_{\ii}$ for every $(\i,\ii)$, nor that the subspace that \fitsO\ $\OO$ belongs to $\Sstar$.

Notice that the difference between this $\OO$ and the $\OO$ from \allOfAKindEgb, where we can guarantee that $\k_{i}=\k_{\ii}$ for every $(\i,\ii)$, is just one $\o$.
$\blacksquare$
\end{myExample}

\pagebreak
\section{Uniqueness}
\label{uniquenessSec}
Determining when there is {\em only one} $\r$-dimensional subspace that \fitsO\ $\OO$ is essential for the proof of \allOfAKindThm.  The answer is given by \uniquenessThm.  We will thus start with the proof of the later, and will leave the proof of the former for \textsection \ref{allOfAKindSec}.

\subsection{The subspace $\UU$}
\label{setSSec}

With no further ado, we begin our analysis.  Consider the set:
\[
\phantomsection
\label{UUDef}
\UU:=\left\{\u \in \R^\d: \bigcap_{\i \in \I} \{\uoi{i} \in \sstaroi{i}\} \right\}.
\]
Observe that $\UU$ is a function of $\Sstar$, $\K$ and $\O$ (see \fittingOIsFuncRmk).

In order to show \uniquenessThm\ we will prove that $\UU$ is a subspace, that it contains all the $\r$-dimensional subspaces that \fitO\ $\O$, that it \fitsO\ $\O$, and will determine its dimension.  We will then conclude that there is only one $\r$-dimensional subspace that \fitsO\ $\OO$ iff $\dim \UUO=\r$.

We begin our work towards these goals.  Similar to the definition of $\UU$, let
\[
\phantomsection
\label{uuDef}
\uu:=\{\u \in \R^\d: \uo \in \sstaro\}.
\]
In other words, $\uu$ is the set of all $\u$'s such that $\uo$ \fitsxo\ in $\sstaro$.  Notice that $\uu$ is a function of $\sstar$ and $\o$.

\begin{myLemma}
\label{SSOsubspaceLem}
$\uu$ and $\UU$ are subspaces.
\end{myLemma}
\begin{proof}
Let $\u, v \in \uu$, and $w=\varsigma\u+\zeta v$ for some scalars $\varsigma, \zeta$.  Since $\sstaro$ is a subspace, $w_\o \in \sstaro$, so $w \in \uu$, hence $\uu$ is a subspace.

It is easy to see that $\UU=\bigcap_{\o \in \O} \uu$.  Since intersections of subspaces are subspaces, we conclude that $\UU$ is a subspace.
\end{proof}

\begin{myLemma}
\label{UUcontainsAllSLem}
$\uu$ contains all the $\r$-dimensional subspaces that \fito\ $\o$, and $\UU$ contains all the $\r$-dimensional subspaces that \fitO\ $\O$.
\end{myLemma}
\begin{proof}
Let $\s$ be an $\r$-dimensional subspace that \fitso\ $\o$, and $\u \in \s$.  Since $\so=\sstaro$ (see \fittingoRmk), $\uo \in \sstaro$, so $\u \in \uu$.  Since $\u$ was arbitrary, we have that $\s \subset \uu$.  Since $\s$ was arbitrary, we conclude that $\uu$ contains all the $\r$-dimensional subspaces that \fito\ $\o$.

Now let $\s$ be an $\r$-dimensional subspace that \fitsO\ $\O$, and $\u \in \s$.  Since $\soi{i}=\sstaroi{i}$ for every $\i \in \I$ (see \fittingOIsFuncRmk), $\uoi{i} \in \sstaroi{i}$ for every $\i \in \I$, hence $\u \in \UU$.  Since $\u$ was arbitrary, we have that $\s \subset \UU$.  Since $\s$ was arbitrary, we conclude that $\UU$ contains all the $\r$-dimensional subspaces that \fitO\ $\O$.
\end{proof}

\begin{myCorollary}
\label{UUfitsOCor}
$\uu$ \fitso\ $\o$.  Furthermore, $\UU$ \fitsO\ $\O$ whenever there is an $\r$-dimensional subspace that \fitsO\ $\O$.
\end{myCorollary}
\begin{proof}
$\sstar$ is an $\r$-dimensional subspace that clearly \fitso\ $\o$, so by \UUcontainsAllSLem, $\sstar \subset \uu$, which implies $\uu$ \fitso\ $\o$ (see \fittingoRmk).

Now suppose there is an $\r$-dimensional subspace $\s$ that \fitsO\ $\O$, such that $\sstaroi{i} = \soi{i}$ for every $\i \in \I$ (see \fittingOIsFuncRmk).  By \UUcontainsAllSLem, $\s \subset \UU$, which implies $\UU$ \fitso\ $\O$.
\end{proof}

\begin{myRemark}
\normalfont
Notice the importance of the requirement that $\s$ is an $\r$-dimensional subspace in \UUcontainsAllSLem\ and \UUfitsOCor.

In the case when $\s$ is assumed to be an $\r$-dimensional subspace that \fitsO\ $\OO$, since $\soi{i}=\sstaroi{i}$ for every $\i$ (see \fittingOIsFuncRmk), we know that every $\u \in \s$ will satisfy $\uoi{i} \in \sstaroi{i}$ for every $\i$, hence $\s \subset \UUO$.

If we drop the assumption that $\s$ is $\r$-dimensional, we only know that $\sstaroi{i} \subset \soi{i}$ for every $\i$, whence we only know that there exist $\u$'s in $\s$ that will satisfy $\uoi{i} \in \sstaroi{i}$ for every $\i$, but there might also be some other $\u$'s in $\s$ that won't satisfy this, whence we cannot conclude that $\s \subset \UUO$.  For example, if $\s=\R^\d$, it is clear that $\s$ will \fitO\ $\OO$, but won't be contained in $\UUO$.
$\blacksquare$
\end{myRemark}

Observe that if there are more than one $\r$-dimensional subspaces that \fitO\ $\O$, since they are contained in $\UU$, $\dim\UU>\r$.  Moreover, since $\UU$ is a subspace, if $\dim \UU > \r$, there are infinitely many $\r$-dimensional subspaces contained in $\UU$.  It is easy to see that infinitely many of such subspaces will \fitO\ $\O$.  In other words, using these results, we have the following.
\begin{myCorollary}
\label{dimUUrCor}
There is only one $\r$-dimensional subspace that \fitsO\ $\O$ iff $\dim \UU=\r$.  Furthermore, if there are more than one $\r$-dimensional subspaces that \fitO\ $\O$, there are infinitely many.
\end{myCorollary}

\subsection{One at a time}
\label{oneAtATimeSec}
We continue our analysis by studying $\uu$.  The key idea is that every hyperplane, i.e., every $(\d-1)$-dimensional subspace of $\R^\d$, is characterized by its orthogonal direction.

\begin{figure}[H]
\centering
\includegraphics[width=4cm]{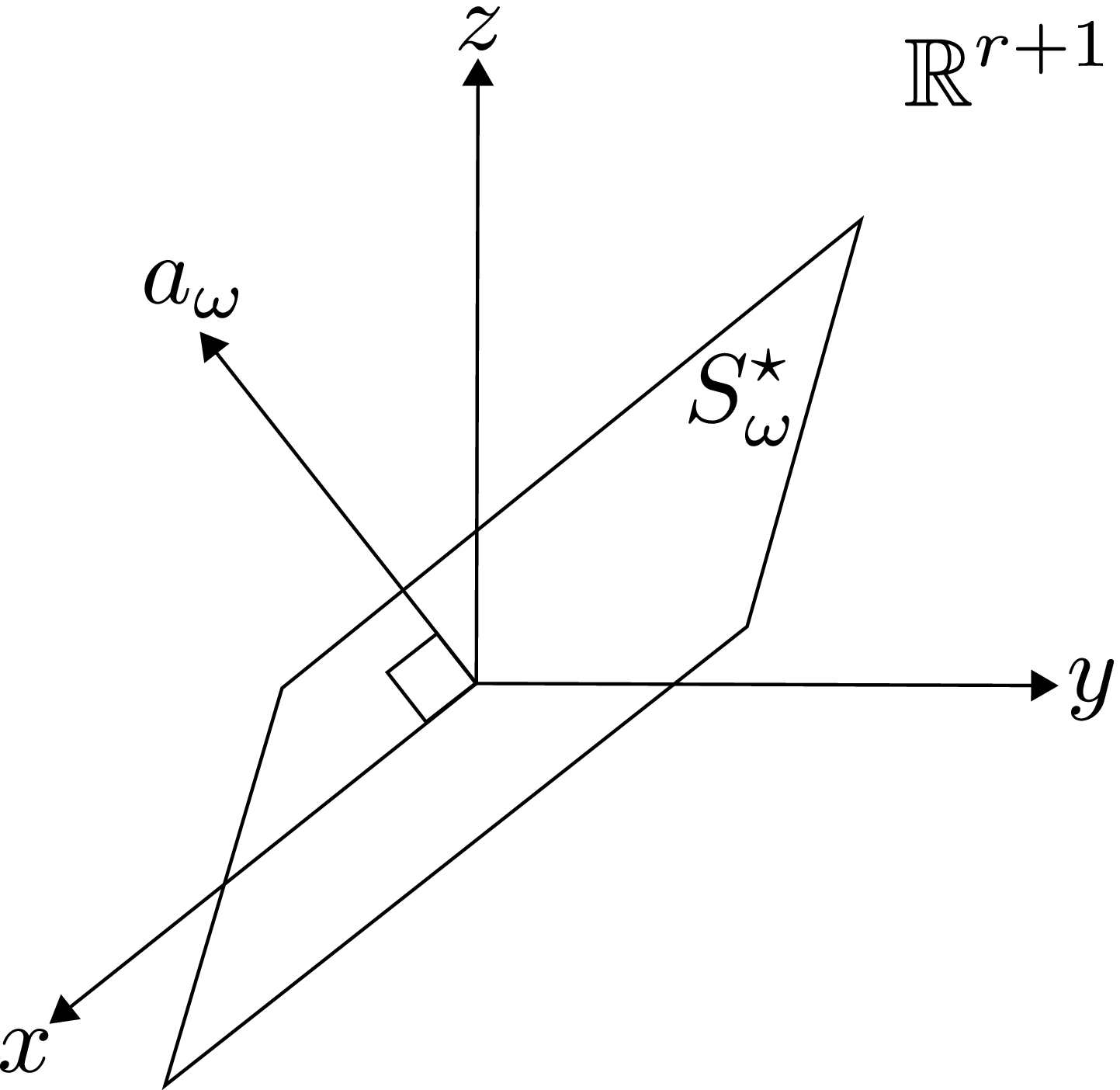}
\caption{Every hyperplane is characterized by its orthogonal direction.}
\label{orthogonalVectorFig}
\end{figure}

By definition, $\sstaro$ is a subspace of $\R^{\r+1}$.  Since $\dim\sstar=\r$, and $\sstar$ is non-\degenerate, it is easy to see that $\dim\sstaro=\r$, i.e., $\sstaro$ is a hyperplane in $\R^{\r+1}$. As such, it is characterized by its orthogonal direction, which can be fully specified by a unique \----up to scalar multiplication\---- non-zero vector of $\R^{\r+1}$ orthogonal to $\sstaro$, say \phantomsection\label{aoDef}$\ao$.  For such $\ao$ we have:
\begin{align*}
\sstaro = \{\uo \in \R^{\r+1}:\< \uo, \ao \> = 0\} = \ker \aoT.
\end{align*}
If we define \phantomsection\label{aDef}$\a$ as the row vector in $\R^\d$ with the entries of $\ao$ in the positions of $\o$, and zeros elsewhere, it is clear that
\begin{align*}
\uu = \{\u \in \R^\d:\< \uo, \ao \> = 0\} = \{\u \in \R^\d:\< \u, \a \> = 0\} = \ker \a.
\end{align*}
Thus $\uu$ is also a hyperplane: the $(\d-1)$-dimensional subspace of $\R^\d$ characterized by $\a$, containing every $\u \in \R^\d$ that satisfies $\uo \in \sstaro$.

Notice that just as $\uu$, $\a$ is a function of $\sstar$ and $\o$, and since $\sstar$ is non-\degenerate,

\begin{myLemma}
\label{aEntriesLem}
$\a$ has exactly $\r+1$ non-zero entries in the positions of $\o$.
\end{myLemma}
\begin{proof}
Suppose for contradiction that $\ao$ has at least one zero entry.  Use $\ups \subset \{1,...,\d\}$ to denote the set of the position of the non-zero entries of $\ao$, and use it analogous to $\o$.  Since $\ao$ is orthogonal to $\sstaro$, we have that $\<\ao,\uo\>=\<\a_\ups,u_\ups\>=0$ for every $u_\ups \in S^\star_\ups$, i.e., $S^\star_\ups \subset \ker \a_\ups$.  Then
\begin{align*}
\dim S^\star_\ups \leq \dim \ker \a^\T_\ups = |\ups|-1 < |\ups|,
\end{align*}
which is a contradiction, as $\sstar$ is non-\degenerate\ by \sstarNonDegenerateAss.  The statement follows directly by the definition of $\a$ as the row vector in $\R^\d$ with the entries of $\ao$ in the positions of $\o$.
\end{proof}

\subsection{Several at Once}
\label{severalAtOnceSec}
Our next step is precisely the most obvious one.  Let \phantomsection\label{aoiDef}$\aoi{i} \in \R^{\r+1}$ be a vector in the orthogonal direction of $\sstaroi{i}$, and \phantomsection\label{aiDef}$\ai{i}$ be the row vector in $\R^\d$ with the entries of $\aoi{i}$ in the positions of $\oi{i}$, and zeros elsewhere.  With \phantomsection\label{uuiDef}$\uui{i}$ as shorthand for \hyperref[uuDef]{$\mathscr{U}_{\omega_i}$}, then
\begin{align*}
\UU=\bigcap_{\i \in \I} \uui{i} = \bigcap_{\i \in \I} \ker \ai{i} = \ker \A,
\end{align*}
where \phantomsection\label{ADef}$\A$  is the $\n \times \d$ matrix with $\{\ai{i}\}_{\i \in \I}$ as its rows.  In the particular case when $\O=\OO$, we use $\AA$ instead of $\A$, i.e., \phantomsection\label{AADef}$\AA$ is the $\N \times \d$ matrix with rows $\{\ai{i}\}_{\i =1}^\N$.  Notice that by \aEntriesLem, $\A$ has exactly $\m$ non-zero columns in the positions of $\J$, while $\AA$ has no zero columns by \unionOOAss.

This way, each $\ai{i}$ defines a hyperplane with all the vectors $\u \in \R^\d$ that satisfy $\uoi{i} \in \sstaroi{i}$, namely $\uui{i}$.  $\UU$ is the intersection of all such hyperplanes, hence it contains all the vectors $\u \in \R^\d$ that satisfy $\uoi{i} \in \sstaroi{i}$ simultaneously for every $\i$.

\begin{figure}[H]
\centering
\includegraphics[width=4.5cm]{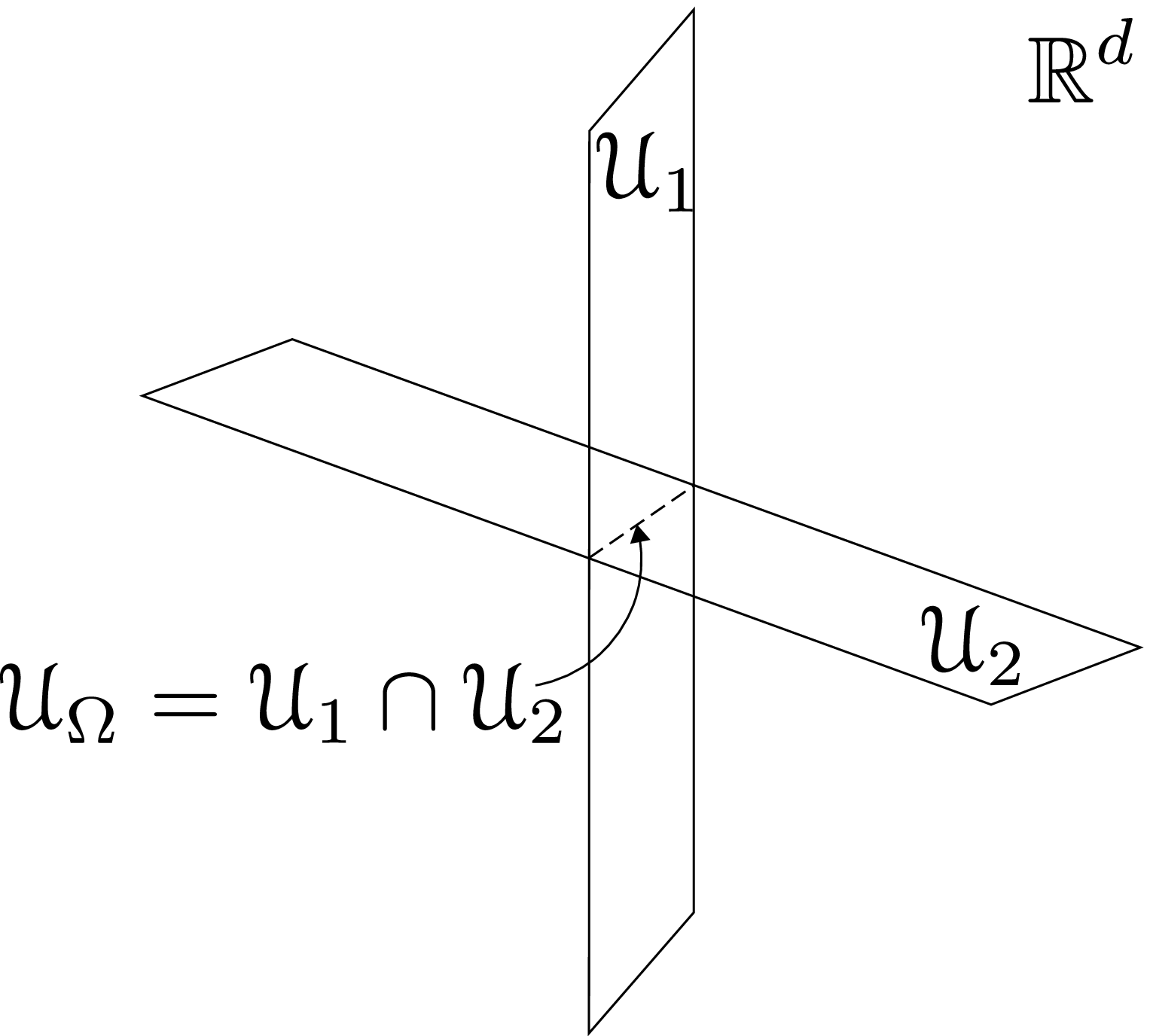}
\caption{$\UU = \ker \A  = \bigcap_{\i \in \I} \ker \ai{i} =\bigcap_{\i \in \I} \uui{i}$.  This figure could be an illustration of the setup in \fittingOEg.}
\label{planesFig}
\end{figure}

This way, $\A$ characterizes $\UU$ just as $\a$ characterizes $\uu$; and just as $\UU$, $\A$ is a function of $\Sstar$, $\K$ and $\O$ (see \fittingORmk).

If we let \phantomsection\label{LDef}$\L$ be the number of linearly independent rows in $\A$, we immediately know from elemental linear algebra\cite{friedberg} that $\dim \UU = \dim \ker \A = \d - \L$.  We state this as a lemma, as it is an observation that will come up in our subsequent analysis.
\begin{myLemma}[$\dim \UU$]
\label{dimUULem}
$\dim \UU = \d-\L$.
\end{myLemma}

It is clear that we are interested on determining when $\O$ defines an $\A$ with linearly independent rows.  But before moving to that, let us discuss one very nice property of $\A$.

\begin{myLemma}
\label{aoWithzerosLem}
There exists no $\mathscr{A}$ with a row with fewer than $\r+1$ non-zero entries such that $\ker \A = \ker \mathscr{A}$.
\end{myLemma}
\begin{proof}
Very similar to the proof of \aEntriesLem, suppose for contradiction that there exists an $\mathscr{A}$ with a row $\var{a}$ with fewer than $\r+1$ non-zero entries such that $\ker \A = \ker \mathscr{A}$.  Use $\ups \subset \{1,...,\d\}$ to denote the set of the positions of the non-zero entries of $\var{a}$, and use it analogous to $\o$.

Let $\s$ be an $\r$-dimensional subspace that \fitsO\ $\O$.  By \UUcontainsAllSLem, $\s \subset \UU$, so $\s \subset \UU = \ker\A=\ker \mathscr{A}$, hence every $\u \in \s$ must satisfy $\< \var{a}_\ups,u_\ups \> =0$, i.e., $S_\ups \subset \ker \var{a}^\T_\ups$.  Then
\begin{align*}
\dim S_\ups \leq \dim \var{a}^\T_\ups = |\ups| -1 < |\ups|,
\end{align*}
i.e., $\s$ is \degenerate.  Since $\s$ was arbitrary, we know that this holds for every $\s$ that \fitsO\ $\O$, which contradicts \existsNonDegenerateAss.
\end{proof}

As one simple consequence of \aoWithzerosLem\ we obtain an extremely useful result:

\begin{myCorollary}[$\L$]
\label{LCor}
$\m \geq \L+\r$.
\end{myCorollary}
\begin{proof}
Suppose without loss of generality that $\A$ has all its zero columns \----if any\---- in the first block, and that its first $\L$ non-zero columns are linearly independent; otherwise we can simply permute the columns of $\A$ accordingly.  Let $\var{A}$ be the $\n \times \m$ matrix with only the non-zero columns of $\A$, and suppose for contradiction that $\m < \L+\r$.  We can then transform $\A$ into the following reduced row echelon form:
\begin{align*}
\left[\begin{matrix} \\ \\ \\ \\ \end{matrix}\right.
\overbrace{
\begin{array}{c}
\\ \multirow{2}{*}{\Scale[2]{0}} \\ \\
\begin{array}{c} \hspace{.35cm}\hspace{.35cm} \end{array}
\end{array}}^{\d-\m}
\overbrace{
\begin{array}{|c}
\\ \multirow{2}{*}{\Scale[2]{\var{A}}} \\ \\
\begin{array}{c} \hspace{.35cm}\hspace{.35cm} \end{array}
\end{array}}^{\m}
\left.\begin{matrix} \\ \\ \\ \\ \end{matrix}\right]
=\A \sim \mathscr{A}=
\left[\begin{matrix} \\ \\ \\ \\ \end{matrix}\right.
\overbrace{
\begin{array}{c}
\vspace{.1cm}\\ \multirow{2}{*}{\Scale[2]{0}} \\ \\
\begin{array}{c} \hspace{.35cm}\hspace{.35cm} \end{array}\vspace{.1cm}
\end{array}}^{\d-\m}
\overbrace{
\underbrace{
\begin{array}{|c}
 \\ \Scale[2]{I} \\ \\ \hline
\begin{array}{c} \hspace{.35cm}\bs{0}\hspace{.35cm} \end{array}
\end{array}}_{\L}
\underbrace{
\begin{array}{|c}
\\ \Scale[2]{\var{B}} \\ \\ \hline
\begin{array}{c} \hspace{.35cm} \bs{0} \hspace{.35cm} \end{array}
\end{array}}_{\m-\L<\r}}^{\m}
\left.\begin{matrix} \\ \\ \\ \\ \end{matrix}\right]
\begin{matrix}
\left. \begin{matrix} \\ \\ \\ \end{matrix} \right\} \L \hspace{.65cm} \vspace{.1cm}  \\ 
\left. \begin{matrix} \\ \end{matrix} \right\} \n-\L.
 \end{matrix}
\end{align*}
We know from elemental linear algebra\cite{friedberg} that $\ker\A=\ker\mathscr{A}$, nevertheless the top $\L$ rows of $\mathscr{A}$ have fewer than $\r+1$ non-zero entries, which is a contradiction by \aoWithzerosLem.
\end{proof}

\subsection{Independence}
\label{independenceSec}

From our discussion in \textsection \ref{severalAtOnceSec}, it is clear that we are interested on determining when $\OO$ defines an $\AA$ with linearly independent rows.

With this in mind we use the following.
\begin{myDefinition}[Independent set]
\index{independent set $\O$}
\index{dependent set $\O$}
\label{independenceDef}
We say $\O$ is {\em independent} iff {\em all} rows of $\A$ are linearly independent.  Otherwise we say $\O$ is {\em dependent}.
\end{myDefinition}

With this definition we can now say that our next goal is to be able to identify when $\O$ is \independentO.  The answer to this is given by the following.
\begin{framed}
\begin{myLemma}[Characterization of \independentO\ sets]
\label{independenceLem}
$\bb{\O}$ is \independentO\ iff
\begin{align*}
\m \geq \n + \r \hspace{1cm} \text{$\forall$ $\O \subset \bb{\O}$}.
\end{align*}
\end{myLemma}
\end{framed}

\begin{myRemark}
\normalfont
Notice that \independenceLem\ defines a matroid\cite{oxley}.
$\blacksquare$
\end{myRemark}

This lemma represents a central part of Theorems \ref{allOfAKindThm} and \ref{uniquenessThm}, and is possibly the most transcending result of the document.  In order to prove it we will require the following.
\begin{myDefinition}[Dependent set, redundant set]
\index{independent set $\o$}
\index{dependent set $\o$}
\index{redundant set}
\label{redundantDef}
We say $\o$ is {\em dependent} on $\O$ or  {\em redundant} iff $\a$ is linearly dependent on the rows of $\A$.  Otherwise, we say $\o$ is independent of $\O$.
\end{myDefinition}

The next definition will play a crucial role in the proofs of Lemmas \ref{independenceLem} and \ref{basisLem} in this section, and also in Lemmas \ref{kCharLem}, \ref{allOfAKindLem} and \ref{basisCharLem}, in \textsection \ref{allOfAKindSec}.

\begin{myDefinition}[Basis]
\index{basis}
\label{basisDef}
We say $\O$ is a {\em basis} of $\o$ iff $\O$ is an \independentO\ set such that $\o$ is \dependento\ on $\O$ but is not \dependento\ on any proper subset of $\O$.  
\end{myDefinition}

\begin{myRemark}
\normalfont
Observe that \bases\ are not unique.  In fact, for a single $\o$ there could be several \bases, and even of different sizes.
$\blacksquare$
\end{myRemark}

\begin{myExample}
\label{severalBasesEg}
\normalfont
Suppose $\r=2$ and
\begin{align*}
\OO = \left[\begin{matrix} 
	\see & \miss & \see & \miss & \miss & \see \\
	\see & \see & \miss & \miss & \miss & \miss \\
	\see & \see & \see & \see & \miss &\miss \\
	\miss & \see & \see & \see & \see & \miss \\
	\miss & \miss & \miss & \see & \see & \see \\
	\miss & \miss & \miss & \miss & \see & \see
	\end{matrix}\right].
\end{align*}

$\oi{3}$ is trivially \dependento\ on itself, so $\{\oi{3}\}$ is a trivial \basis\ of $\oi{3}$.  One can verify using \independenceLem\ that $\oi{3}$ is \dependento\ on $\{\oi{1},\oi{2}\}$.  Then $\{\oi{1},\oi{2}\}$ is also a \basis\ of $\oi{3}$.  Finally, $\oi{3}$ is also \dependento\ on $\{\oi{4},\oi{5},\oi{6}\}$, so $\{\oi{4},\oi{5},\oi{6}\}$ is also a \basis\ of $\oi{3}$.  Notice that these three different \bases\ are not even of the same size.
$\blacksquare$
\end{myExample}

The crux of \independenceLem, and all our further results for that matter, lies in the following statement.
\begin{framed}
\begin{myLemma}[Property of \bases]
\label{basisLem}
Let $\O$ be a \basis\ of $\o$.  Then $\m = \n + \r$.
\end{myLemma}
\end{framed}

\begin{proof}
Let $\O$ be a \basis\ of $\o$.  If $\O$ is a trivial \basis, i.e., if $\O=\{\o\}$, it is trivially true that $\m=\n+\r$.  Suppose then that $\O$ is non-trivial.

First observe that $\o \subset \J$; otherwise $\a$ would have a nonzero entry corresponding to a zero column of $\A$, whence $\a$ could not possibly be linearly dependent on $\A$.

We can assume without loss of generality that $\A$ has all its zero columns \----if any\---- in the first block and $\a$ has its non-zero entries in the last $\r+1$ columns; otherwise, we may just permute the columns of $\a$ and $\A$ accordingly.  Also assume without loss of generality that the first non-zero entry of $\a$ is $1$ \----otherwise we can just scale the row\---- and let $\hat{\var{a}}$ denote the $1 \times \r$ row with the remaining non-zero entries of $\a$, such that we can write:
\begin{align}
\label{firstDecompositionAEq}
\left[\begin{matrix} \\ \\ \\ \\ \end{matrix}\right.
\begin{matrix}
\\ \A \\ \\ \hline
\a
\end{matrix}
\left.\begin{matrix} \\ \\ \\ \\ \end{matrix}\right]
=
\left[\begin{matrix} \\ \\ \\ \\ \end{matrix}\right.
\overbrace{
\begin{array}{c}
\vspace{.1cm}\\ \multirow{2}{*}{\Scale[2]{0}} \\ \\
\begin{array}{c} \hspace{.35cm}\hspace{.35cm} \end{array}\vspace{.1cm}
\end{array}}^{\d-\m}
\overbrace{
\underbrace{
\begin{array}{|c}
\\\Scale[2]{B} \\ \\ \hline
\begin{array}{c|r} \hspace{.5cm}\bs{0}\hspace{.5cm} & 1 \end{array}
\end{array}}_{\m-\r}
\underbrace{
\begin{array}{|c}
\\ \Scale[2]{C} \\ \\ \hline
\hspace{.5cm} \hat{\var{a}} \hspace{.5cm}
\end{array}}_{\r}}^{\m}
\left.\begin{matrix} \\ \\ \\ \\ \end{matrix}\right]
\begin{matrix}
\left. \begin{matrix} \\ \\ \\ \end{matrix} \right\} \n \\
\left. \begin{matrix} \\ \end{matrix} \right\} 1. \hspace{.1cm}
 \end{matrix}
\end{align}
Recall that we want to show that $\m=\n+\r$.  Since $\O$ is \independentO, we know $\n=\L$, so by \LCor\ we immediately know that $\m \geq \n+\r$.  As we specified in \textsection \ref{resultsSec}, all our statements hold for almost every $\Sstar$, so it suffices to show that $\m>\n+\r$ only in a set of measure zero.  Suppose then that $\m>\n+\r$.  This implies that $B$ has strictly more than $\n$ columns.

Observe that the rows of $B$ are linearly independent.  To see this, very similar to what we did in \LCor, suppose for contradiction that they are not.  This implies that we can transform $A$ into the following reduced row echelon form:
\begin{align*}
\A \sim \mathscr{A}=
\left[\begin{matrix} \\ \\ \\ \\ \end{matrix}\right.
\overbrace{
\begin{array}{c}
\vspace{.1cm}\\ \multirow{2}{*}{\Scale[2]{0}} \\ \\
\begin{array}{c} \hspace{.35cm}\hspace{.35cm} \end{array}\vspace{.1cm}
\end{array}}^{\d-\m}
\overbrace{
\underbrace{
\begin{array}{|c|}
 \\ \Scale[2]{\mathscr{B}} \\ \\ \hline
\begin{array}{c} \hspace{.35cm}\bs{0}\hspace{.35cm} \end{array}
\end{array}}_{\L}
\underbrace{
\begin{array}{c}
\vspace{.1cm}\\ \multirow{2}{*}{\Scale[2]{\var{C}}} \\ \\
\begin{array}{c} \hspace{.35cm}\hspace{.35cm} \end{array}\vspace{.1cm}
\end{array}}_{\r}}^{\m}
\left.\begin{matrix} \\ \\ \\ \\ \end{matrix}\right]
\begin{matrix}
\begin{matrix} \\ \\ \\ \end{matrix}  \\ 
\left. \begin{matrix} \\ \end{matrix} \right\} >0.
 \end{matrix}
\end{align*}
We know from elemental linear algebra\cite{friedberg} that $\ker \A = \ker \mathscr{A}$.  Nevertheless, the last row of $\mathscr{A}$ has at most $\r$ non-zero entries, which is a contradiction by \aoWithzerosLem.

Let $\var{a}$ be the $1 \times \m$ row with only the entries of $\A$ in the positions of $\J$, and $\var{A}$ be the $\n \times \m$ matrix with only the columns of $\A$ in the positions of $\J$.

Going back to \eqref{firstDecompositionAEq}, since the rows of $B$ are linearly independent and $\m-\r>\n$, we know, $B$ has $\n$ linearly independent columns.  Let $\var{B}$ denote the $\n \times \n$ block of $B$ that contains $\n$ linearly independent columns, and $\mathscr{B}$ the $\n \times (\m-\n-\r)$ remaining block of $B$.  We can thus assume without loss of generality that:
\begin{align*}
\left[\begin{matrix} \\ \\ \\ \\ \end{matrix}\right.
\begin{matrix}
\\ \var{A} \\ \\ \hline
\var{a}
\end{matrix}
\left.\begin{matrix} \\ \\ \\ \\ \end{matrix}\right]
=
\left[\begin{matrix} \\ \\ \\ \\ \end{matrix}\right.
\overbrace{
\underbrace{
\begin{array}{c}
 \\ \Scale[2]{\mathscr{B}} \\ \\ \hline
\begin{array}{c} \hspace{.5cm}\bs{0}\hspace{.5cm} \end{array}
\end{array}}_{\m - \n - \r \geq 1}
\underbrace{
\begin{array}{|c}
\\ \Scale[2]{\var{B}} \\ \\ \hline
\begin{array}{c|r} \hspace{.5cm}\bs{0}\hspace{.5cm} & 1 \end{array}
\end{array}}_{\n}}^{B}
\underbrace{
\begin{array}{|c}
\\ \Scale[2]{C} \\ \\ \hline
\hspace{.5cm} \hat{\var{a}} \hspace{.5cm}
\end{array}}_{\r}
\left.\begin{matrix} \\ \\ \\ \\ \end{matrix}\right]
\begin{matrix}
\left. \begin{matrix} \\ \\ \\ \end{matrix} \right\} \n \\
\left. \begin{matrix} \\ \end{matrix} \right\} 1. \hspace{.1cm}
 \end{matrix}
\end{align*}
Notice that the column of $B$ corresponding to the $1$ in $\var{a}$ must belong to $\var{B}$ (otherwise, we have that $\beta \var{B}=0$, with $\beta$ as in \eqref{aLdOnAEq}, which implies that $\var{B}$ has a linearly dependent row, hence a linearly dependent column).

We can further assume without loss of generality that the first non-zero entry of every row of $\var{A}$ is $1$; otherwise we may just scale each row.  Finally, we may also assume that the first column of $\var{A}$, namely the first column of $\mathscr{B}$ has all its $\flat$ non-zero entries \----all ones by construction\---- on the top rows; otherwise we may just permute the rows of $\var{A}$ accordingly, such that we have:
\begin{align}
\label{AconstructionEq}
\left[\begin{matrix} \\ \\ \\ \\ \end{matrix}\right.
\begin{matrix}
\\ \var{A} \\ \\ \hline
\var{a}
\end{matrix}
\left.\begin{matrix} \\ \\ \\ \\ \end{matrix}\right]
=
\begin{matrix}
\flat \left. \begin{matrix} \vspace{.1cm} \\ \end{matrix} \right\{ \\
\left. \begin{matrix} \\ \vspace{.25cm} \\ \end{matrix} \right.
\end{matrix}
\left[\begin{matrix} \\ \\ \\ \\ \end{matrix}\right.
\overbrace{
\underbrace{
\begin{array}{c}
\multirow{2}{*}{\Scale[2]{\bs{1}}} \\
\\
\bs{0}\vspace{.175cm} \\ \hline
\bs{0}
\end{array}}_{1}
\underbrace{
\begin{array}{|c}
\\ \Scale[2]{\mathcal{B}} \\ \\ \hline
\begin{array}{c} \hspace{.5cm}\bs{0}\hspace{.5cm} \end{array}
\end{array}}_{\m - \n - \r -1 \geq 0}}^{\mathscr{B}}
\underbrace{
\begin{array}{|c}
\\ \Scale[2]{\var{B}} \\ \\ \hline
\begin{array}{c|r} \hspace{.5cm}\bs{0}\hspace{.5cm} & 1 \end{array}
\end{array}}_{\n}
\underbrace{
\begin{array}{|c}
\\ \Scale[2]{C} \\ \\ \hline
\hspace{.5cm} \hat{\var{a}} \hspace{.5cm}
\end{array}}_{\r}
\left.\begin{matrix} \\ \\ \\ \\ \end{matrix}\right]
\begin{matrix}
\left. \begin{matrix} \\ \\ \\ \end{matrix} \right\} \n \\
\left. \begin{matrix} \\ \end{matrix} \right\} 1. \hspace{.1cm}
 \end{matrix}
\end{align}
Since $\a$ is minimally linearly dependent on $\A$, we know there exists a unique row vector $\beta \in \R^{\n}$ with all non-zero entries such that
\begin{align}
\label{aLdOnAEq}
\beta \A= \a,
\end{align}
which implies $1<\flat \leq \n$.  In particular, using \eqref{aLdOnAEq} on the $\var{B}$ block of \eqref{AconstructionEq}, we have that $\beta \var{B}=\left[\begin{array}{c|r} \bs{0} & 1 \end{array}\right]$, and since $\var{B}$ is full-rank, we can solve for $\beta$:
\begin{align}
\label{betaEq}
\beta= \left[\begin{array}{c|r} \bs{0} & 1 \end{array}\right] \var{B}^{-1},
\end{align}
i.e., $\beta$ is the the last row of $\var{B}^{-1}$.  We know from elemental linear algebra\cite{friedberg} that
\begin{align*}
\var{B}^{-1} = \frac{1}{|\var{B}|} \var{B}^*,
\end{align*}
where $|\cdot|$ and $\cdot^*$ denote the determinant and the adjoint matrix of $\cdot$ , respectively.  Since $\var{B}^*$ is the transpose of the cofactor matrix of $\var{B}$, we have that
\begin{align}
\label{inverseElementEq}
\var{b}^{-1}_{\n i} = \pm_{i\n} \frac{|\var{B}_{i\n}|}{|\var{B}|},
\end{align}
where $\var{B}_{i \n}$ denotes the $(\n-1) \times (\n-1)$ minor of $\var{B}$ obtained by removing the $\i^{th}$ row and the $\n^{th}$ column of $\var{B}$, and $\pm_{i \n}$ denotes the $(\i,\n)^{th}$ entry of the following matrix:
\begin{align*}
\pm := \left[ \begin{matrix}
+ & - & + & - & \cdots \\
- & + & - & + & \cdots \\
+ & - & + & - & \cdots \\
- & + & - & + & \cdots \\
\vdots & \vdots & \vdots & \vdots & \ddots
\end{matrix} \right],
\end{align*}
For example, if:
\begin{align*}
\var{B} = \left[ \begin{array}{ccc|c}
\var{b}_{11} & \var{b}_{12} & \var{b}_{13} & \var{b}_{14} \\ \hline
\var{b}_{21} & \var{b}_{22} & \var{b}_{23} & \var{b}_{24} \\
\var{b}_{31} & \var{b}_{32} & \var{b}_{33} & \var{b}_{34} \\
\var{b}_{41} & \var{b}_{42} & \var{b}_{43} & \var{b}_{44}
\end{array} \right],
\end{align*}
then
\begin{align*}
\pm_{1\n} |\var{B}_{1 \n}| = -
\left| \begin{matrix}
\var{b}_{21} & \var{b}_{22} & \var{b}_{23} \\
\var{b}_{31} & \var{b}_{32} & \var{b}_{33} \\
\var{b}_{41} & \var{b}_{42} & \var{b}_{43}
\end{matrix} \right|.
\end{align*}
Now observe that for $\i \neq 1$, we have:
\begin{align*}
|\var{B}_{i \n}| = \sum_{\ii=1}^{\n-1} \pm_{1 \ii} \var{b}_{1 \ii} |\var{B}^{1\ii}_{i\n}|,
\end{align*}
where $\var{B}^{1\ii}_{i\n}$ is the $(\n-2) \times (\n-2)$ minor of $\var{B}_{i \n}$ obtained by removing the first row and the $\ii^{th}$ column of $\var{B}_{i\n}$.  In our example,
\begin{align*}
\var{B} = \left[ \begin{array}{cc|c|c}
\var{b}_{11} & \var{b}_{12} & \var{b}_{13} & \var{b}_{14} \\ \hline
\var{b}_{21} & \var{b}_{22} & \var{b}_{23} & \var{b}_{24} \\ \hline
\multicolumn{2}{c|}{\multirow{2}{*}{\Scale[1]{\var{B}^{13}_{2\n}}}} & \var{b}_{33} & \var{b}_{34} \\
\multicolumn{2}{c|}{} & \var{b}_{33} & \var{b}_{44}
\end{array} \right].
\end{align*}
Recall that $\beta_i \neq 0$ for every $\i$, and since $\beta_i = \frac{\pm_{i\n} |\var{B}_{i \n}|}{|\var{B}|}$, it is clear that $|\var{B}_{i \n}| \neq 0$ for every $\i$.  This implies that there is at least one $\ii$ for which $\var{b}_{1 \ii} \neq 0$.

Now let us look back at \eqref{aLdOnAEq}.  Using the first column of \eqref{AconstructionEq} we obtain $\beta \left[\begin{array}{c|r} \bs{1} & \bs{0} \end{array}\right]^\var{T} = 0$.  Substituting \eqref{betaEq}, \eqref{inverseElementEq}, and factoring out the common term $|\var{B}|$, we obtain:
\begin{align}
\label{fVarietyEq}
\sum_{i=1}^\flat \pm_{i \n}|\var{B}_{i \n}| = 0.
\end{align}
We will now show that $f:=\sum_{i=1}^\flat \pm_{i \n}|\var{B}_{i \n}|$ is a non-zero polynomial.  Write:
\begin{align*}
f:=\sum_{\i=1}^\flat \pm_{i \n}|\var{B}_{i \n}| &= \pm_{1\n} |\var{B}_{1 \n}| + \sum_{i=2}^\flat \pm_{i\n}|\var{B}_{i\n}| \\
&= \pm_{1\n}|\var{B}_{1\n}| + \sum_{i=2}^\flat \pm_{i\n} \sum_{\ii=1}^{\n-1} \pm_{1 \ii} \var{b}_{1 \ii} |\var{B}^{1\ii}_{i\n}| \\
&=\pm_{1\n}|\var{B}_{1\n}| + \sum_{\ii=1}^{\n-1} \var{b}_{1 \ii} \underbrace{ \sum_{i=2}^\flat \pm_{i\n} \pm_{1 \ii} |\var{B}^{1\ii}_{i\n}|}_{=:\var{c}_\ii},
\end{align*}
where $\var{c}_\ii$ is a constant polynomial of $\var{b}_{1\ii}$.  By simple inspection, one can see that $|\var{B}_{1 \n}|$ does not depend on $\var{b}_{1\ii}$, i.e., it is also a constant polynomial of $\var{b}_{1\ii}$.  Thus, if $\var{c}_\ii \neq 0$ for some $\ii$ for which $\var{b}_{1\ii} \neq 0$, we can immediately conclude that $f$ is a non-zero polynomial of $\var{b}_{1\ii}$.

On the other hand, if $\var{c}_\ii=0$ for every $\ii$ for which $\var{b}_{1\ii} \neq 0$, then $f=\pm_{1\n}|\var{B}_{1\n}|$.  Similar to what we did for $\i \neq 1$, we can write:
\begin{align*}
|\var{B}_{1 \n}| = \sum_{\ii=1}^{\n-1} \pm_{2 \ii} \var{b}_{2 \ii} |\var{B}^{2\ii}_{1\n}|,
\end{align*}
where $\var{B}^{2\ii}_{i\n}$ is the $(\n-2) \times (\n-2)$ minor of $\var{B}_{1 \n}$ obtained by removing the second row and the $\ii^{th}$ column of $\var{B}_{1\n}$.  Again, since $\beta_1 \neq 0$, there is at least one $\ii$ for which $\var{b}_{2 \ii} \neq 0$, thus $f$ is a non-zero polynomial of $\var{b}_{2\ii}$.

Observe that every $\var{b}_{i \ii}$ is either zero, or one of the entries of $\aoi{i}$.  Thus $f$ is a non-zero polynomial of at least one of the entries of $\aoi{1}$ or $\aoi{2}$, say $\aoi{i\ii}$.

Since $\flat>1$, $\aoi{i\ii} \neq 1$, hence $f$ is a non-zero polynomial in $\R[\aoi{i\ii}]$.  We thus know from elemental measure theory and algebraic geometry\cite{murray} that the variety defined by \eqref{fVarietyEq} has measure zero.  In other words, there is only a subset of measure zero over $\aoi{i\ii}$, hence over the set of all sets $\Sstar$, for which $\m>\n+\r$, which concludes the proof.
\end{proof}

We are now ready to present the proof of \independenceLem.

\begin{proof}(\independenceLem).  We prove both directions by contrapositive.  Explicitly, we will show that $\bb{\O}$ is \dependentO\ iff $\exists$ $\O \subset \bb{\O}$ with $\m<\n+\r$.

($\Rightarrow$) Let $\bb{\O}$ be \dependentO.  By definition, it contains a \redundant\ $\o$ with \basis\ $\bar{\O} \subset \OO \backslash \o$.  By \basisLem, $\bar{\n}=\bar{\m}-\r$.

Take $\O:=\bar{\O} \cup \o$.  It is clear that $\m=\bar{\m}$.  Nevertheless, $\n=\bar{\n}+1$.  Thus $\m<\n+\r$, and we have the first implication.

($\Leftarrow$) Suppose $\exists$ $\O \subset \bb{\O}$ with $\m<\n+\r$.  By \LCor, $\n>\L$, which implies $\O$ is be \dependentO. Of course, since $\O \subset \bb{\O}$, $\bb{\O}$ is also \dependentO, which concludes the second part of the proof.
\end{proof}

\begin{myExample}
\label{liEg}
\normalfont
By \independenceLem, the $\OO$'s from Examples \ref{observationSetsEg}, \ref{fittingOEg} and \ref{independentNotBasisEg} are \independentO, for their respective $\r$'s.  On the other hand, the $\OO$'s from Examples \ref{mainThmEg}, \ref{severalBasesEg}, \ref{infinitelyManySubspacesEg} and \ref{allOfAKindEgb} are \dependentO, for their respective $\r$'s.

Notice that $\OO$ may be \dependentO\ and still satisfy that there is only one subspace that \fitsO\ it.  Such is the case of Examples \ref{mainThmEg}, \ref{severalBasesEg} and \ref{allOfAKindEgb}: they are \dependentO, but they contain an \independentO\ set of size $\d-\r$.  In other words, they have some \redundant\ vectors.
$\blacksquare$
\end{myExample}

\begin{myRemark}
\normalfont
As we will see in \textsection \ref{KcharacterizationSec}, \redundant\ vectors are not useless; they are required to guarantee that different incomplete vectors belong to the same subspace.  In fact, it is easy to see that the conditions of \allOfAKindThm\ require the existence of \redundant\ vectors in $\OO$.
$\blacksquare$
\end{myRemark}

\subsection{All you need is $\d-\r$}
\label{allYouNeedSec}
Recall that we are interested on determining conditions to guarantee that there is only one $\r$-dimensional subspace that \fitsO\ $\OO$.  Observe that this is not implied by $\OO$ being \independentO\ nor viceversa.  Nevertheless, as a simple consequence of \dimUUrCor, \dimUULem\ and \independenceDef, we obtain the following.

\begin{myCorollary}
\label{uniquenessIndependenceCor}
Suppose $\n=\d-\r$.  There is only one $\r$-dimensional subspace that \fitsO\ $\O$ iff $\O$ is \independentO.
\end{myCorollary}

In other words, Corollary \ref{uniquenessIndependenceCor}, is telling us that all we need is that $\OO$ has $\d-\r$ \independentO\ $\o$'s to guarantee that there is only one $\r$-dimensional subspace that \fitsO\ $\OO$.  \uniquenessThm\ is precisely a combination of this observation and the characterization of \independentO\ sets.

\begin{myExample}
\normalfont
Consider $\OO$ as in \observationSetsEg.  $\OO$ satisfies the conditions of Corollary \ref{uniquenessIndependenceCor}, so there is only one $\r$-dimensional subspace that \fitsO\ $\OO$.
$\blacksquare$
\end{myExample}

\subsection{Proof of \uniquenessThm}
\label{proofUniquenessSec}
We conclude this section with the proof of \uniquenessThm, which comes immediately as a consequence of Corollaries \ref{dimUUrCor} and \ref{LCor} and Lemmas \ref{dimUULem} and \ref{independenceLem}.  Observe that \uniquenessThm\ essentially states that there exits only one $\r$-dimensional subspace that \fitsO\ $\OO$ iff $\OO$ contains an \independentO\ set of size $\d-\r$.

\begin{proof}(\uniquenessThm)
($\Rightarrow$) We prove this by contrapositive.  Suppose $\nexists$ $\bb{\O} \subset \OO$ of size $\d-\r$ such that $\m \geq \n+\r$ for every $\O \subset \bb{\O}$. By \independenceLem, we know there is no \independentO\ set of size $\d-\r$ in $\OO$, i.e., there are no $\d-\r$ linearly independent rows in $\AA$.  By \dimUULem, $\dim \UU>\r$, thus by \dimUUrCor\ there exist infinitely many $\r$-dimensional subspaces that \fitO\ $\OO$.

($\Leftarrow$) Suppose there exists an $\bb{\O} \subset \OO$ of size $\d-\r$ such that $\m \geq \n+\r$ for every $\O \subset \bb{\O}$. By \independenceLem\ $\bb{\O}$ is \independentO, i.e., $\AA$ contains at least $\d-\r$ linearly independent rows.  Furthermore, by \LCor\ $\AA$ contains exactly $\d-\r$ linearly independent rows.  By \dimUULem, $\dim \UU=\r$, thus by \dimUUrCor\ there is only one $\r$-dimensional subspace that \fitsO\ $\bb{\O}$.  Since $\bb{\O} \subset \OO$, there is also only one $\r$-dimensional subspace that \fitsO\ $\OO$.
\end{proof}

\pagebreak
\section{All of a kind}
\label{allOfAKindSec}

Notice that so far, \uniquenessThm\ is only stating when there will be only one $\r$-dimensional subspace that \fitsO\ $\OO$.  It is not yet implying anything about $\K$.  For all we know such subspace could \fitO\ $\OO$, even if the $\x$'s belong to different subspaces.

In other words, even if there is only one $\r$-dimensional subspace that \fitsO\ $\OO$, there is yet nothing that assures us that all the columns of $\X$ indeed lie in one $\r$-dimensional subspace.

\begin{myExample}
\label{allOfAKindEga}
\normalfont
Consider the same setup as in \fittingOEg.  One can easily verify using \uniquenessThm\ that there is only one $\r$-dimensional subspace that \fitsO\ $\OO$ .  But suppose that the columns of $\X$ don't lie in the same $\r$-dimensional subspace, i.e., $\ki{1} \neq \ki{2}$, as in \prologueaEg.

One can see that there are infinitely many $\r$-dimensional subspaces that will satisfy $\soi{1}=\sstaroi{1}$; the $\r$-dimensional subspaces contained in $\uui{1}$, e.g., any of the subspaces in \severalOptionsFiga.  These will \fito\ $\oi{1}$.  There are also infinitely many $\r$-dimensional subspaces that will satisfy $\soi{2}=\sstaroi{2}$; the $\r$-dimensional subspaces contained in $\uui{2}$, e.g., any of the subspaces in \severalOptionsFigb.  These will \fito\ $\oi{2}$.

\begin{figure}[H]
     \begin{center}
        \subfigure[$\r$-dimensional subspaces that satisfy $\soi{1}=\sstaroi{1}$, hence \fito\ $\oi{1}$.]{
           \label{severalOptionsFiga}
            \includegraphics[width=5cm]{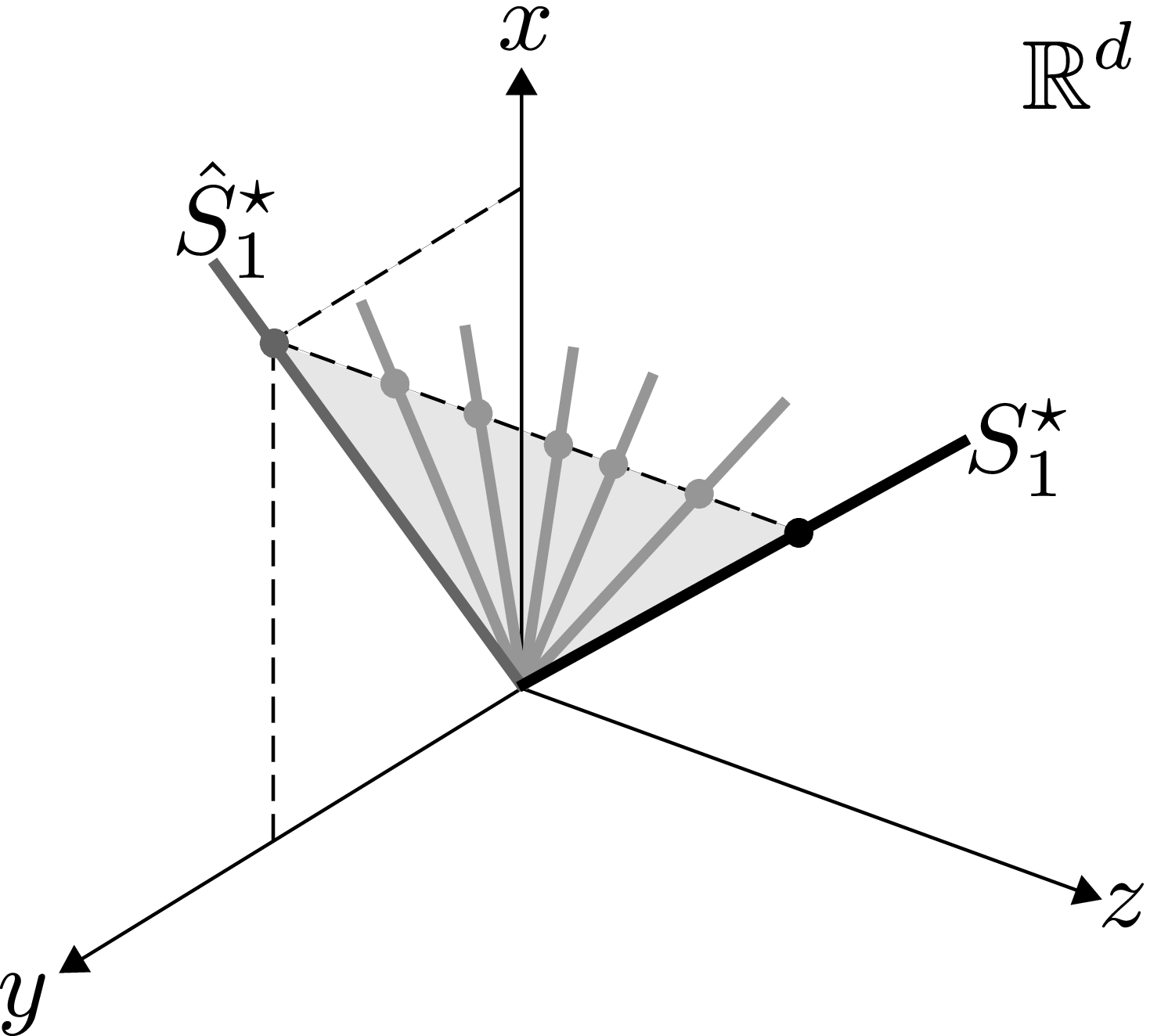}
        } \hspace{.5cm}
        \subfigure[$\r$-dimensional subspaces that satisfy $\soi{2}=\sstaroi{2}$, hence \fito\ $\oi{2}$.]{
           \label{severalOptionsFigb}
           \includegraphics[width=5cm]{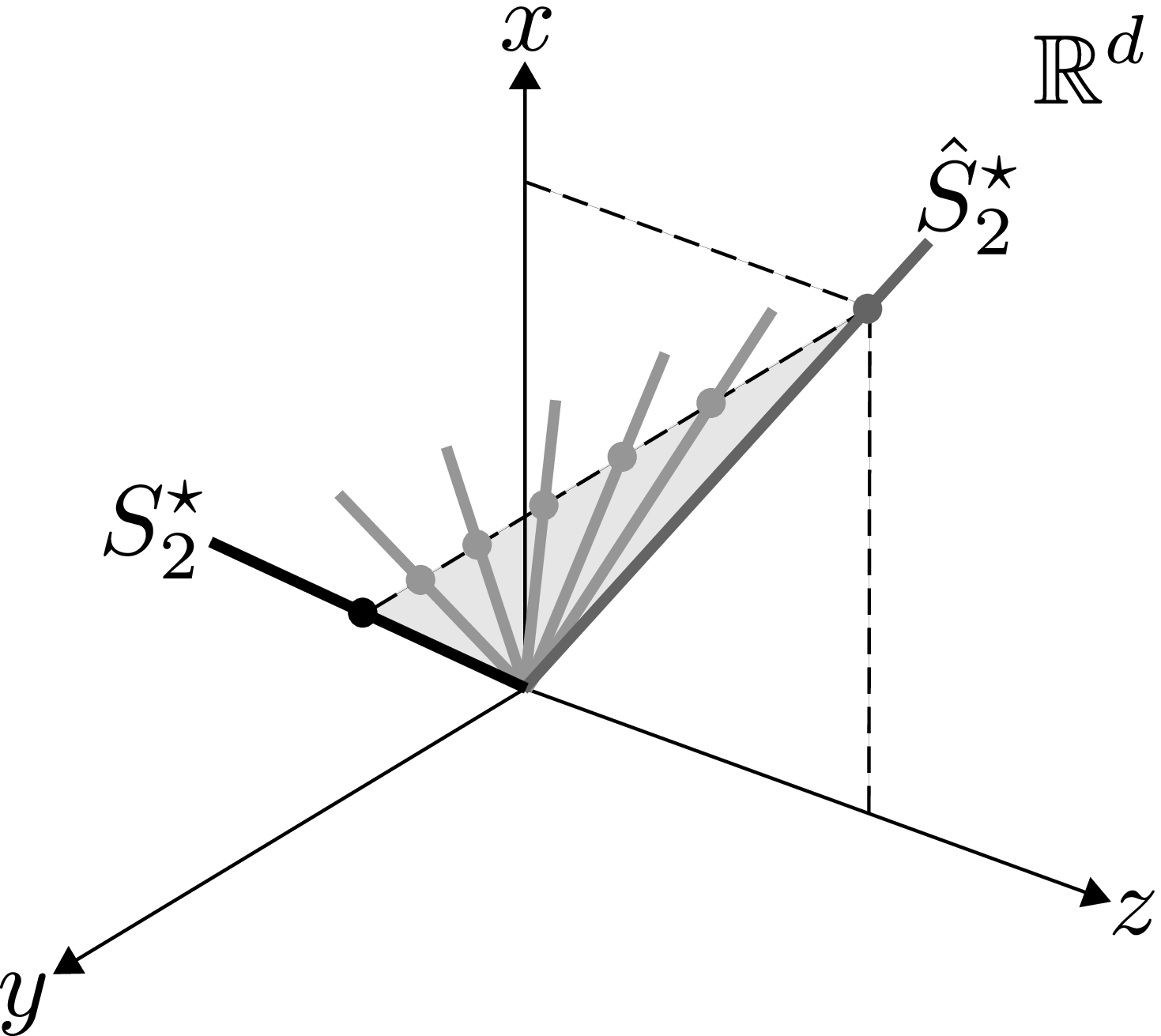}
        }
    \end{center}
    \caption{There are infinitely many $\r$-dimensional subspaces that satisfy $\soi{i}=\sstaroi{i}$, hence \fito\ $\oi{i}$; namely, the subspaces whose projection onto $\Rdo{\omega_i}$ is the same as the projection of $\sstari{i}$ onto $\Rdo{\omega_i}$.}
   \label{severalOptionsFig}
\end{figure}

And despite the columns of $\X$ don't lie in the same $\r$-dimensional subspace, there is one $\r$-dimensional subspace, $\s$, that satisfies $\soi{1} = \sstaroi{1}$ and $\soi{2} = \sstaroi{2}$ simultaneously, and thus \fitsO\ $\OO$: $\s=\UUO=\uui{1} \cap \uui{2}$.

\begin{figure}[H]
\centering
\includegraphics[width=5cm]{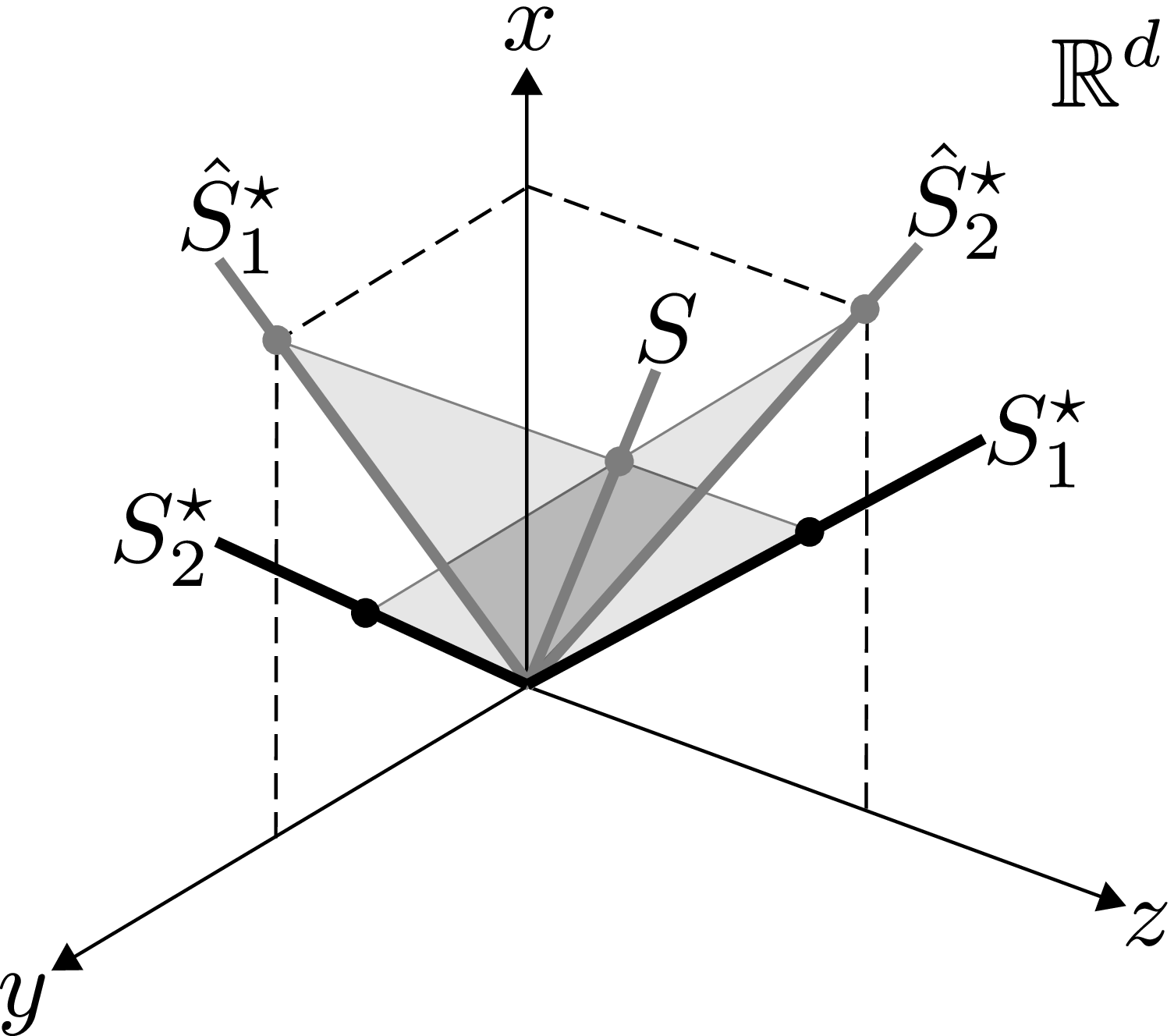}
\label{differentSsFig}
\caption{Despite $\ki{1} \neq \ki{2}$, there is one $\r$-dimensional subspace $\s$ that satisfies $\soi{1} = \sstaroi{1}$ and $\soi{2} = \sstaroi{2}$ simultaneously, and thus \fitsO\ $\OO$: $\s=\UUO=\uui{1} \cap \uui{2}$.}
\end{figure}

For an explicit instance of the example above, take \prologuebEg.  It is easy to see that $\s$ \fitsO\ $\OO$ despite the columns of $\X$ do not lie in a $1$-dimensional subspace.
$\blacksquare$
\end{myExample}

We want to know when can we be sure that if there is only one $\r$-dimensional subspace that \fitsO\ $\OO$, it is because all the columns of $\X$ indeed lie in an $\r$-dimensional subspace.

\begin{myExample}
\label{allOfAKindEgb}
\normalfont
Continuing with \allOfAKindEga, suppose we had an additional $\oi{3}=\{1,3\}$, i.e., suppose
\begin{align*}
\OO = \left[\begin{matrix}
\see & \see & \miss \\
\see & \miss & \see \\
\miss & \see & \see
\end{matrix}\right],
\end{align*}
Then almost surely, $\UUO=\uui{1} \cap \uui{2} \cap \uui{3}$ will only contain an $\r$-dimensional subspace iff all the columns of $\X$ lie in an $\r$-dimensional subspace.

\begin{figure}[H]
\centering
\includegraphics[width=5cm]{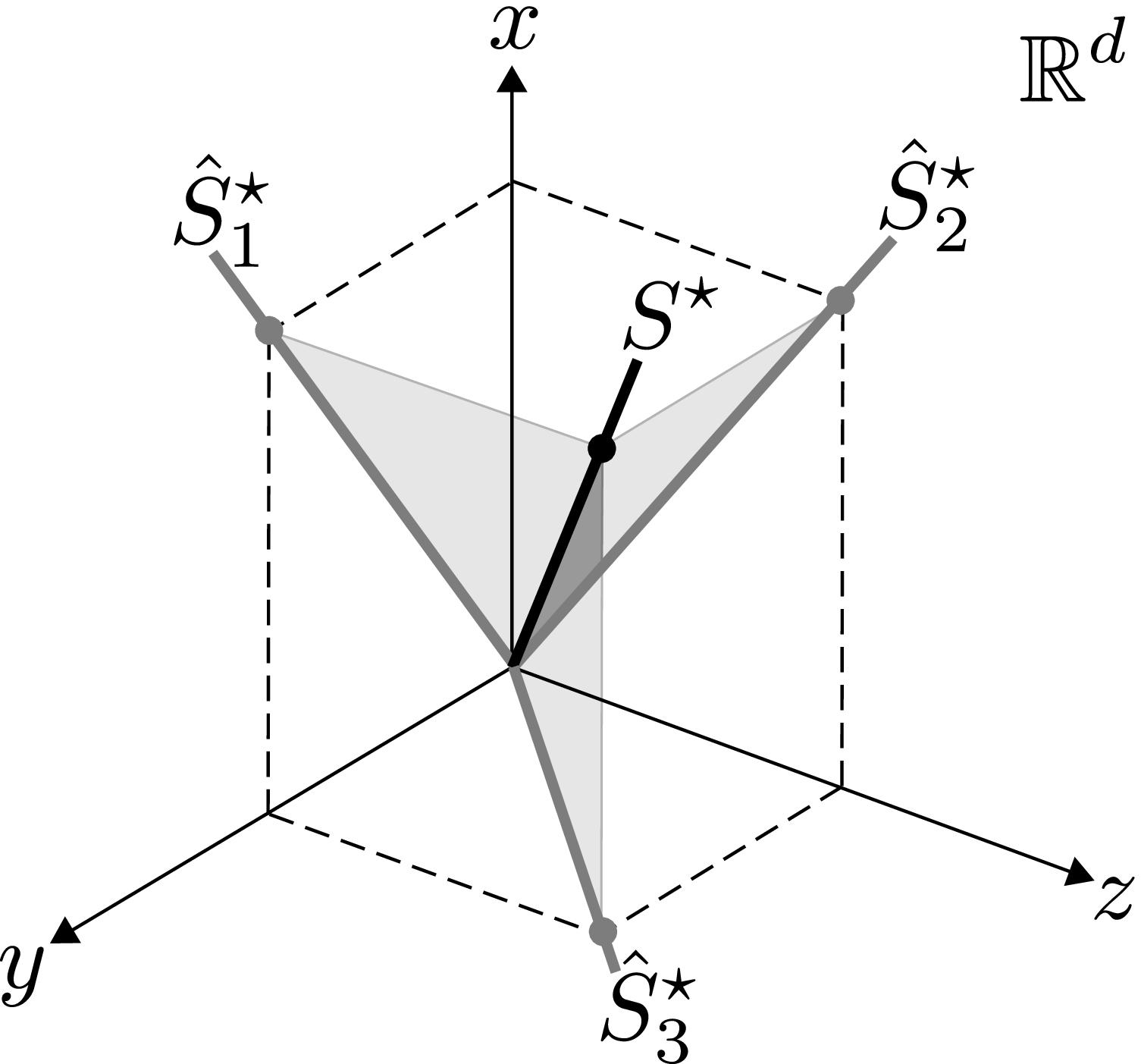}
\caption{Almost surely, $\UUO$, the intersection of three hyperplanes in $\R^3$, each corresponding to one of the $\uui{i}$'s, is a line iff $\ki{1}=\ki{2}=\ki{3}$, whence $\UUO=S^\star$.}
\label{mustMatchFig}
\end{figure}

$\blacksquare$
\end{myExample}

In other words, we want to derive necessary and sufficient conditions on $\OO$ to guarantee that all the columns of $\X$ indeed lie in an $\r$-dimensional subspace, i.e., that $\ki{i}=\ki{\ii}$ for every $(\i,\ii)$.

We want those conditions to be sufficient in the sense that if $\OO$ satisfies such conditions, then it {\em must} be true that $\ki{i}=\ki{\ii}$ for every $(\i,\ii)$.

We want those conditions to be necessary in the sense that if $\OO$ does not satisfy these conditions, it cannot be guaranteed that $\ki{i}=\ki{\ii}$ for every $(\i,\ii)$, i.e., $\ki{i}$ {\em could} be different from $\ki{\ii}$ for some $(\i,\ii)$, implying that the columns of $\X$ {\em might} not all lie in the same $\r$-dimensional subspace.

The answer to this question, the main one of this document, is given by \allOfAKindThm, which we will show in this section.

\subsection{Characterization of $\K$}
\label{KcharacterizationSec}
We start the work towards the proof of \allOfAKindThm\ with the following lemma.  It is essential for our further analysis, as will allow us to determine when columns of $\X$ lie in the same $\r$-dimensional subspace.

\begin{framed}
\begin{myLemma}[Characterization of $\K$]
\label{kCharLem}
Let $\O$ be a \basis\ of $\oi{i}$.  Then $\ki{i}=\ki{\ii}$ for every $\ii \in \I$.
\end{myLemma}
\end{framed}

\begin{proof}
Let $\O$ be a \basis\ of $\oi{i}$.  Since $\ai{i}$ is minimally linearly dependent on $\A$ by definition, we may write
\begin{align*}
\ai{i} = \sum_{\ii \in \I} \beta_{\ii} \ai{\ii},
\end{align*}
where $\beta_{\ii} \neq 0$ for every $\ii \in \I$.  On the other hand, $\ai{i}$ is a non-zero function of $\oi{i}$ and $\sstari{i}$, say $f_\i(\oi{i},\sstari{i})$.  Similarly, $\ai{\ii}$ is a non-zero function of $\oi{\ii}$ and $\sstari{\ii}$, say $f_{\ii}(\oi{\ii},\sstari{\ii})$.  We thus have that
\begin{align*}
f_\i(\oi{i},\sstari{i}) = \sum_{\ii \in \I} \beta_{\ii} f_{\ii}(\oi{\ii},\sstari{\ii}).
\end{align*}
The subspaces in $\Sstar$ keep no relation between each other for almost every $\Sstar$.  Thus, almost surely, the only way that these equality can hold is iff $\sstari{i}=\sstari{\ii}$ for every $(\i,\ii)$, i.e., iff $\ki{i}=\ki{\ii}$ for every $\ii \in \I$.
\end{proof}

Of comparable importance is the converse of the previous lemma.  It will allow us to determine when columns of $\X$ {\em might} lie in different subspaces.  This, together with our former result give us a complete way to characterize $\K$.

\begin{framed}
\begin{myLemma}[Converse characterization of $\K$]
\label{kCharConverseLem}
Let $\oi{i}$ be \independento\ of $\O$.  Then $\ki{i}$ {\em might} be different from $\ki{\ii}$ for every $\ii \in \I$.
\end{myLemma}
\end{framed}

\begin{proof}
Since we only need to show that $\ki{i}$ {\em might} be different from $\ki{\ii}$ for every $\ii \in \I$, it suffices an example.  Take \allOfAKindEga.
\end{proof}

\subsection{All you need is $\d-\r$ of a kind}
\label{allYouNeed2Sec}
We now know from \kCharLem\ that the columns of $\X$ corresponding to a set and its \basis\ belong to the same subspace.  We also know from \uniquenessThm\ that there is only one $\r$-dimensional subspace that \fitsO\ $\OO$ if $\OO$ contains an \independentO\ set of size $\d-\r$.  Combining these two ideas we obtain the following lemma, which intuitively tells us that if we find a set with a \basis\ of size $\d-\r$, then {\em all} columns correspond to the same subspace.  Conversely, it tells us that if there is one set for which we cannot find a \basis\ of size $\d-\r$, the columns of $\X$ {\em might} belong to different subspaces.

\begin{framed}
\begin{myLemma}[All of a kind]
\label{allOfAKindLem}
$\ki{i}=\ki{\ii}$ for every $(\i,\ii)$ if $\exists$ $\o$ with a \basis\ of size $\d-\r$.
Conversely, $\ki{i}$ {\em might} be different from $\ki{\ii}$ for some $(\i,\ii)$ if $\nexists$ \basis\ of size $\d-\r$ for some $\o$.
\end{myLemma}
\end{framed}

\begin{proof}
($\Leftarrow$)  Suppose $\exists$ $\o$ with a \basis\ $\O$ of size $\d-\r$.  By \kCharLem, $\ki{i}=\ki{\ii}$ for every $\ii \in \I$.  Then $\sstari{i}$ clearly \fitsO\ $\O$.  Since $\O$ is \independentO\ and $\n=\d-\r$, by Corollary \ref{uniquenessIndependenceCor} there is only one $\r$-dimensional subspace that \fitsO\ $\O$, so $\sstari{i}$ is the only $\r$-dimensional subspace that \fitsO\ $\O$.  Finally, observe that $\sstari{i}$ cannot \fitO\ every $\O$ unless $\ki{i}=\ki{\ii}$ for every $\ii \in \I$.

($\Rightarrow$)  Suppose $\nexists$ $\o$ with \basis\ of size $\d-\r$.

Let $\O$ be the set of sets that $\o$ can be \dependento\ on, i.e.
\begin{align*}
\text{$\O = \{\bb{\o} \in \bb{\O} : \bb{\O} \subset \OO$ is a \basis\ of $\o \}$}.
\end{align*}
We will show that $\ki{i}$ {\em might} be different for the elements of $\I$ and the elements of $\I^\c$.  By \kCharConverseLem, it suffices to show that every $\bb{\o} \in \O$ is \independento\ of $\O^\c$.

Let $\bb{\o} \in \O$, and $\bb{\O}$ be a \basis\ of $\o$ that contains $\bb{\o}$.  Suppose for contradiction that $\bb{\o}$ is \dependento\ on $\O^\c$ and let $\bar{\O} \subset \O^\c$ be a \basis\ of $\bb{\o}$.

By our definition of \basis, $\o$ is not \dependento\ on $\bb{\O} \backslash \bb{\o}$.  Nevertheless, since $\bar{\O}$ is a \basis\ of $\bb{\o}$, $\o$ is \dependento\ on $(\bb{\O} \backslash \bb{\o}) \cup \bar{\O}$.  This implies that there is a \basis\ of $\o$ in $(\bb{\O} \backslash \bb{\o}) \cup \bar{\O}$ that contains {\em at least} one set of $\bar{\O}$, say $\bar{\o}$.  Then $\bar{\o}$ belongs to $\O$.  Since $\bar{\O} \subset \O^\c$, $\bar{\o}$ belongs to $\O^\c$ as well, which is a contradiction.

This implies that $\bb{\o}$ is \independento\ of $\O^\c$.  Since $\bb{\o}$ was arbitrary, we conclude that every $\bb{\o} \in \O$ is \independento\ of $\O^\c$.  Thus, by \kCharConverseLem, $\ki{i}$ {\em might} be different from $\ki{\ii}$ for $i \in \I$ and $\ii \in \I^\c$.
\end{proof}

\begin{myRemark}
\label{oneThenAllRmk}
\normalfont
This lemma implies that {\em if} there is an $\o$ with a \basis\ of size $\d-\r$, then {\em every} $\o$ has a \basis\ of size $\d-\r$.  Conversely, {\em if} there is an $\o$ with no \basis\ of size $\d-\r$, then {\em no} $\o$ has a \basis\ of size $\d-\r$.  In other words, we have the following.
$\blacksquare$
\end{myRemark}

\begin{framed}
\begin{myCorollary}[One, then All]
\label{oneThenAllCor}
There exists an $\o$ with a \basis\ of size $\d-\r$ iff every $\o$ has a \basis\ of size $\d-\r$.
\end{myCorollary}
\end{framed}

\subsection{Bases Characterization}
\label{basisCharSec}
\allOfAKindLem\ tells us that the columns of $\X$ {\em must} belong to the same subspace iff there is an $\o$ with a \basis\ of size $\d-\r$.  The only remaining step towards the proof of \allOfAKindThm\ is to determine when will $\OO$ contain an $\o$ with such a \basis.  The following lemma makes use of \independenceLem\ to give us a characterization of \bases.  This is then used in \allOfAKindThm\ to determine when will $\OO$ contain an $\o$ with a \basis\ of size $\d-\r$.

\begin{framed}
\begin{myLemma}[Bases Characterization]
\label{basisCharLem}
Let $\o$ be given and $\O$ be an \independentO\ set.  Let $\bar{\O} = \O \cup \o$.  $\O$ is a \basis\ of $\o$ iff $\bar{\m}<\bar{\n}+\r$ and $\bb{\m} \geq \bb{\n}+\r$ for every $\bb{\O} \subsetneq \bar{\O}$.
\end{myLemma}
\end{framed}
\begin{proof}
($\Rightarrow$)  Suppose $\O$ is a \basis\ of $\o$.  Then $\bar{\O}=\O \cup \o$ is \dependentO.  By \independenceLem, $\bar{\m}<\bar{\n}+\r$.  Let $\bb{\O} \subsetneq \bar{\O}$.  By our definition of \basis, $\bb{\O}$ is \independentO.  Again, by \independenceLem, $\bb{\m} \geq \bb{\n}+\r$.  Since $\bb{\O}$ was arbitrary, we conclude that $\bb{\m} \geq \bb{\n}+\r$ for every $\bb{\O} \subsetneq \bar{\O}$, as desired.

($\Leftarrow$)  Assume that $\bb{\m} \geq \bb{\n}+\r$ for every $\bb{\O} \subsetneq \bar{\O}=\O \cup \o$.  This implies by \independenceLem\ that $\O$ is \independentO\ and $\o$ is not \dependento\ on any subset of $\O$.  Further assume that $\bar{\m}<\bar{\n}+\r$.  By \independenceLem, $\bar{\O}$ is \dependentO.  Then $\o$ is \dependento\ on $\O$, and so $\O$ is a \basis\ of $\o$.
\end{proof}

\subsection{Proof of \allOfAKindThm}
\label{proofAllOfAKindSec}
We are finally ready to give the proof of \allOfAKindThm, which comes directly as a consequence of Lemmas \ref{independenceLem}, \ref{allOfAKindLem} and \ref{basisCharLem}.  Notice that, in a nutshell, the condition of \allOfAKindThm\ is that $\OO$ contains a set $\o$ with a \basis\ $\bb{\O}$ of size $\d-\r$.

\begin{proof}
($\Rightarrow$) Assume $\exists$ $\bb{\O} \subset \OO$ of size $\d-\r+1$ such that $\m \geq \n+\r$ for every $\O \subsetneq \bb{\O}$.  Let $\o \in \bb{\O}$.  By \basisCharLem, $\bb{\O} \backslash \o$ is a \basis\ of $\o$ of size $\d-\r$.  By \allOfAKindLem, $\ki{i}=\ki{\ii}$ for every $(\i,\ii)$.  By the same arguments as in the proof of such lemma, $\sstari{i}$ is the only $\r$-dimensional subspace that \fitsO\ $\OO$.

($\Leftarrow$)  Assume $\nexists$ $\o$ with a \basis\ of size $\d-\r$.  By \allOfAKindLem\ $\ki{i}$ {\em might} be different from $\ki{\ii}$ for some $(\i,\ii)$.  Furthermore, $\OO$ {\em might} not even contain an \independentO\ set of size $\d-\r$, whence, by \uniquenessThm, there could be infinitely many $\r$-dimensional subspaces that \fitO\ $\OO$.  Moreover, even if $\OO$ contains such a set, implying by \independenceLem\ that there is only one $\r$-dimensional subspace that \fitsO\ $\OO$, since $\ki{i}$ {\em might} be different from $\ki{\ii}$ for some $(\i,\ii)$, such subspace {\em might} not even be equal to any of the subspaces in $\Sstar$.
\end{proof}

We conclude the section with a nice converse that comes as a direct consequence of \allOfAKindThm.

\begin{framed}
\begin{myCorollary}[Converse of \allOfAKindThm]
\label{converseAllOfAKindCor}
Assume $\OO$ satisfies the assumptions of \allOfAKindThm.  There exists no $\r$-dimensional that \fitsO\ $\OO$ iff $\ki{i} \neq \ki{\ii}$ for some $(\i,\ii)$.
\end{myCorollary}
\end{framed}

\pagebreak
\section{Intuitively speaking}
\label{intuitivelySpeakingSec}
In this section we give some intuitive explanations of our results from \textsection \ref{uniquenessSec}, as well as the key ideas behind them.

The idea behind $\uu$ is that every $\r$-dimensional subspace $\s$ that \fitso\ $\o$ must satisfy $\so = \sstaro$.  Recall that $\hats$ is the projection of $\s$ onto $\Rdo{\omega}$, so essentially, the condition $\so = \sstaro$ is telling us that the projections of $\s$ and $\sstar$ onto $\Rdo{\omega}$ are the same, i.e., the projection onto $\Rdo{\omega}$ of every vector in $\s$ lies in $\hatsstar$.

\begin{figure}[H]
\centering
\includegraphics[width=4.5cm]{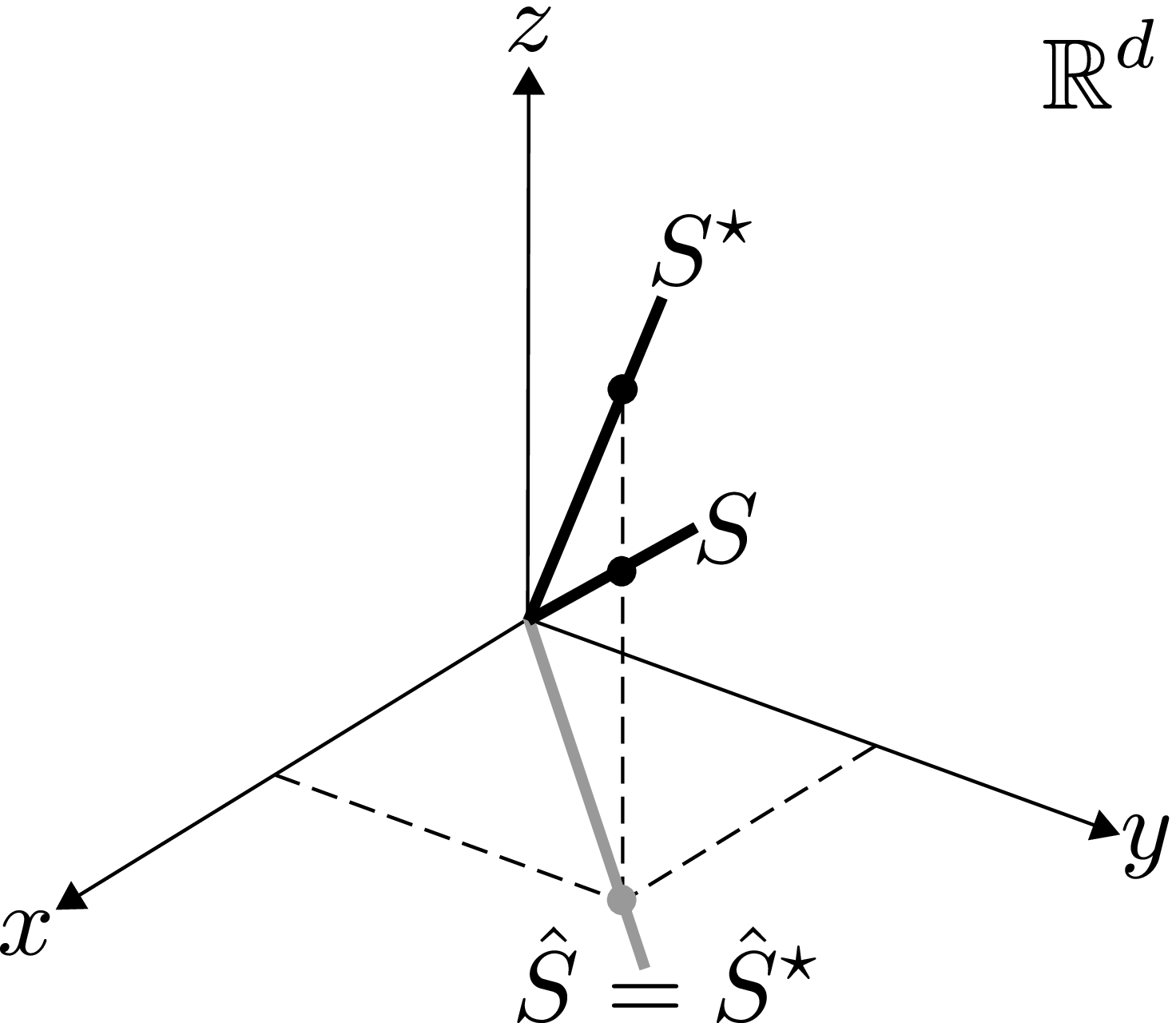}
\caption{The condition $\so = \sstaro$ is equivalent to saying that $\hats$ and $\hatsstar$, the projections of $\s$ and $\sstar$ onto $\Rdo{\omega}$, are the same.}
\label{intuitivelyProjectionFig}
\end{figure}

Essentially, $\uu$ characterizes all the subspaces that \fito\ $\o$ by characterizing all the vectors whose projection onto $\Rdo{\omega}$ lies in $\hatsstar$; these are precisely the vectors that satisfy $\uo \in \sstaro$.

Let us recall that the main goal of \textsection \ref{uniquenessSec} is to determine when there is only one $\r$-dimensional subspace that \fitsO\ $\O$.  Since all the $\r$-dimensional subspaces that \fitO\ $\O$ are contained in $\UU$, which is just $\bigcap_{\i \in \I} \uui{i}$, iff we can make sure that $\UU$ is an $\r$-dimensional subspace, we will be sure that there is only one $\r$-dimensional subspace that \fitsO\ $\O$.  In order to do so, we can characterize all the vectors $\u$ that lie in $\UU$.

The key intuitive idea to do so is that the entries of a vector of an $\r$-dimensional subspace are determined given only $\r$ of its entries.  Since every $\u \in \uu$ must \fitxo\ in $\sstar$, every $\uo$ must lie in $\sstaro$.  What that means is that one entry of every $\uo$ is constrained as a function of $\sstar$ and the other $\r$ entries of $\uo$.  More specifically for every $\u \in \uu$ and for any $\j \in \o$ and \phantomsection\label{jcDef}$\jc:=\o \backslash \j$, $\uj{j}$ is determined given $\uj{\jc}$ and $\sstar$.

\begin{figure}[H]
\centering
\includegraphics[width=5cm]{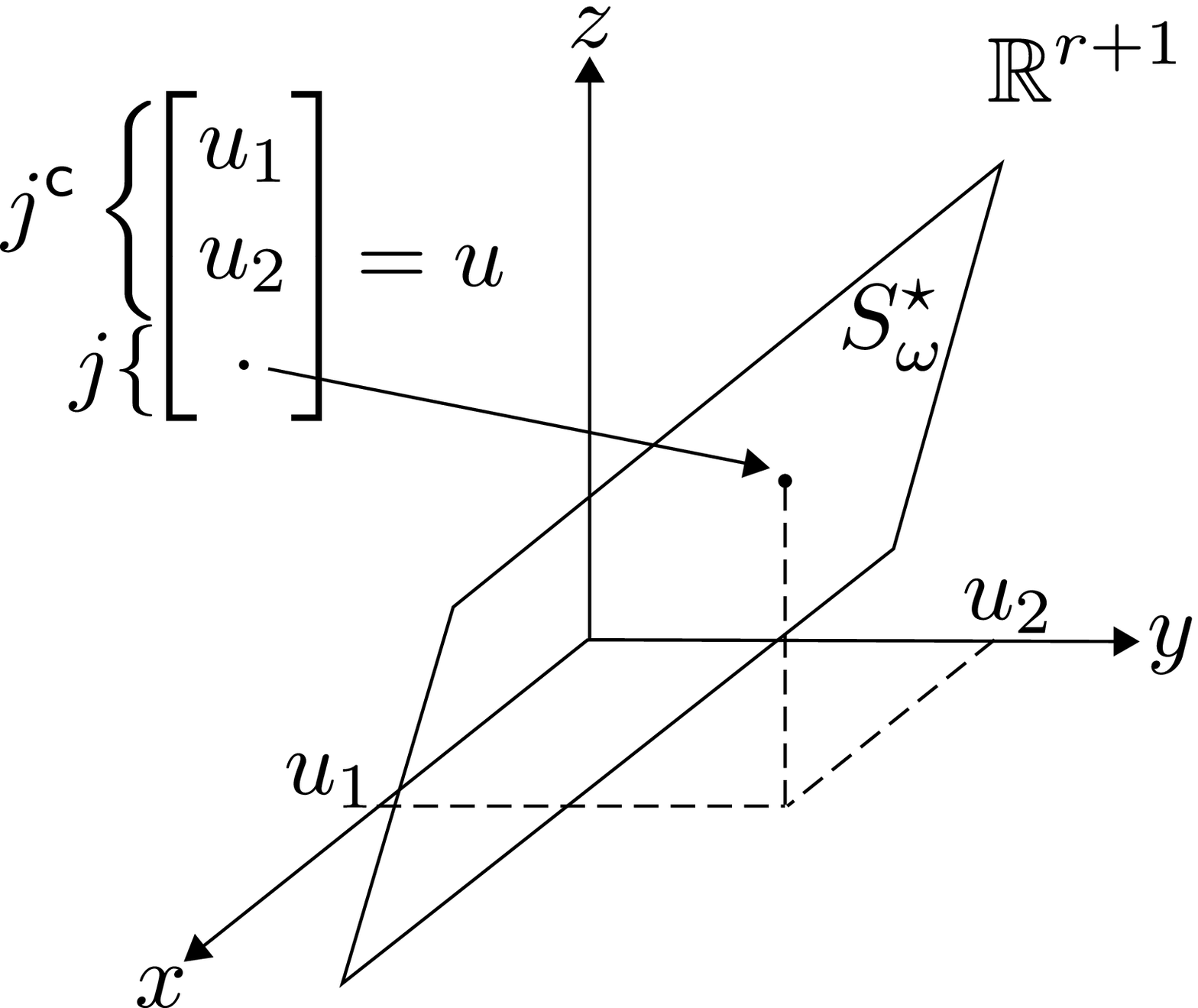}
\caption{Suppose $\r=2$.  One entry of every $\uo \in \uu$, say $\uj{j}$, is determined given the remaining two entries of $\uo$, say $\uj{\jc}$.}
\label{oneAtATimeFig}
\end{figure}

To see how exactly an entry of $\uo$ is constrained by the other $\r$ entries of $\uo$, observe that since $\uo$ must lie in $\sstaro$, $\u$ must satisfy
\begin{align*}
\uo = \Ustaro \gamma
\end{align*}
for some $\gamma \in \R^\r$.  Of course, we can rewrite this as
\begin{align*}
\left[\begin{matrix}
\uj{\jc} \\
\uj{j}
\end{matrix}\right] &= 
\left[\begin{matrix}
\Ustari{\jc} \\
\Ustari{j}
\end{matrix}\right]
\gamma.
\end{align*}
We can focus on the top block and solve for $\gamma$:
\begin{align*}
\gamma = (\Ustari{\jc})^{-1} \uj{\jc},
\end{align*}
where we know $(\Ustari{\jc})^{-1}$ exists, as $\sstar$ is non-\degenerate.  We then focus on the bottom block to obtain:
\begin{align}
\label{ujEq}
\uj{j} &= \underbrace{\Ustari{j} (\Ustari{\jc})^{-1}}_{\hat{\var{a}}} \uj{\jc},
\end{align}
where the only unknowns are $\uj{j}$ and $\uj{\jc}$, so we may write
\begin{align*}
\underbrace{
\left[\begin{matrix}
1 & -\Ustari{j} (\Ustari{\jc})^{-1}
\end{matrix}\right]}_{\aoT}
\left[\begin{matrix}
\uj{j} \\ \uj{\jc}
\end{matrix}\right] = 0,
\end{align*}
or even as
\begin{align*}
\a \u = 0,
\end{align*}
where $\a$ has the entries of $\ao$ in the positions of $\o$ and zeros elsewhere, i.e.
\begin{align*}
\underbrace{
\left[\begin{matrix}
 \hspace{.5cm} & 1 & \hspace{.5cm} & -\Ustari{j} (\Ustari{\jc})^{-1} & \hspace{.5cm}
\end{matrix}\right]}_{\begin{matrix}a \\ 1 \times \d \end{matrix}}
\underbrace{
\left[\begin{matrix}
 \\ \\ \uj{j} \\ \\ \\ \uj{\jc} \\ \\ \\
\end{matrix}\right]}_{\begin{matrix} \u \\ \d \times \r \end{matrix}} = 0,
\end{align*}
where the blank spaces represent zeros.

Notice that $\a$ is just as in \textsection \ref{oneAtATimeSec}, and as we mentioned there, $\a$ has exactly $\r+1$ non-zero entries in the positions of $\o$.  The $1$ will always be in the $\j$ position, and $-\Ustari{j} (\Ustari{\jc})^{-1}$ in the $\jc$ positions.  Also, $\a$ depends on $\sstar$, but the choice of the basis $\Ustar$ will only \----if anything\---- scale $\a$.  In other words, $\a\u=0$ will always describe the same system of equations, no matter the choice of $\Ustar$ and $\j$.

In conclusion, since $\j$ in our discussion was an arbitrary element of $\o$, we conclude that $\o$ constrains one entry of $\uo$, namely $\uj{j}$, as a function of $\sstar$ and the other entries of $\uo$, namely $\uj{\jc}$, through the equation $\a\u=0$.  By constraining such entry, we are making sure that $\uo$ is {\em aligned} with $\sstaro$, i.e., we are making sure that $\uo$ \fitsxo\ in $\sstaro$.

\begin{myExample}
\normalfont
With $\r=2$,
\begin{align*}
\o = \left[\begin{matrix} 
	\see \\
	\see \\
	\miss \\
	\see \\
	\miss \\
	\miss \\
	\miss
	\end{matrix}\right] \hspace{.5cm} \implies \hspace{.5cm}
u = \left[\begin{array} {cc}
	\multicolumn{2}{c}{\multirow{2}{*}{\Scale[1.5]{\uj{\jc}}}} \\ \\ \hline
	\\ \hline
	\multicolumn{2}{c}{\uj{j}} \\ \hline
	\\ \\ \\
	\end{array}\right]
\begin{array}{l}
	\\ \\ \\
	\text{$\longleftarrow$ Determined given $\uj{\jc}$ and $\sstar$.} \\
	\\
	\\
	\\
\end{array}
\end{align*}
$\blacksquare$
\end{myExample}

Under the setup of \textsection \ref{severalAtOnceSec}, each linearly independent $\ai{i}$ is giving us an linearly independent vector orthogonal to a distinct projection of $\sstari{i}$ onto a distinct $\Rdo{\omega_i}$.  Recall that $\hatsi{i}$ is the projection of $\s$ onto $\Rdo{\omega_i}$.  So when we constrain the vectors $\u \in \UU$ through multiple $\ai{i}$'s, essentially what we are doing is making sure that the projections of $\UU$ and $\sstari{i}$ onto $\Rdo{\omega_i}$ are the same {\em simultaneously} for every $\i$.

\begin{figure}[H]
\centering
\includegraphics[width=5cm]{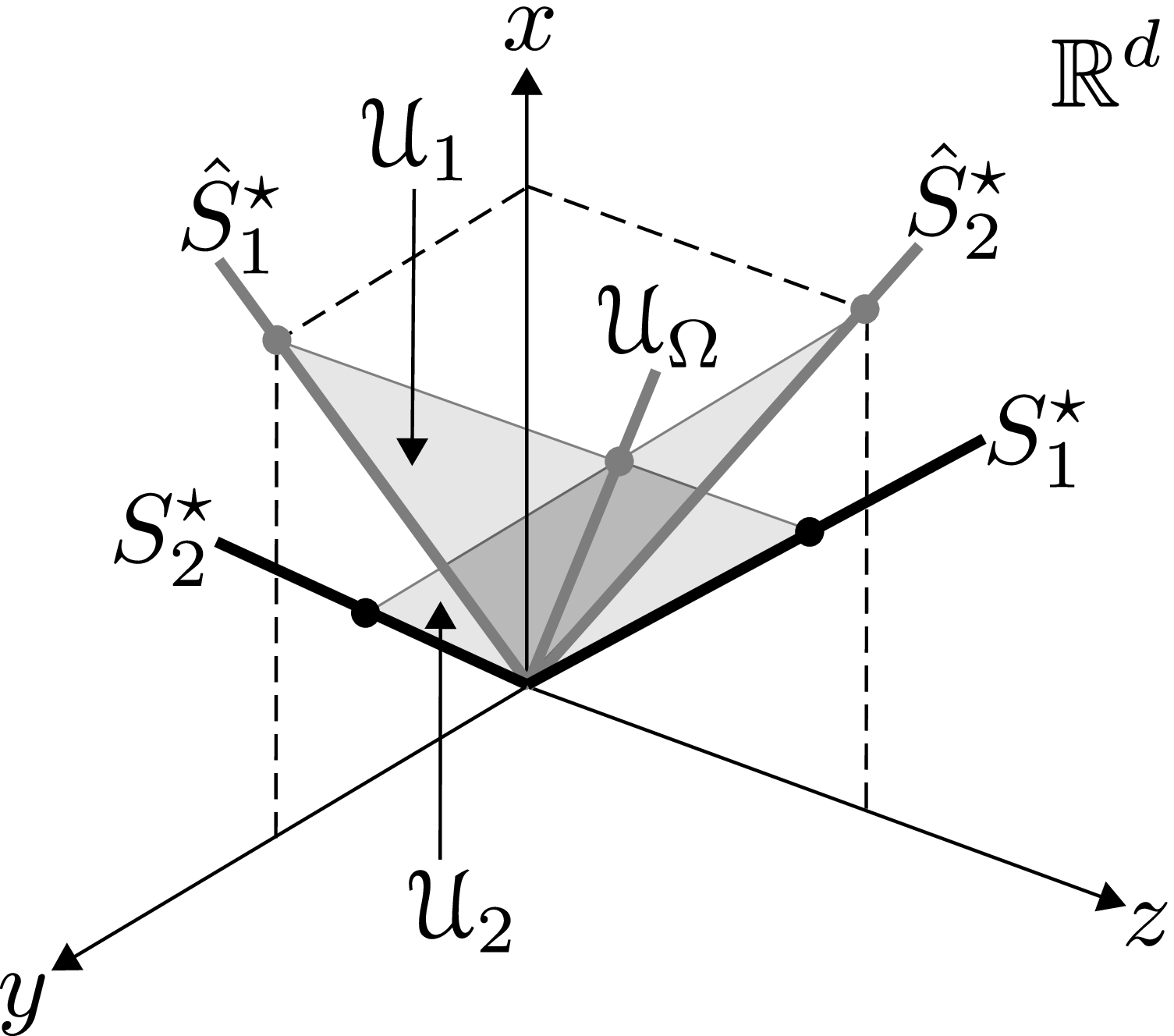}
\label{intuitiveUUFig}
\caption{With the same setup as in \fittingOEg, $\oi{1}$ would project onto the $(x,y)$-plane and $\oi{2}$ onto the $(x,z)$-plane.  So when we use $\A$ to constrain several entries of $\u$, basically what we are doing is making sure that the projection of $\UU$ and $\sstari{1}$ onto $\Rdo{\omega_1}$, and the projections of $\UU$ and $\sstari{2}$ onto $\Rdo{\omega_2}$ are the same.}
\end{figure}

Under the setup of this section, every $\u \in \UU$ must satisfy $\uoi{i} \in \sstaroi{i}$, for every $\i \in \I$.  Since each $\oi{i}$ constrains one entry of $\uo$, any $\u \in \UU$ will now have several constrained entries, determined as functions of other $\r$ entries of $\u$ through the equation $\A\u=0$, with $\A$ as defined in \textsection \ref{severalAtOnceSec}.  How many constrained entries will $\u \in \UU$ have?  As many as \independentO\ $\o$'s are contained in $\O$, knowing that there are at most $\d-\r$.

In other words, each \independentO\ $\oi{i}$ is constraining one entry of $\uoi{i}$, say $\uj{j_i}$, according to the other $\r$ entries of $\uoi{i}$, say $\uj{\jci{i}}$.  And no matter how many $\oi{i}$'s we have, there can be at most $\d-\r$ \independentO\ ones, i.e., we cannot constrain more than $\d-\r$ $\uj{j}$'s; this is formalized in \LCor.

\begin{myExample}
\normalfont
With $\r=2$,
\begin{align*}
\{\oi{1},\oi{2}\} = \left[\begin{matrix} 
	\see & \miss\\
	\see & \miss \\
	\see & \miss \\
	\miss & \miss \\
	\miss & \see \\
	\miss & \see \\
	\miss & \see
	\end{matrix}\right] \implies
\u = \left[\begin{array} {cc}
	\multicolumn{2}{c}{\multirow{2}{*}{\Scale[1.5]{\uj{\jci{1}}}}} \\ \\ \hline
	\multicolumn{2}{c}{\uj{\j_1}} \\ \hline
	\\ \hline
	\multicolumn{2}{c}{\multirow{2}{*}{\Scale[1.5]{\uj{\jci{2}}}}} \\ \\ \hline
	\multicolumn{2}{c}{\uj{\j_2}}
	\end{array}\right]
\begin{array}{l}
	\\ \\
	\text{$\longleftarrow$ $\j_1$.  Determined by $\sstaroi{1}$ and $\uj{\jci{1}}$.} \\
	\\
	\\ \\
	\text{$\longleftarrow$ $\j_2$.  Determined by $\sstaroi{2}$ and $\uj{\jci{2}}$.}
\end{array}
\end{align*}
$\blacksquare$
\end{myExample}

\subsection{Using Lemma \ref{independenceLem}}
\label{usingIndependenceLem}
Intuitively, what \textsection \ref{independenceSec} is telling us is that every \independento\ $\o$ constrains an entry of every $\u \in \UU$ out of the $\d-\r$ that may be constrained, while \textsection \ref{allYouNeedSec} says that iff we can constrain $\d-\r$ entries of $\u$, we will be guaranteed that there is only one $\r$-dimensional subspace that \fitsO\ $\O$.

In other words, all we need to do is find an \independentO\ set of size $\d-\r$.  \independenceLem\ then completes the picture by telling us that a set $\bb{\O}$ is \independentO\ iff $\m \geq \n+\r$ for every $\O \subset \bb{\O}$.

\begin{myExample}
\label{intuitionEg}
\normalfont
Consider the same setup as in \fittingOEg, we can see that $\OO$ is the only possible set with $\d-\r$ $\o$'s.  One can trivially verify that $\m \geq \n+\r$ for every $\O \subset \OO$, hence there is only one $\r$-dimensional subspace that \fitsO\ $\OO$.

Intuitively, what this means is that given whichever entry of $\u \in \UU$, one of the remaining entries of $\u$ is constrained according to $\sstaroi{1}$ and an other of the remaining entries of $\u$ is constrained according to $\sstaroi{2}$.

\begin{figure}[H]
\centering
\includegraphics[width=5cm]{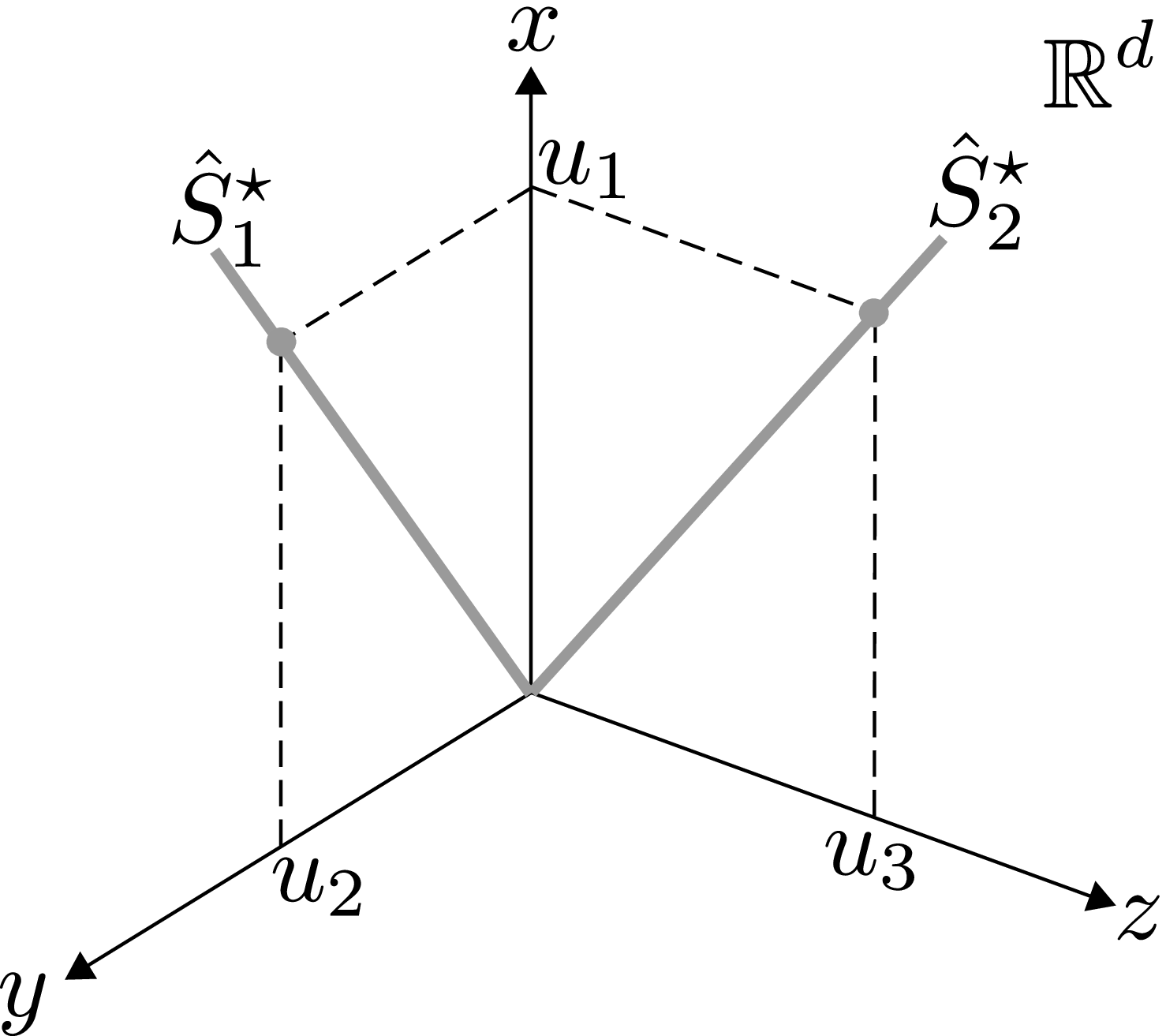}
\caption{Given any entry of $\u$, its remaining ones are constrained according to $\sstaroi{1}$ and $\sstaroi{2}$ in order to guarantee that $\uoi{i}$ \fitsO\ in $\sstaroi{i}$ for every $\i$.}
\label{severalFig}
\end{figure}

For example, if $\uj{1}$ is given, $\uj{2}$ and $\uj{3}$ are constrained according to $\sstaroi{1}$ and $\sstaroi{2}$.
\begin{align*}
\U = \left[\begin{matrix}
\uj{1} \\ \uj{2} \\ \uj{3}
\end{matrix}\right]
\begin{array}{l}
\longleftarrow \text{Given.} \\
\longleftarrow \text{Constrained according to $\sstaroi{1}$.} \\
\longleftarrow \text{Constrained according to $\sstaroi{2}$.}
\end{array}
\end{align*}
Explicitly, using \eqref{ujEq} we can obtain:
\begin{align*}
\uj{2} &= \Ustari{2} \uj{1} / \Ustari{1}, \\
\uj{3} &= \Ustari{3} \uj{1} / \Ustari{1}.
\end{align*}
$\blacksquare$
\end{myExample}

Conversely, $\OO$ is only constraining fewer than $\d-\r$ entries of $\u$ if there is no such $\O$, i.e., there would be least one {\em free} entry of $\u$ to choose arbitrarily, whence there would be infinitely many $\r$-dimensional subspaces that \fitO\ $\OO$.

\begin{myExample}
\label{infinitelyManySubspacesEg}
\normalfont
Consider $\r=2$ and
\begin{align*}
\OO = \left[\begin{matrix} 
	\see & \miss & \see & \miss \\
	\see & \see & \miss & \miss \\
	\see & \see & \see & \miss \\
	\miss & \see & \see & \see \\
	\miss & \miss & \miss & \see \\
	\miss & \miss & \miss & \see
	\end{matrix}\right].
\end{align*}
Again, $\OO$ is the only possible set with $\d-\r$ $\o$'s.  Nevertheless, $\OO$ does not satisfy the condition that $\m \geq \n+\r$ for every one of its subsets $\O$.  Specifically, we can see that $\O=\{\oi{1},\oi{2}, \oi{3} \}$ is a subset of $\OO$ with $\m<\n+\r$.  By \uniquenessThm, there are infinitely many $\r$-dimensional subspaces that \fitO\ $\OO$.

The fact that $\O=\{\oi{1},\oi{2}, \oi{3} \}$ satisfies $\m<\n+\r$ implies that one of the $\o$'s of $\O$ is \redundant.  This implies that there are at most $\d-\r-1$ \independentO\ $\o$'s in $\OO$.  Hence, there will be at least one {\em free} entry of $\u$ to choose arbitrarily \----for example, one of the last two\---- and so there are infinitely many $\r$-dimensional subspaces that \fitO\ $\OO$.
$\blacksquare$
\end{myExample}

\subsection{The idea behind Lemma \ref{independenceLem}}
\label{ideaBehindIndependenceLemSec}
The condition of \independenceLem\ is extremely simple and concrete; it depends only on the most elemental invariants of $\O$: essentially, cardinalities of its subsets.  In this section we will explain the intuition that led to this lemma.

One of the directions, that $\bb{\O}$ is \dependentO\ if $\m<\n+\r$ for some $\O \subset \bs{\O}$, is directly given by \LCor.  The idea behind it is that if we think of $\n$ as the number of equations in our system $\A\u=0$, and $\m$ as the number of unknowns, since $\dim \ker \A \geq \r$ (otherwise there would be no $\r$-dimensional subspace that \fitsO\ $\O$, as every $\r$-dimensional subspace that \fitsO\ $\O$ is contained in $\UU=\ker\A$), then $\A$ must have at most $\m-\r$ linearly independent equations.  If it has more than $\m-\r$ equations, then it must contain some linearly dependent ones, or else $\dim \ker \A < \r$.

The other direction, that if $\bb{\O}$ is dependent, it must contain a subset $\O$ with $\m <\n+\r$, is essentially given by \basisLem.  We will spend the rest of this section explaining the intuition behind the later.

On one hand, we have that if $\O$ is a \basis\ of $\o$, then $\o$ is minimally \dependento\ on $\O$.  On the other one, we know that $\o$ determines $\uj{j}$ as a function of $\uj{\jc}$ through $\a\u=0$, so $\o$ being \dependento\ on $\O$ intuitively means that $\O$ also determines $\uj{j}$ as a function of $\uj{\jc}$ and $\Sstar$ through $\A\u=0$.

The intuition behind \basisLem\ is that given $\uj{\jc}$, if $\O$ is a \basis\ of $\o$, $\uj{j}$ cannot be determined by $\O$ if there is at least one undetermined entry of $\u$ in the positions of $\J \backslash \jc$.  Why would this be true? Well, let $\uj{\jc}$ be given.  If $\O$ is a trivial \basis, it is clear that $\j$ is the only element of $\J \backslash \jc$, and $\uj{j}$ is therefore determined by $\O$ given $\uj{\jc}$.

If $\O$ is not a trivial \basis, since $\O$ is independent, $\O$ constrains $\n$ entries of $\uj{\J \backslash \jc}$.  Let $\bs{J}$ denote the positions of such entries.  Explicitly, each $\oi{i} \in \O$ determines one entry of $\uj{\bs{J}}$, say $\uj{j_i}$.  Fix this correspondence, and let $\jci{i}:=\oi{i} \backslash \j_\i$.

Since $\o$ is \dependento\ on $\O$, $\j \in \bs{J}$, i.e., there must be some $\oi{I_0} \in \O$, that determines $\uj{j}$ as a function of $\uj{\jci{I_0}}$.  But how could $\oi{I_0}$ determine $\uj{j}$ if there is an undetermined entry in $\uj{\jci{\I_0}}$?  It can't, unless some of the $\uj{\jci{i}}$'s cancel out, which almost surely won't happen.  Then almost surely $\uj{\jci{I_0}}$ must be determined, i.e., $\jci{I_0} \subset \bs{J}$.

Now it is convenient to define $J_0 := \j$, and for $t>0$, $I_t$ as the set of $\i$'s such that $\oi{i}$ determines an entry of $\u$ in the positions of $J_t$, and $J_t$ as the set of positions of $\u$ that have not yet been determined at time $t-1$ and belong to some of the $\oi{i}$ with $\i \in I_{t-1}$.  More precisely,
\begin{align*}
J_t &:= \left( \bigcup_{\i \in I_{t-1}} \oi{i} \right) \backslash \left( \bigcup_{\tau=0}^{t-1} J_{\tau} \right) \backslash \jc,\\
I_t &:= \{\i :\j_\i \in J_t \cap \bs{J} \}.
\end{align*}
To give some intuitive meaning to this, consider the Tanner graph defined by the set $\O$ with disjoint sets of {\em column} vertices $\I$ and {\em row} vertices $\J$, where there is an edge between row vertex $\j \in \J$ and column vertex $\i \in \I$ if $\j \in \oi{i}$, and there is a special correspondence between $\i \in \I$ and $\j_\i \in \bs{J}$ indicating that $\oi{i}$ determines $\uj{j_i}$.

For example, one such graph could look like the following, where bold edges represent the correspondence between $\i \in \I$ and $\j_\i \in \bs{J}$, dashed vertices represent $\jc$, and bold vertices represent the elements of $\bs{J}$.

\begin{center}
	\begin{tikzpicture}
		\node [title] (vdots) at (1,-7) {$\vdots$};
		\node [title] (j) at (0,.25) {$\J$};
		\node [title] (i) at (2,.25) {$\I$};
		
		\node [label] (j1) at (0,-.5){$$}; \node [vertex,dashed] (j1) at (0,-.5){};
		\node [label] (j2) at (0,-1) {$$}; \node [vertex,dashed] (j2) at (0,-1) {};
		\node [label] (j3) at (0,-1.5) {$\j$}; \node [vertex, line width=1.5pt] (j3) at (0,-1.5) {};
		\node [label] (j4) at (0,-2.5) {$$}; \node [vertex, line width=1.5pt] (j4) at (0,-2.5) {};
		\node [label] (j5) at (0,-3) {$$}; \node [vertex, line width=1.5pt] (j5) at (0,-3) {};
		\node [label] (j6) at (0,-4) {$$}; \node [vertex, line width=1.5pt] (j6) at (0,-4) {};
		\node [label] (j7) at (0,-4.5) {$\j_\i$}; \node [vertex, line width=1.5pt] (j7) at (0,-4.5) {};
		\node [label] (j8) at (0,-5) {$$}; \node [vertex, line width=1.5pt] (j8) at (0,-5) {};
		\node [label] (j9) at (0,-6.5) {$$}; \node [vertex] (j9) at (0,-6.5) {};
		\node [label] (j10) at (0,-7) {$$}; \node [vertex] (j10) at (0,-7) {};
		
		\node [title](oi) at (-1.2,-1){$\left. \begin{matrix} \\ \\ \\ \\ \end{matrix} \o \right\{ $};
		\node [title](ooi) at (-.58,-.75){$\left. \begin{matrix} \\ \\ \end{matrix} \jc \right\{ $};
		\node [title](ooic) at (-.58,-1.5){$\left. \begin{matrix} \\ \end{matrix} J_0 \right\{ $};
		
		\node [title](jj1) at (-1,-2.75){$\left. \begin{matrix} \\ \\ \end{matrix} J_1=\jci{I_0} \right\{ $};
		
		\node [title](jj2) at (-.6,-4.5){$\left. \begin{matrix} \\ \\ \\ \\ \end{matrix} J_2 \right\{ $};
		\node [title](jj3) at (-.6,-6.75){$\left. \begin{matrix} \\ \\ \end{matrix} J_3 \right\{ $};
		
		\node [label] (i1) at (2,-2) {$$}; \node [vertex] (i1) at (2,-2) {};
		\node [label] (i2) at (2,-3) {$$}; \node [vertex] (i2) at (2,-3) {};
		\node [label] (i3) at (2,-3.5) {$$}; \node [vertex] (i3) at (2,-3.5) {};
		\node [label] (i4) at (2,-5.5) {$$}; \node [vertex] (i4) at (2,-5.5) {};
		\node [label] (i5) at (2,-6) {$\i$}; \node [vertex] (i5) at (2,-6) {};
		\node [label] (i6) at (2,-6.5) {$$}; \node [vertex] (i6) at (2,-6.5) {};
		
		\node [title](ii1) at (2.6,-2){$\left\} \begin{matrix} \\ \end{matrix} I_0 \right.$};
		\node [title](ii1) at (2.6,-3.25){$\left\} \begin{matrix} \\ \\ \end{matrix} I_1 \right.$};
		\node [title](ii2) at (2.6,-6){$\left\} \begin{matrix} \\ \\ \\ \\ \end{matrix} I_2 \right.$};
				
		\draw [font=\scriptstyle]
				(j3) edge [line width=1.5pt] (i1)
				(j4) edge (i1)
				(j5) edge (i1)
				(j4) edge[line width=1.5pt] (i2)
				(j5) edge[line width=1.5pt] (i3)
				(j6) edge (i2)
				(j2) edge[dashed] (i2)
				(j7) edge (i3)
				(j8) edge (i3)
				(j6) edge[line width=1.5pt] (i4)
				(j7) edge[line width=1.5pt] (i5)
				(j8) edge[line width=1.5pt] (i6)
				(j3) edge[dashed] (i4)
				(j4) edge (i4)
				(j5) edge (i5)
				(j9) edge (i5)
				(j9) edge (i6)
				(j10) edge (i6);
	\end{tikzpicture}
\end{center}
Since $J_1 = \jci{I_0}$, $J_t \subset \bs{J}$ for $t=0,1$.  We can now do induction on $t$, and see that if there is an undetermined $\j^* \in J_{t+1}$, there will almost surely be an undetermined $\j_\i \in J_t$, as there will be an $\i \in I_t$ determining $\uj{j_i}$ as a function of $\uj{\jci{i}}$.  And how can $\uj{j_i}$ be determined, if it is a function of $\uj{\jci{i}}$, an there is an undetermined entry in $\uj{\jci{i}}$, namely $\j^*$?   Again, it can't, unless some of the $\uj{\jci{i}}$'s cancel out, which almost surely won't happen.  Then almost surely $\uj{J_{t+1}}$ must be determined, i.e., $J_{t+1} \subset \bs{J}$.
\begin{center}
	\begin{tikzpicture}
		\node [title] (vdots) at (1,-4) {$\vdots$};
		\node [title] (j) at (0,-3.25) {$\J$};
		\node [title] (i) at (2,-3.25) {$\I$};
		\node [title] (then) at (-1.8,-5.7) {$\Uparrow$};
		\node [title] (vdots) at (1,-7) {$\vdots$};
		
		\node [label] (j6) at (0,-4) {$$}; \node [vertex, line width=1.5pt] (j6) at (0,-4) {};
		\node [label] (j7) at (0,-4.5) {$\j_\i$}; \node [vertex, line width=1.5pt] (j7) at (0,-4.5) {};
		\node [label] (j8) at (0,-5) {$$}; \node [vertex, line width=1.5pt] (j8) at (0,-5) {};
		\node [label] (j9) at (0,-6.5) {$\j^*$}; \node [vertex] (j9) at (0,-6.5) {};
		\node [label] (j10) at (0,-7) {$$}; \node [vertex, line width=1.5pt] (j10) at (0,-7) {};
		
		\node [title](jj2) at (-0.5,-4.525){$\rightarrow$};
		\node [title](jj2) at (-2.7,-4.3){$\uj{j_i}$ depends on $\uj{\j^*}$, so it is};
		\node [title](jj2) at (-2.5,-4.75){not determined given $\uj{\jci{i}}!$};
		\node [title](jj2) at (-5.0,-4.5){$\left. \begin{matrix} \\ \\ \\ \\ \end{matrix} J_t \right\{ $};
		\node [title](ooic) at (-1.6,-6.5){not determined $\rightarrow$};
		\node [title](jj3) at (-3.4,-6.75){$\left. \begin{matrix} \\ \\ \end{matrix} J_{t+1} \right\{ $};
		
		\node [label] (i4) at (2,-5.5) {$$}; \node [vertex] (i4) at (2,-5.5) {};
		\node [label] (i5) at (2,-6) {$\i$}; \node [vertex] (i5) at (2,-6) {};
		\node [label] (i6) at (2,-6.5) {$$}; \node [vertex] (i6) at (2,-6.5) {};
		
		\node [title](ii2) at (2.6,-6){$\left\} \begin{matrix} \\ \\ \\ \\ \end{matrix} I_t \right.$};
				
		\draw [font=\scriptstyle]
				(j6) edge[line width=1.5pt] (i4)
				(j7) edge[line width=1.5pt] (i5)
				(j8) edge[line width=1.5pt] (i6)
				(j9) edge (i5)
				(j9) edge (i6)
				(j10) edge (i6);
	\end{tikzpicture}
\end{center}
We will see that $\J \backslash \jc$ is contained in the union of the $J_t$'s, and since all the entries of $u$ in the positions of the $J_t$'s are determined, we will conclude that all the entries of $\uj{\J \backslash \jc}$ are determined.  Since it takes one $\o$ to constrain one entry of $\u$, $\O$ must have $|\J \backslash \jc|=\m-\r$ sets, i.e., $\n=\m-\r$, which is precisely the statement of the lemma.  To see this, let $\J_t$ and $\I_t$ denote the unions of the $J_t$'s and the $I_t$'s, respectively.  
\begin{center}
	\begin{tikzpicture}
		\node [title] (j) at (0,.25) {$\J$};
		\node [title] (i) at (2,.25) {$\I$};
		
		\node [label] (j1) at (0,-.5){$$}; \node [vertex,dashed] (j1) at (0,-.5){};
		\node [label] (j2) at (0,-1) {$$}; \node [vertex,dashed] (j2) at (0,-1) {};
		\node [label] (j3) at (0,-1.5) {$$}; \node [vertex, line width=1.5pt] (j3) at (0,-1.5) {};
		\node [label] (j4) at (0,-2.5) {$$}; \node [vertex, line width=1.5pt] (j4) at (0,-2.5) {};
		\node [label] (j5) at (0,-3) {$$}; \node [vertex, line width=1.5pt] (j5) at (0,-3) {};
		
		\node [title](ooi) at (-.58,-.75){$\left. \begin{matrix} \\ \\ \end{matrix} \jc \right\{ $};
		\node [title](ooic) at (-.58,-1.5){$\left. \begin{matrix} \\ \end{matrix} J_0 \right\{ $};
		\node [title](jj1) at (-.6,-2.75){$\left. \begin{matrix} \\ \\ \end{matrix} J_1 \right\{ $};
		\node [title](jj1) at (-1.2,-2.25){$\left. \begin{matrix} \\ \\ \\ \\ \\ \end{matrix} \J_2 \right\{ $};
		
		\node [label] (i1) at (2,-2) {$$}; \node [vertex] (i1) at (2,-2) {};
		\node [label] (i2) at (2,-3) {$$}; \node [vertex] (i2) at (2,-3) {};
		\node [label] (i3) at (2,-3.5) {$$}; \node [vertex] (i3) at (2,-3.5) {};
		
		\node [title](ii1) at (2.7,-2){$\left\} \begin{matrix} \\ \end{matrix} I_0 \right.$};
		\node [title](ii1) at (2.7,-3.25){$\left\} \begin{matrix} \\ \\ \end{matrix} I_1 \right.$};
		\node [title](ii1) at (3.3,-2.725){$\left\} \begin{matrix} \\ \\ \\ \\ \\ \end{matrix} \I_t \right.$};
				
		\draw [font=\scriptstyle]
				(j3) edge [line width=1.5pt] (i1)
				(j4) edge (i1)
				(j5) edge (i1)
				(j4) edge[line width=1.5pt] (i2)
				(j5) edge[line width=1.5pt] (i3)
				(j1) edge[dashed] (i2)
				(j2) edge[dashed] (i2)
				(j2) edge[dashed] (i3)
				(j4) edge (i3);
	\end{tikzpicture}
\end{center}
Let $\bb{\O}=\{\oi{i} : i \in \I_t\}$, and notice that $\uj{j}$ is determined by $\bb{\O}$, which implies $\o$ is \dependento\ on $\bb{\O}$.  Moreover, since $\I_t \subset \I$, $\bb{\O} \subset \O$.  Since $\o$ is minimally \dependento\ on $\O$, this means that $\bb{\O}=\O$, i.e., $\I_t=\I$, hence $\J=\J_t \cup \jc$.  Notice that $|J_t|=|I_t|$ for every $t$, so $|\J_t|=|\I_t|$.  This implies
\begin{align*}
\m = |\J| = \underbrace{|\J_t|}_{|\I_t|} + |\jc| = \underbrace{|\I_t|}_{|\I|} + \r = \n + \r,
\end{align*}
as desired.

\independenceLem\ then comes as a direct consequence of \basisLem\ by noticing that if $\bs{\O}$ contains a \dependento\ $\o$, it must contain a \basis\ $\O$, and then the union of $\o$ and $\O$ will satisfy $\m<\n+\r$.

\subsection{About our assumptions}
\label{aboutOurAssumptionsSec}
It is clear that \rDimensionAss-\sizeoAss\ are our only {\em real} assumptions, if any.  And despite them being so lenient, for completeness we give a further discussion about them in this section.  Specially to highlight how elemental they are, describe their tight relation, and talk about their generalizations.

We start with \sizeoAss.  Let us not loose sight of our final goal: whenever $\s$ \fitsXO\ $\hatXI$, we want to use $\hatX$ to validate if $\hatXI$ really lies in $\s$, just as we would use the complete $\y$ (see \textsection \ref{introSec}): by saying that if $\s$ also \fitsXO\ $\hatX$, it is because both, $\hatX$ and $\hatXI$ really lie in $\s$.

If $ |\o| \leq\r$, by our non-\degenerate\ assumptions \sstarNonDegenerateAss\ and \existsNonDegenerateAss, $\so=\R^{|\o|}=\sstaro$ for every $\sstar \in \Sstar$.  This implies that every $\hatx$ will \fitxo\ in $\s$, whether or not $\x$ or $\hatXI$ really lie in $\s$, which totally defeats the purpose of $\hatx$.  In other words, $\hatx$ would be completely useless.

\begin{myExample}
\normalfont
With $\Ustari{1}$ and $\Ustari{2}$ as in \fittingOEg, if $|\o|\leq \r$, i.e., if $\o=\{\j\}$ or $\o=\emptyset$, it is easy to see that any non-\degenerate\ $\s$ \----in this case any subspace spanned by a $\U$ with no zero entries\---- would \fitXO\ every $\hatx$, no matter which subspace it belongs to.
$\blacksquare$
\end{myExample}

Of course, having such an $\o$ would not harm us at all; we could simply ignore it.  This way, \sizeoAss\ is a mere formality stating without loss of generality that we do not have useless information, knowing, of course, that if we had some, being so trivial to identify an $\o$ with $|\o|\leq\r$, we could simply discard it.

On the other hand, if $|\o|>\r+1$ it can only be easier to determine the missing entries of $\hatx$, and to determine if $\x$ really belongs to $\s$, as any subspace that \fitso\ a larger $\o$ will have to satisfy more restrictions.  More precisely, there are fewer vectors that satisfy $\uoi{1} \in \sstaroi{1}$ than vectors that satisfy $\uoi{2} \in \sstaroi{2}$ if $|\oi{1}|>|\oi{2}|$, whence fewer vectors in $\uui{1}$ than in $\uui{2}$.  In other words, there can only be fewer subspaces that \fitxo\ $\hatxi{1}$ than subspaces that \fitxo\ $\hatxi{2}$ if $|\oi{1}|>|\oi{2}|$.

And how could we use this information? Well, just as with $\y$, for a generic $\xo \in \sstaro$ with $|\o| > \r$, we have that
\begin{enumerate}
\item[\aProp]
There is only one $\r$-dimensional subspace of $\R^{|\o|}$ that fits $\xo$.
\item[\bProp]
Obvious, but essential, this single $\x$ lies in one and only one of the subspaces of $\Sstar$
\end{enumerate}
Using this observation, we can easily generalize \sizeoAss\ as follows: if $|\o|>\r+1$, we could {\em split} $\o$ into $\O$, a set of $|\o|-\r+1$ sets, each of size $\r+1$, that satisfies \eqref{allOfAKindThmEq}, such that $\O$ {\em behaves} just as $\o$ in the sense that if there is an $\r$-dimensional subspace that \fitsO\ $\O$:
\begin{enumerate}
\item[\aProp]
There is only one $\r$-dimensional subspace of $\R^{|\o|}$ that \fitsO\ $\O$, just as there is only one $\r$-dimensional subspace of $\R^{|\o|}$ that \fitso\ $\o$.
\item[\bProp]
$\ki{i}=\ki{\ii}$ for every $\i,\ii \in \I$, just as there is only one $\k$ for $\o$.
\end{enumerate}

\begin{myExample}
\normalfont
Suppose $\r=2$. Then we can split $\oi{1},\oi{2} \in \OO$ below as:
\begin{align*}
\OO = \left[ \begin{matrix}
\see & \miss \\
\see & \miss \\
\see & \miss \\
\see & \see \\
\see & \see \\
\miss & \see \\
\miss & \see \\
\miss & \see \\
\end{matrix} \right]
\sim
\left[ \begin{matrix} \\ \\ \\ \\ \\ \\ \\ \\ \end{matrix} \right.
\underbrace{
\begin{matrix}
\see & \miss & \miss & \see \\
\see & \see & \miss & \miss \\
\see & \see & \see & \miss \\
\miss & \see & \see & \see \\
\miss & \miss & \see & \see \\
\miss & \miss & \miss \\
\miss & \miss & \miss \\
\miss & \miss & \miss
\end{matrix}}_{\O_1} \hspace{.2cm}
\underbrace{
\begin{matrix}
\miss & \miss & \miss \\
\miss & \miss & \miss \\
\miss & \miss & \miss \\
\see & \miss & \miss & \see \\
\see & \see & \miss & \miss \\
\see & \see & \see & \miss \\
\miss & \see & \see & \see \\
\miss & \miss & \see & \see
\end{matrix}}_{\O_2}
\left. \begin{matrix} \\ \\ \\ \\ \\ \\ \\ \\ \end{matrix} \right].
\end{align*}
Notice that both, $\O_1$ and $\O_2$ satisfy the conditions of \allOfAKindThm.
$\blacksquare$
\end{myExample}

We would thus obtain several orthogonal directions from $\o$ rather than just one; more precisely, each $\oi{i}$ would define a $(|\oi{i}|-\r+1) \times \d$ matrix, say $A_i$, where each subset of $|\oi{i}|-\r$ rows of $A_i$ is full-rank.  We could then use $A_i$ just as $\ai{i}$ in \textsection \ref{severalAtOnceSec}, and construct $\A$ as
\begin{align*}
\A = \left[\begin{matrix}
A_1 \\ \vdots \\ A_\N
\end{matrix}\right],
\end{align*}
and all the statements would work as before, with some straightforward adaptations.

This way, rather than an assumption, $|\o|>\r$ in \sizeoAss\ is a requirement, while $|\oi{i}|=\r+1$ $\forall$ $\i$ is a convenience that simplifies our notation, but at the same time forces us to work with the {\em minimal} possible assumption on $|\o|$.

Now let us focus on \sstarNonDegenerateAss.  In \textsection \ref{oneAtATimeSec} we mentioned (\aEntriesLem) that the orthogonal direction of $\sstaro$, $\ao$, has exactly $\r+1$ entries, which we just confirmed in \textsection \ref{intuitivelySpeakingSec}.  But why do we even care that $\ao$ has all non-zero entries?  And what about the subspaces whose orthogonal direction has at least one non-zero entry?  The answers to these questions are tightly related to the concept of \degenerate\ subspaces, and help to give a better understanding of \sstarNonDegenerateAss\ and \sizeoAss.

An other interpretation of \aEntriesLem\ is that all hyperplanes that have an orthogonal direction $\ao$ with at least one zero \----the hyperplanes {\em aligned} with the {\em canonical} hyperplanes\---- are \degenerate.

\begin{figure}[H]
\centering
\includegraphics[width=5cm]{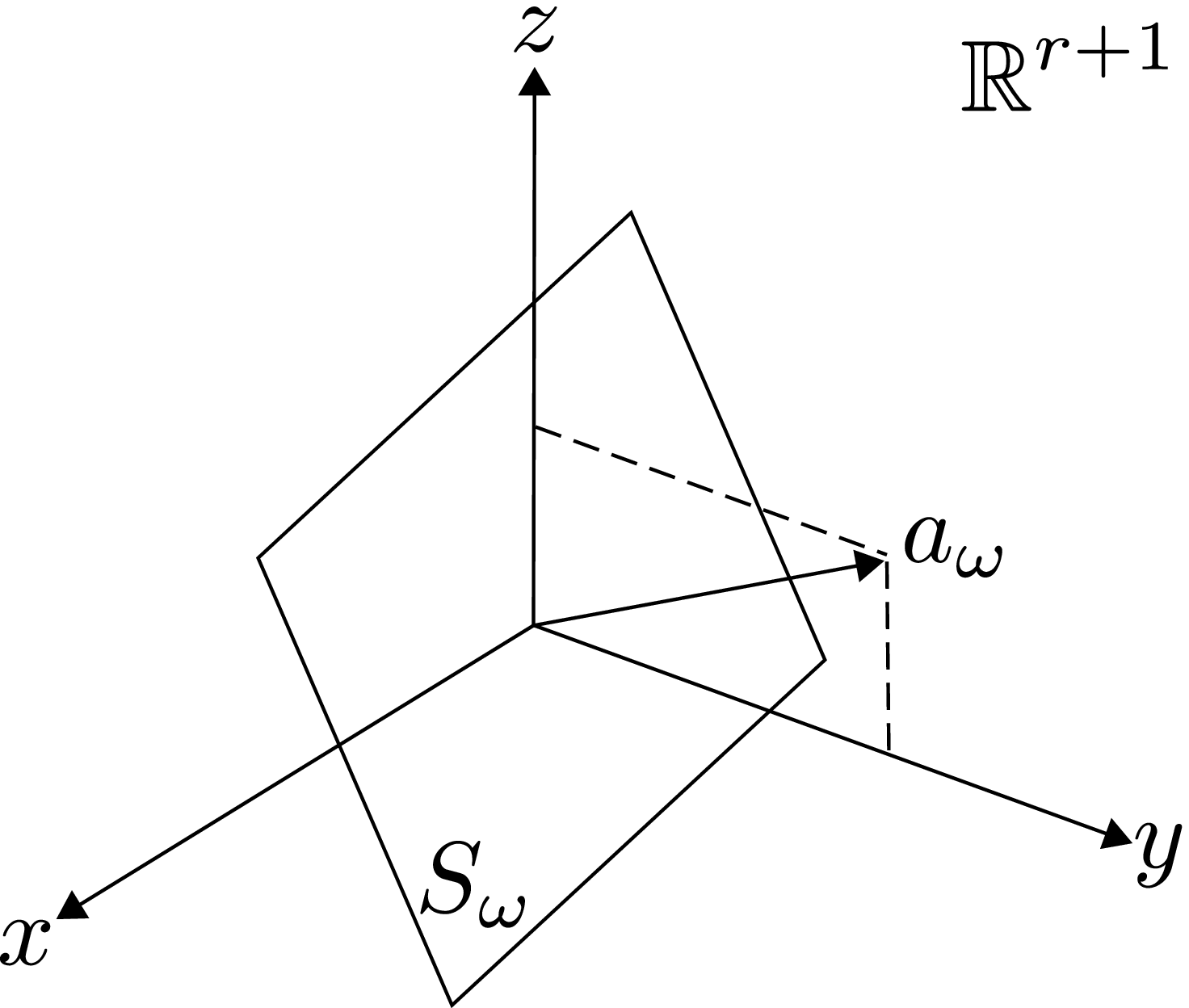}
\caption{The planes of $\R^3$ that have an $\ao$ with at least one zero \----the ones {\em aligned} with the {\em canonical} planes, $(x,y)$ $(x,z)$ or $(y,z)$\---- are \degenerate.}
\label{badaomegaFig}
\end{figure}

So, if $\ao$ had at least one zero entry, $\sstar$ would have to be \degenerate.  And how would that affect us? Well, if $\sstar$ were \degenerate, and $\Ustari{\jc}$ turned out to be rank-defficient, $(\Ustari{\jc})^{-1}$ would not exist, and we wouldn't be able to represent $\uj{j}$ as a function of $\uj{\jc}$ and $\sstar$ (see \textsection \ref{intuitivelySpeakingSec}).

\begin{myExample}
\normalfont
For an example of a \degenerate\ subspace, suppose $\sstaro$ is the $(y,z)$ plane in $\R^3$ and
\begin{align*}
\u=\left[ \begin{matrix} \uj{1} \\ \uj{2} \\ \uj{3} \end{matrix} \right], \hspace{.5cm}
\o=\left[ \begin{matrix} \see \\ \see \\ \miss \end{matrix} \right], \hspace{.5cm} \text{such that} \hspace{.5cm}
\hatu = \left[ \begin{matrix} 0 \\ \uj{2} \\ \miss \end{matrix} \right].
\end{align*}
Then, given $\hatu$, one could say nothing about $\uj{3}$, as it could potentially be any point in the $(y,z)$ plane.

\begin{figure}[H]
\centering
\includegraphics[width=5cm]{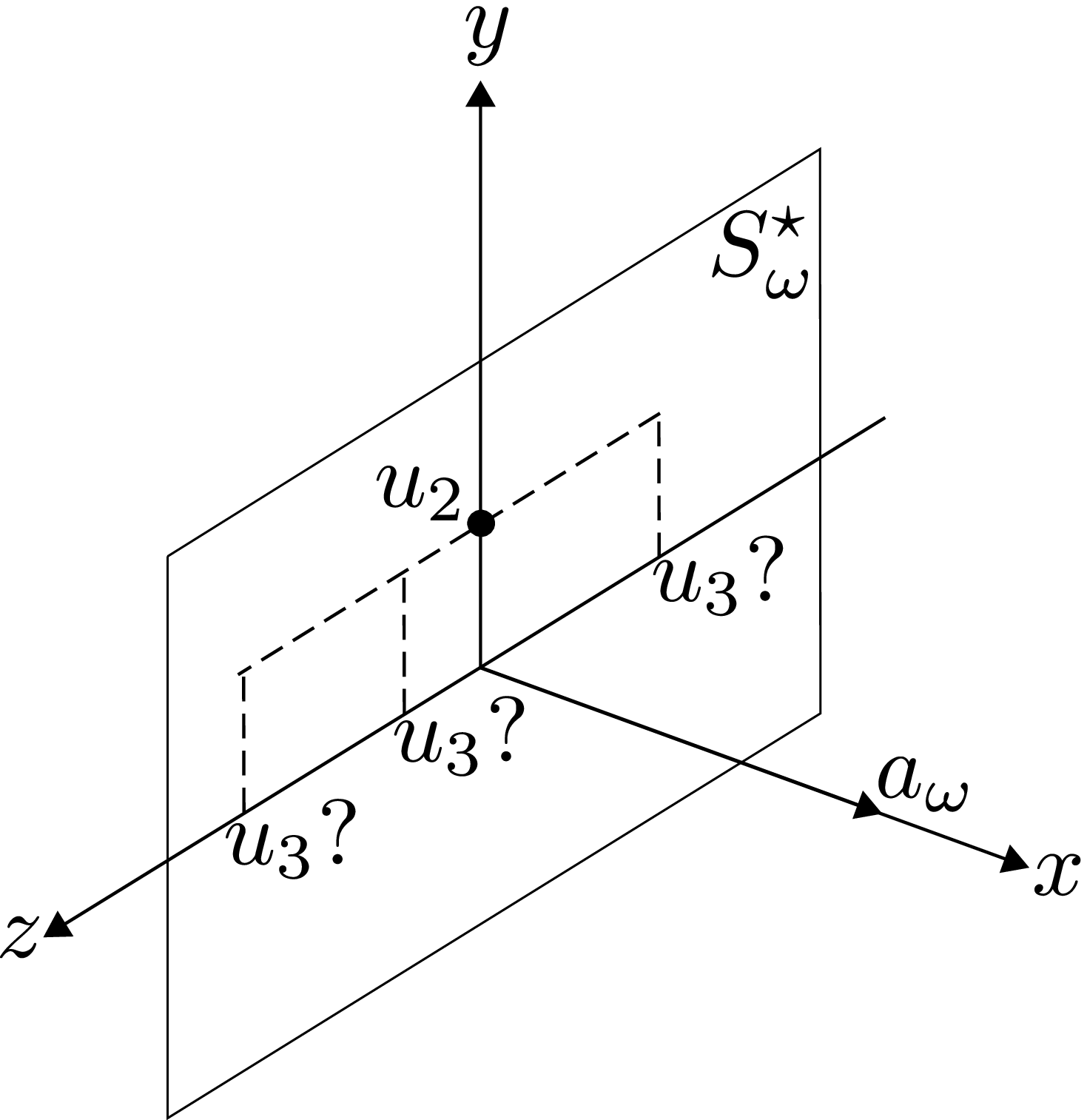}
\caption{Given $\hatu$, one could say nothing about $\uj{3}$, as it could potentially be any point in the $(y,z)$ plane.}
\end{figure}
$\blacksquare$
\end{myExample}

A little thought shows that under \sstarNonDegenerateAss: there are no subspaces of dimension lower than $\r$ that \fito\ $\o$; there are always infinitely many subspaces of dimension larger than $\r$ that \fito\ $\o$; there are infinitely many subspaces of dimension larger than $\r$ that \fitO\ $\O$ if there is at least one $\r$-dimensional subspace that \fitsO\ $\O$, and there is always a trivial subspace that \fitsO\ $\O$: $\R^\d$.

Generalizing \rDimensionAss\ becomes straightforward with these observations: if dimensions of the subspaces in $\Sstar$ are unknown simply, iterate over $\r=1,...\d$.  But more importantly, these observations tell us what would happen if $\sstar$ were \degenerate, i.e., if we dropped \sstarNonDegenerateAss.  First, there would be infinitely many $\r$-dimensional subspaces that \fitO\ $\OO$, no matter what $\OO$ is; even if $\{1,...,\d\} \in \OO$, just as there could be infinitely many $\r$-dimensional subspaces that \fitGen\ a generic $\y$ from an $(\r-1)$-dimensional subspace.

But a \degenerate\ $\r$-dimensional subspace is just a subspace that is even lower-rank {\em in some sections}; the {\em dimensions} corresponding to $\ups \subset \{1,...,\d\}$ with $|\ups| \leq \r$ such that $\dim S_\ups < |\ups|$.  So if we wanted to determine $\sstar$ and whether or not our data really lies in $\s$ whenever $\s$ fits our data, we could also iterate over $\tau=1,...,\r$ to find the lower-dimension {\em portion} of $\sstar$.  It is clear that in such cases, an $\o$ with $|\o|\leq\r$ would definitely be useful to identify the lower-dimensional {\em portions} of $\sstar$.

\subsection{Almost every $\neq$ every}
We close this section with an important reminder: our results hold for almost every $\Sstar$.  This means that there are {\em bizarre} and {\em extremely unlikely} sets of subspaces for which our results do not hold.

To see what exactly we mean by {\em bizarre}, consider the converse of \allOfAKindThm, stated precisely in \converseAllOfAKindCor.  It states that if $\OO$ satisfies the conditions of the theorem, and $\ki{i} \neq \ki{\ii}$ for some $(\i,\ii)$, almost surely there is no possible way that the projections of some $\r$-dimensional subspace $\s$ and $\sstar$ onto $\Rdo{\omega_i}$, namely $\hatsi{i}$ and $\hatsstari{i}$, are equal for every $\i$.  For this to happen, $\hatsstari{i}$ would have to be {\em aligned} with $\hatsstari{\ii}$ in the following sense.

\begin{myExample}
\normalfont
Suppose $\ki{1} \neq \ki{2} \neq \ki{3}$, and $\OO$ is as in \allOfAKindEgb.  Let's consider $\oi{1}$ and $\oi{2}$ first.  The intersection of $\uui{1}$ and $\uui{2}$ is a line (see \intesectionUoneUtwoFig).
\begin{figure}[H]
\centering
\includegraphics[width=5cm]{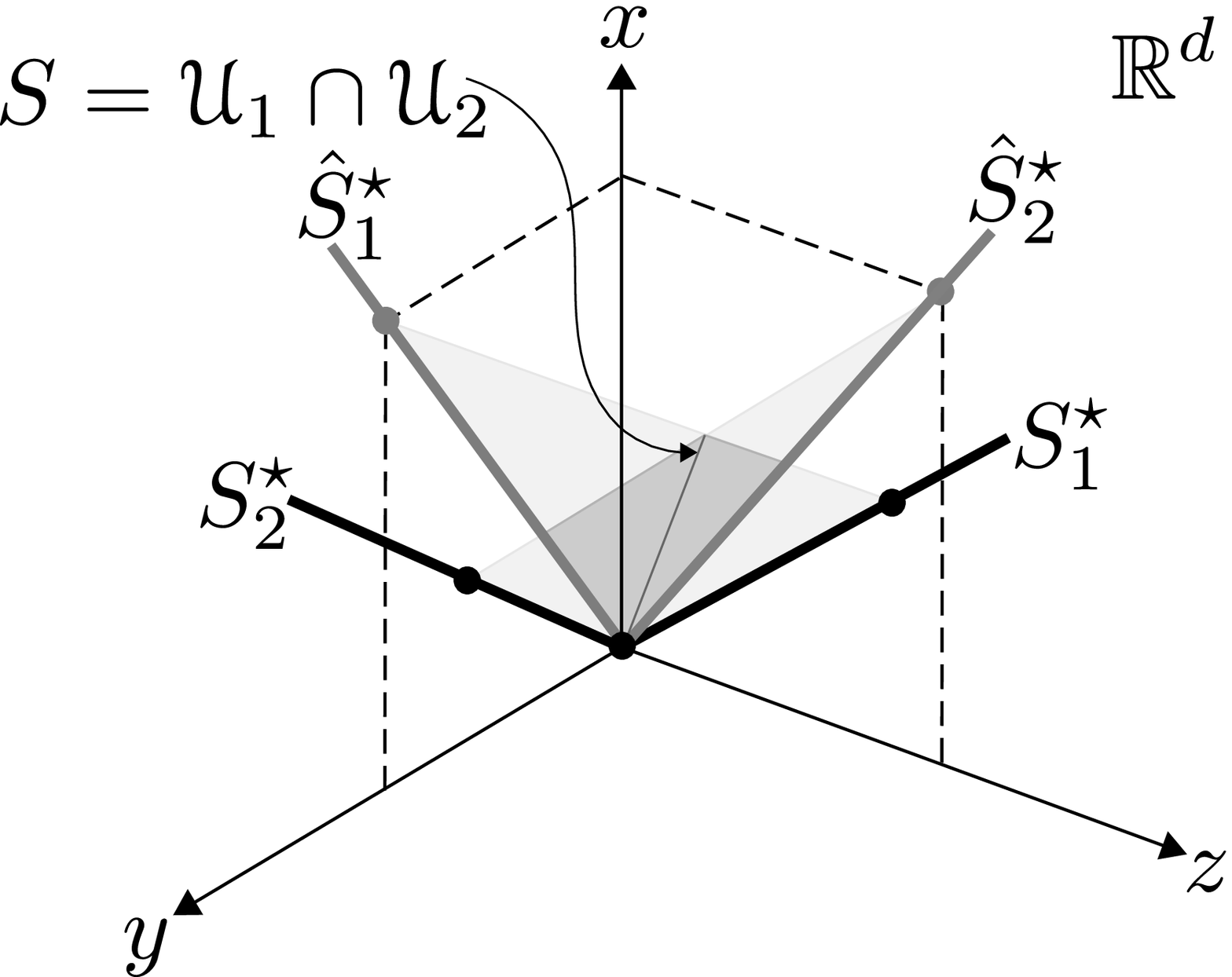}
\caption{The intersection of $\uui{1}$ and $\uui{2}$.}
\label{intesectionUoneUtwoFig}
\end{figure}

Recall that there is an $\r$-dimensional subspace that \fitsO\ $\O$ only if $\dim\UU\geq\r$, and $\UU=\bigcap_{\i=1}^\N \uui{i}$, so in this case $\dim\UU=\r=1$ only if $\sstari{3}$ is {\em aligned} with $\sstari{1}$ and $\sstari{2}$ in the sense that the projections of $\sstari{3}$ and $\s=\uui{1} \cap \uui{2}$ onto $\Rdo{\omega_3}$, in this case the $(y,z)$-plane, are the same.  Of course, there is only a set of measure zero for which $\sstari{3}$ will satisfy this, e.g., the set of subspaces in the shaded plane of \differentSthreeFig.

\begin{figure}[H]
\centering
\includegraphics[width=5cm]{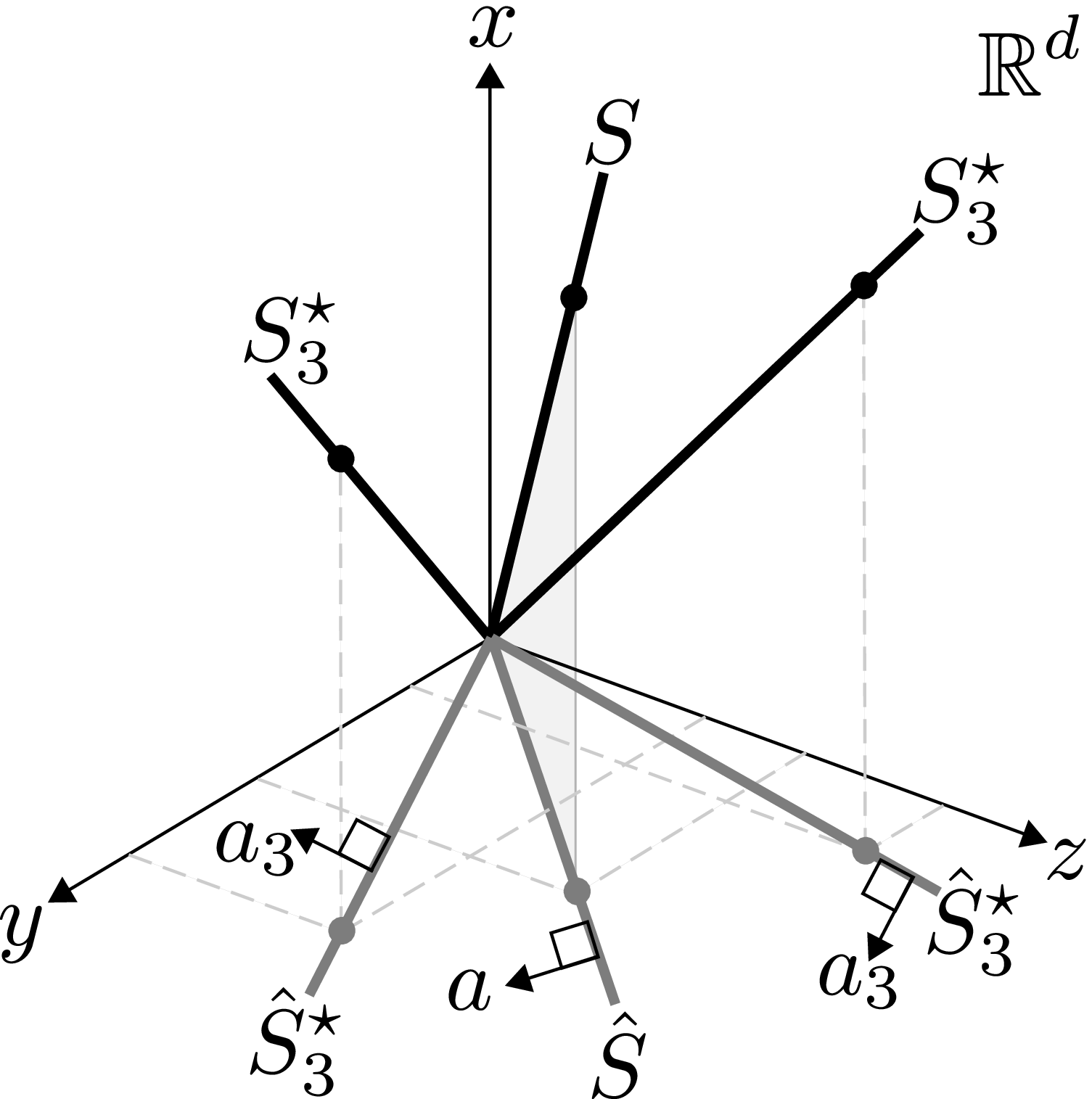}
\caption{There is an $\r$-dimensional subspace that \fitsO\ $\O$ only if $\sstari{3}$ lies in the shaded plane, which is a set of measure zero, as $\sstari{3} \subset \R^3$; equivalently, if $\ai{3}$ is aligned with $\a$, which is also a set of measure zero, as $\aoi{3} \in \R^2$.}
\label{differentSthreeFig}
\end{figure}

If $\sstari{3}$ happened to satisfy this condition, we would have a situation as in \almostNeverFig.

\begin{figure}[H]
\centering
\includegraphics[width=5cm]{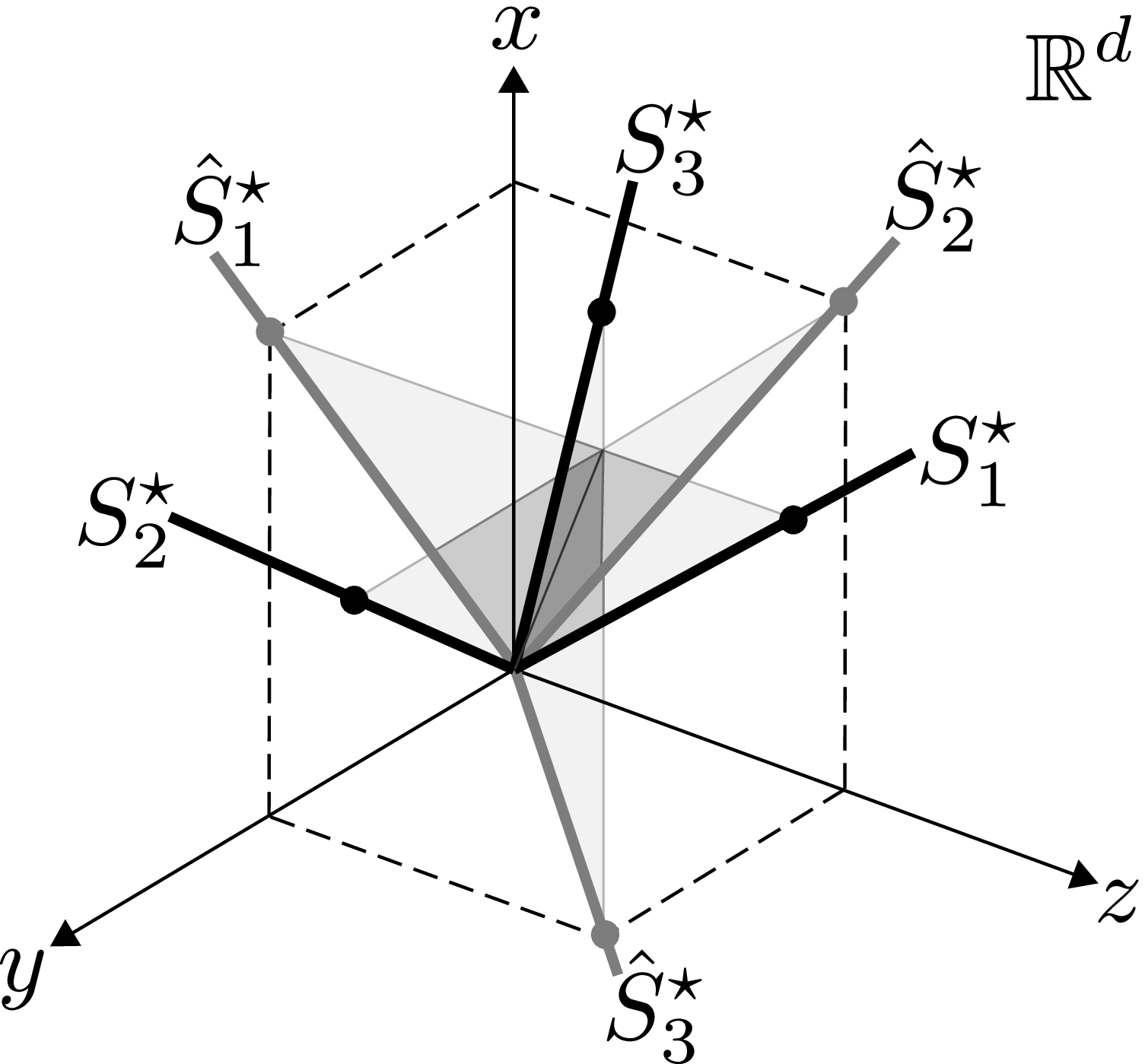}
\caption{Set of measure zero in which $\sstari{1}$, $\sstari{2}$ and $\sstari{3}$ are {\em aligned}.}
\label{almostNeverFig}
\end{figure}

In general, of course, the intersection of $\d-\r+1$ distinct hyperplanes in $\R^\d$, is an $\r-1$ subspace, as in \almostSurelyFig.

\begin{figure}[H]
\centering
\includegraphics[width=5cm]{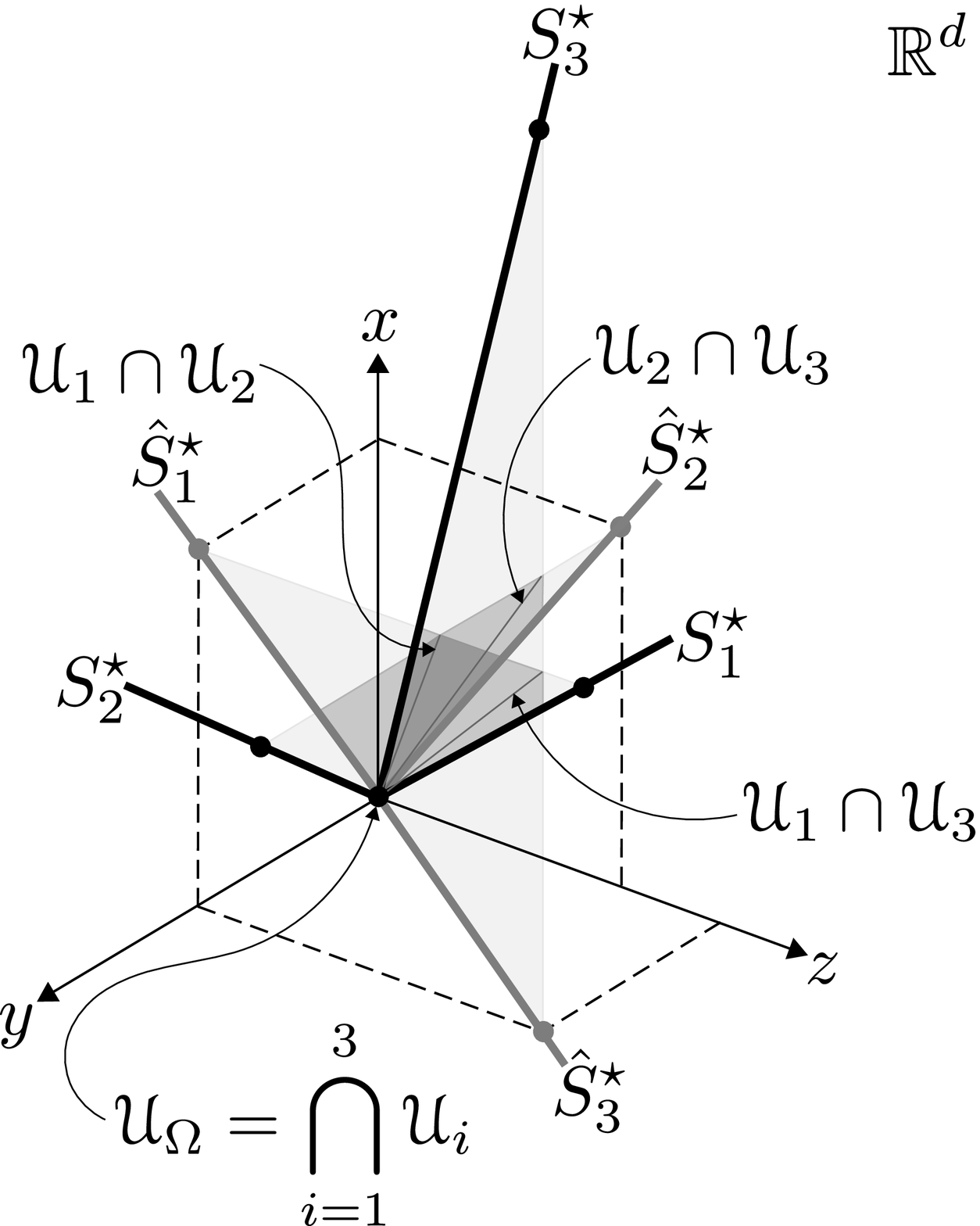}
\caption{Almost surely, the intersection of $\d-\r+1$ distinct planes in $\R^\d$, is an $(\r-1)$-dimensional subspace.  In our example, the intersection of the three different planes in $\R^3$ corresponding to each of the $\uui{i}$'s has dimension zero, i.e., it is just the origin.}
\label{almostSurelyFig}
\end{figure}

$\blacksquare$
\end{myExample}

A similar phenomenon occurs for \basisLem, where $\O$ may be a \basis\ for $\o$ even if $\m < \n+\r$.  For this to happen, some projections of $\sstari{i}$ would have to be {\em aligned} with some other projections of $\sstari{\ii}$.

All this to summarize that our results hold for almost every $\Sstar$.  For example, if $\Sstar$ were a collection of {\em random} subspaces, we would know that our results hold with probability $1$.

\pagebreak
\section{Epilogue}
\label{epilogueSec}
We now have deterministic necessary and sufficient conditions to determine if certain incomplete data really {\em lies} in a subspace whenever it {\em fits} in a subspace.  Moreover, we also have deterministic necessary and sufficient conditions to determine if there is only one subspace that {\em fits} certain partially observed data.  We answer these questions by characterizing when and only when a set of incomplete vectors \behaves\ as a single but complete one, in the sense described in \textsection \ref{introSec}.

Being such fundamental problems, it comes as no surprise that our results as well as the tools we used to derive them, have many powerful consequences and relevant applications; already have been discussed some of these in \textsection \ref{preambleSec}.  Finally, it is presented a brief discussion of some other very interesting related problems for future research.

\begin{description}
\item[P1.]
This is one particularly interesting idea: could a set of incomplete vectors $\hatXI$ behave as a set of $\r$ complete ones, forming a basis for a subspace $\s$, such that $\s$ could be uniquely determined only by $\hatXI$?

To be more precise, as we said in \textsection \ref{preambleSec}, in this document we assumed that a subspace $\s$ that fit certain incomplete data $\hatXI$ was given to us, and wanted to validate if $\hatXI$ really {\em lied} in $\s$ using $\hatX$ as a validating vector.  We now want to know what necessary and sufficient properties should $\hatXI$ have, e.g., what conditions on its observed entries, say $\Upsilon$, such that $\s$ could be uniquely identified solely from $\hatXI$, in a deterministic way, just as it would be determined by a set of $\r$ complete vectors.

This is exactly one of the low-rank matrix completion open questions: what are the necessary and sufficient conditions on $\Upsilon$ to guarantee that the subspace where $\XI$ lies can be uniquely identified from $\hatXI$.  We now know when and only when a set of generic incomplete vectors $\hatX$ \behaves\ as a complete generic vector in the sense described in \textsection \ref{introSec}.  I think this might give some insight to finding necessary and sufficient conditions on $\Upsilon$ for deterministic low-rank matrix completion, and conjecture that

\vspace{.3cm}
\begingroup
\leftskip1.5em
\rightskip\leftskip
\noindent
{\em All that is really required is that $\hatXI$ contains a set of $\r$ sets of incomplete vectors $\hatX_1,...,\hatX_\r$ observed in $\OO_1,...,\OO_\r$ such that each $\OO_\tau$ satisfies the conditions of \allOfAKindThm}
\par
\endgroup
\vspace{.3cm}

\item[P2.]
It is also of interest to characterize precisely and in a very concrete and intuitive manner what are the characteristics of the set of measure zero under which our results do not hold.

\item[P3.]
There is a relation between this work and concepts like the spark of a matrix, $\r$-connected graphs, matroids and matchings, to mention some.

Using tools from the corresponding areas one may derive powerful corollaries of our results.  For example, one corollary of \independenceLem\ is that there are infinitely many $\r$-dimensional subspaces that \fitO\ $\OO$ if the Tanner graph defined by $\OO$, as described in \textsection \ref{ideaBehindIndependenceLemSec}, is not $\r$-connected in the {\em row} vertices.

I would also like to pursue investigate these relations in more depth, with the hope of arriving at a non-combinatorial relaxation of the conditions of \independenceLem\ that could be computationally efficient to verify.

\item[P4.]
In some sense, in \textsection \ref{uniquenessThm} we are reconstructing a subspace from its projections onto $\Rdo{\omega}$.  I would also like to explore this idea more, and see if and when can we reconstruct a subspace from arbitrary projections.

\item[P5.]
Observe that $\AA\u=0$ is essentially describing a linear variety defined by $\N$ equations and $\d$ unknowns, where each equation only involves a subset of the unknowns; $\OO$ is the set describing which variables are involved in which equation, i.e., the {\em support} of equation $\i$ are the variables $\uoi{i}$.

\uniquenessThm\ bounds the dimension of such variety only as a function of the supports of its equations; I believe that it can also be used to bound the dimension of similar generic algebraic varieties, not necessarily linear, based only on the support of each of the polynomials that define it.
\end{description}

\pagebreak
\section{List of symbols, references and index}
\label{symbolsTab}
\begin{center}
\begin{longtable}{| c | l | l |}
\hline
\textbf{Symbol} & \textbf{Description} & \textbf{pp.} \\
\hline
\endfirsthead
\multicolumn{2}{l}%
{\tablename\ \thetable\ -- \textit{Continued from previous page}} \\
\hline
\textbf{Symbol} & \textbf{Description} & \textbf{pp.} \\
\hline
\endhead
\hline \multicolumn{2}{r}{\textit{Continued on next page}} \\
\endfoot
\hline
\endlastfoot
$\AA$ & $\N \times \d$ matrix with $\{\ai{i}\}_{\i=1}^\N$ as its rows & \pageref{AADef} \\
$\A$ & $\n \times \d$ matrix with $\{\ai{i}\}_{i \in \I}$ as its rows. & \pageref{ADef} \\ \hline
$\ao$ & Vector in $\R^{\r+1}$ orthogonal to $\sstaro$. & \pageref{aoDef} \\
$\a$ & $1 \times \d$ row vector with the entries of $\ao$ in the positions of $\o$. & \pageref{aDef} \\ \hline
$\d$  & Ambient dimension. & \pageref{dDef} \\ \hline
$\I$ & Indices of the $\o$'s in $\OO$ that belong to $\O$. & \pageref{IJDef} \\
$\i$, $\ii$ & Used to index vectors.  In general, $\in \{1,...,\N\}$. & \pageref{XDef} \\ \hline
$\J$ & Observed {\em rows} of $\O$. & \pageref{IJDef} \\
$\j$ & Used to index entries or elements.  In general, $\in \{1,...,\d\}$. & \pageref{jDef} \\
$\jc$ & Given $\j$ and $\o$, $\jc:= \o \backslash \j$. & \pageref{jcDef} \\ \hline
$\K$ & Multiset of indices $\ki{i} \in \{1,...,\KK\}$ indicating that $\xii{i} \in \sstark{k_i}$. & \pageref{KDef} \\
$\KK$ & Number of subspaces in $\Sstar$. & \pageref{KKDef} \\
$k$ & Used to index {\em distinct} subspaces.  In general, $\in \{1,...,\KK\}$. & \pageref{SstarkDef} \\
$\ki{i}$ & Index denoting which subspace of $\Sstar$ $\xii{i}$ belongs to. & \pageref{KDef} \\ \hline
$\L$ & Number of linearly independent rows in $\A$. & \pageref{lDef} \\
$\m$ & Number of distinct observed {\em rows} in $\O$. & \pageref{nmDef} \\
$\N$ & Number of vectors in $\X$. & \pageref{XDef} \\
$\n$ & Number of sets that $\O$ contains. & \pageref{nmDef} \\ \hline
$\OO$ & Set of observed entries of $\hatX$. & \pageref{OODef} \\
$\O$ & Subset of $\OO$. & \pageref{ODef} \\
$\o$ & Set of observed entries of $\hatx$.  Subset of $\{1,...,\d\}$ of size $\r+1$ & \pageref{oDef} \\
$\r$ & Dimension of the subspaces. & \pageref{rDef} \\ \hline
\multirow{2}{*}{$\Rdo{\omega}$} & Subspace of $\R^\d$ spanned by the canonical vectors. & \multirow{2}{*}{\pageref{hatsDef}} \\
 &  corresponding to the positions of $\o$ & \\ \hline
$\s$ & Arbitrary subspace, generally $\r$-dimensional. & \pageref{sDef} \\
$\hats$ & Projection of $\s$ onto $\Rdo{\omega}$. & \pageref{hatsDef} \\
$\so$ & Subspace of $\R^{|\o|}$.  The restriction of $\s$ to $\o$. & \pageref{soDef} \\ \hline
$\sstar$ & Arbitrary subspace of $\Sstar$. & \pageref{sstarDef} \\
$\sstark{k}$ & $\k^{th}$ subspace of $\Sstar$. & \pageref{SstarkDef} \\
$\sstari{i}$ & Subspace where $\xii{i}$ lies.  $\ki{i}^{th}$ subspace of $\Sstar$. & \pageref{sstariDef} \\
$\hatsstar$ & Projection of $\sstar$ onto $\Rdo{\omega}$. & \pageref{hatsstariDef} \\
$\sstaro$ & Subspace of $\R^\d$.  The restriction of $\sstar$ to $\o$. & \pageref{sstaroiDef} \\ \hline
$\Sstar$ & Set of $\KK$ $\r$-dimensional subspaces. $\X$ lies in their union. & \pageref{SstarkDef} \\ \hline
$\UU$ & Subspace of all $\r$-dimensional subspaces that \fitO\ $\O$. & \pageref{UUDef} \\
$\uu$ & Subspace of all $\r$-dimensional subspaces that \fito\ $\o$. & \pageref{uuDef} \\
$\uui{i}$ & Subspace of all $\r$-dimensional subspaces that \fito\ $\oi{i}$. & \pageref{uuiDef} \\ \hline
$\U$ & Arbitrary basis of $\s$. & \pageref{UDef} \\
$\hatU$ & Arbitrary basis of $\hats$. & \pageref{hatsDef} \\
$\Uo$ & Arbitrary basis of $\so$. & \pageref{soDef} \\ \hline
$\Ustar$& Arbitrary basis of $\sstar$. & \pageref{UstarDef} \\
$\hatUstar$ & Arbitrary basis of $\hatsstar$. & \pageref{hatsDef} \\
$\Ustaro$ & Arbitrary basis of $\sstaro$. & \pageref{UstaroDef} \\ \hline

$\u$ & Arbitrary vector, typically from $\s$, $\uu$ or $\UU$. & \pageref{UUDef} \\
$\hatu$ & Incomplete version of $\u$; analogous to $\hatx$. & \pageref{hatxDef} \\
$\uo$ & Vector in $\R^{|\o|}$ with the entries of $\u$ in the positions of $\o$. & \pageref{xoDef} \\ \hline
$\X$ & Set of $\N$ vectors in the union of the subspaces in $\Sstar$. & \pageref{XDef} \\
$\hatX$ & Incomplete version of $\X$. & \pageref{hatXDef} \\
$\hatXI$ & Set of incomplete vectors that \fitXO\ in an $\r$-dimensional subspace. & \pageref{XIDef} \\ \hline
$\x$ & Arbitrary vector of $\X$. & \pageref{xDef} \\
$\xii{i}$ & $\i^{th}$ vector of $\X$. & \pageref{XDef} \\
$\hatx$ & Incomplete version of $\x$. & \pageref{hatxDef} \\
$\xo$ & Vector in $\R^{|\o|}$ with the entries of $\x$ in the positions of $\o$. & \pageref{xoDef} \\ \hline
$\y$ & Complete generic vector. & \pageref{yDef} \\ \hline
$\see$ & Symbol to denote that an entry is observed. & \pageref{seeDef} \\
$\miss$ & Symbol to denote that an entry is missing. & \pageref{missDef} \\ \hline
$\bb{\Scale[.5]{\circ}}$ & Analogous to $\Scale[.5]{\circ}$ , e.g., $\bb{\n}=\bb{\O}$  & \pageref{nmDef} \\
$\bar{\Scale[.5]{\circ}}$ & Analogous to $\Scale[.5]{\circ}$ , e.g., $\bar{\m}$ is the number of observed {\em rows} of $\bb{\O}$ & \pageref{nmDef} \\
$\Scale[.5]{\circ}_\i$ & Analogous to $\Scale[.5]{\circ}$ , e.g., $\hatsstari{i}$ is the projection of $\sstari{i}$ onto $\Rdo{\omega_i}$ & \pageref{hatsstariDef}
\end{longtable}
\end{center}

\pagebreak
\small
\label{referencesSec}
\bibliographystyle{IEEEbib}
\vspace{-2mm}
\bibliography{ToLieOrNotToLie}

\printindex

\end{document}